\documentclass[12pt, oneside]{book} 
\usepackage{KUstyle}

\usepackage[round]{natbib}
\usepackage{bibentry}
\nobibliography*

 \graphicspath{ {images/} }

\usepackage{amsthm} 

\newtheorem*{prop*}{Proposition}

\newtheorem{theorem}{Theorem}[chapter]

\newtheorem*{lemma*}{Lemma}

\newtheorem{definition}{Definition}[chapter]

\ptype{Ph.D. Thesis}
\author{Yijie Zhang}
\title{On Cold Posteriors of Probabilistic Neural Networks}
\subtitle{Understanding the Cold Posterior Effect and A New Way to Learn Cold Posteriors with Tight Generalization Guarantees}
\advisor{Advisors: Christian Igel and Sadegh Talebi}
\date{This thesis has been submitted to the Ph.D. School of The Faculty of Science, University of Copenhagen on September $3^{\text{rd}}$ 2024.}

\usepackage{titlesec}
\usepackage{fancyhdr}

\fancypagestyle{plain}{         
\fancyhf{}
\fancyfoot[C]{\thepage}}%

\renewcommand{\headrulewidth}{0pt}

\newcommand\mymainpagestyle{%
\fancyhf{}      
\fancyhead[L]{\nouppercase{\footnotesize{\chaptername~ \thechapter~ |~ \leftmark}} \renewcommand{\headrulewidth}{0.4pt} \headrule \renewcommand{\headrulewidth}{0pt}}
\setlength{\headheight}{25pt}
\fancyfoot[C]{\thepage}
}

\usepackage{amsmath,amsfonts,bm}

\def\eqref#1{equation~\ref{#1}}

\DeclareMathAlphabet{\mathsfit}{\encodingdefault}{\sfdefault}{m}{sl}
\SetMathAlphabet{\mathsfit}{bold}{\encodingdefault}{\sfdefault}{bx}{n}

\newcommand{\E}{\mathbb{E}}

\newcommand{\R}{\mathbb{R}}

\newcommand{\KL}{\mathrm{KL}}

\DeclareMathOperator*{\argmin}{arg\,min}

\usepackage{amsthm}
\usepackage{hl_short}
\usepackage[most]{tcolorbox}
\newtcolorbox{Overview}[1][]{%
   enhanced,
   detach title,
   title=Section Overview,
   sharp corners=northwest,
   fonttitle=\bfseries,
   before    = \par\smallskip\centering,
   after     = \par,
   colback   = green!5,
   colframe  = green!50!black, 
   coltitle  = green!5!white, 
   colbacktitle = green!50!black,
   attach boxed title to top left={yshift=-4mm,yshifttext=-2mm},
   boxed title style={rounded corners, arc = 2.5mm, outer arc = 1.5mm},
   arc       = 4mm, 
   outer arc = 3.5mm, 
   #1}

\newtheorem{proposition}[theorem]{Proposition}
\newtheorem{insight}{Insight}
\newcommand{\daset}{H}

\usepackage{caption}
\usepackage{subcaption}
\usepackage{booktabs}

\usepackage[utf8]{inputenc} 
\usepackage[T1]{fontenc}    
\usepackage{hyperref}       
\usepackage{url}            
\usepackage{amsfonts}       
\usepackage{nicefrac}       
\usepackage{microtype}      
\usepackage{xcolor}         
\usepackage{multirow}
\usepackage{enumitem}

\usepackage{amsthm}
\usepackage{amssymb}
\usepackage{amsmath}
\usepackage{hyperref}
\usepackage{xifthen}
\usepackage{xparse}
\usepackage{dsfont}
\usepackage{rotating}
\usepackage{xcolor}
\usepackage{mathtools}

\usepackage{algorithm}
\usepackage{algorithmicx}
\usepackage{algpseudocode}

\newcommand{\lr}[1]{\left (#1\right)}
\newcommand{\lrc}[1]{\left \{#1\right\}}
\newcommand{\lrs}[1]{\left [#1 \right]}

\NewDocumentCommand{\1}{o}{\mathds 1{\IfValueT{#1}{\lr{#1}}}}
\let\P\undefined
\NewDocumentCommand{\P}{o}{\mathbb P{\IfValueT{#1}{\lr{#1}}}}

\newcommand{\PP}{\mathcal P}

\newcommand{\XX}{\mathcal X}
\newcommand{\YY}{\mathcal Y}
\newcommand{\DD}{\mathcal D}
\newcommand{\ZZ}{\mathcal Z}

\newcommand{\Utrain}{U^{\mathrm{train}}}
\newcommand{\Uval}{U^{\mathrm{val}}}
\newcommand{\nval}{n^{\mathrm{val}}}

\newcommand{\HH}{\mathcal H}

\let\emptyset\varnothing

\DeclareMathOperator{\kl}{kl}

\DeclareMathOperator{\B}{B}
\newcommand{\Ex}{\mathcal{E}}

\usepackage{titlesec}

\titleformat{\paragraph}
{\normalfont\normalsize\bfseries}{\theparagraph}{1em}{}
\titlespacing*{\paragraph}
{0pt}{3.25ex plus 1ex minus .2ex}{1.5ex plus .2ex}

\begin{document}


\maketitle
\frontmatter 
\pagestyle{plain} 

\newpage \ \newpage 

\section*{Abstract}
\label{sec:abstract}
\addcontentsline{toc}{section}{Abstract} 
Bayesian inference provides a principled probabilistic framework for quantifying uncertainty by updating beliefs based on prior knowledge and observed data through Bayes' theorem. In Bayesian deep learning, neural network weights are treated as random variables with prior distributions, allowing for a probabilistic interpretation and quantification of predictive uncertainty. However, Bayesian methods lack theoretical generalization guarantees for unseen data. PAC-Bayesian analysis addresses this limitation by offering a frequentist framework to derive generalization bounds for randomized predictors, thereby certifying the reliability of Bayesian methods in machine learning.

Temperature $T$, or inverse-temperature $\lambda = \frac{1}{T}$, originally from statistical mechanics in physics, naturally arises in various areas of statistical inference, including Bayesian inference and PAC-Bayesian analysis. In Bayesian inference, when $T < 1$ (``cold'' posteriors), the likelihood is up-weighted, resulting in a sharper posterior distribution. Conversely, when $T > 1$ (``warm'' posteriors), the likelihood is down-weighted, leading to a more diffuse posterior distribution. By balancing the influence of observed data and prior regularization, temperature adjustments can address issues of underfitting or overfitting in Bayesian models, bringing improved predictive performance.

We begin by investigating the cold posterior effect (CPE) in Bayesian deep learning. We demonstrate that misspecification leads to CPE only when the Bayesian posterior underfits. Additionally, we show that tempered posteriors are valid Bayesian posteriors corresponding to different combinations of likelihoods and priors parameterized by temperature $T$. Fine-tuning $T$ thus allows for the selection of alternative Bayesian posteriors with less misspecified likelihood and prior distributions.

Next, we introduce an effective PAC-Bayesian procedure, Recursive PAC-Bayes (RPB), that enables sequential posterior updates without information loss. This method is based on a novel decomposition of the expected loss of randomized classifiers, which reinterprets the posterior loss as an excess loss relative to a scaled-down prior loss, with the latter being recursively bounded. We show empirically that RPB significantly outperforms prior works and achieves the best generalization guarantees. 

We then explore the connections between Recursive PAC-Bayes, cold posteriors ($T < 1$), and KL-annealing (where $T$ increases from $0$ to $1$ during optimization), showing how RPB's update rules align with these practical techniques and providing new insights into RPB's effectiveness. 

Finally, we present a novel evidence lower bound (ELBO) decomposition for mean-field variational global latent variable models, which could enable finer control of the temperature $T$. This decomposition could be valuable for future research, such as understanding the training dynamics of probabilistic neural networks.

\newpage
\section*{Resumé}
\label{sec:resume}
\addcontentsline{toc}{section}{Resumé}
Bayesiansk inferens giver en principiel, probabilistisk ramme for at kvantificere usikkerhed ved at opdatere overbevisninger baseret på tidligere viden og observerede data gennem Bayes' sætning. I bayesiansk dyb læring behandles neurale netværksvægte som stokastiske variable med fordeling før dataindsamlingen, hvilket muliggør en probabilistisk fortolkning og kvantificering af prædiktiv usikkerhed. Dog mangler bayesianske metoder teoretiske generaliseringsgarantier for uobserverede data. PAC-bayesiansk analyse imødegår denne begrænsning ved at tilbyde en frekventistisk ramme til at udlede generaliseringsgrænser for randomiserede prædiktorer og dermed sikre pålideligheden af bayesianske metoder i maskinlæring.

Temperatur $T$, eller omvendt temperatur $\lambda = \frac{1}{T}$, stammer oprindeligt fra statistisk mekanik i fysik og opstår naturligt i forskellige områder af statistisk inferens, herunder bayesiansk inferens og PAC-bayesiansk analyse. I bayesiansk inferens, når $T < 1$ (``kolde'' posteriors), opvejes sandsynligheden, hvilket resulterer i en skarpere posteriorfordeling. Omvendt, når $T > 1$ (``varme'' posteriors), nedvejes sandsynligheden, hvilket fører til en mere diffus posteriorfordeling. Ved at balancere indflydelsen fra observerede data og forudgående regulering kan temperaturjusteringer afhjælpe problemer med under- eller overtilpasning i bayesianske modeller og dermed forbedre prædiktiv ydeevne.

Vi begynder med at undersøge det kolde posteriorfænomen (CPE) i bayesiansk dyb læring. Vi demonstrerer, at fejl i modellen fører til CPE kun, når den bayesianske posterior undertilpasser. Derudover viser vi, at tempererede posteriors er gyldige bayesianske posteriors, der svarer til forskellige kombinationer af sandsynligheder og forudgående distributioner parameteriseret ved temperatur $T$. Finjustering af $T$ tillader således valg af alternative bayesianske posteriors med mindre fejl i sandsynlighed og forudgående distributioner.

Derefter introducerer vi en effektiv PAC-bayesiansk procedure, Recursive PAC-Bayes (RPB), der muliggør sekventielle posterioropdateringer uden informations tab. Denne metode er baseret på en ny opdeling af det forventede tab for randomiserede klassifikatorer, som fortolker posterior-tabet som et overskudstab i forhold til et nedskaleret forudgående tab, hvor sidstnævnte rekursivt begrænses. Vi viser empirisk, at RPB signifikant overgår tidligere værker og opnår de bedste generaliseringsgarantier.

Vi undersøger derefter forbindelserne mellem Recursive PAC-Bayes, kolde posteriors ($T < 1$), og KL-annealing (hvor $T$ stiger fra $0$ til $1$ under optimering), hvilket viser, hvordan RPB's opdateringsregler stemmer overens med disse praktiske teknikker og giver nye indsigter i RPB's effektivitet.

Endelig præsenterer vi en ny dekomposition af evidens-lower-bound (ELBO) for mean-field variational globale latente variabelmodeller, som kunne muliggøre finere kontrol af temperaturen $T$. Denne dekomposition kan være værdifuld for fremtidig forskning, såsom forståelse af træningsdynamikken i probabilistiske neurale netværk.

\newpage
\section*{Acknowledgements}
\label{sec:acks}
\addcontentsline{toc}{section}{Acknowledgements}
Compared to some of my dearest friends and colleagues, I had almost never envisioned embarking on the journey of a Ph.D., let alone becoming a Ph.D. I still remember the first time I considered the possibility of this path: it was in Amsterdam, where I completed my master's thesis, a project I found incredibly interesting and fulfilling. It was the first time I truly experienced the joy of research. I also remember receiving the offer from Christian for this position. Although I felt happy and excited, the fear of the unknown journey ahead was even more predominant. The journey of a Ph.D. is never easy. It means spending years learning and exploring the boundaries of a particular area of human knowledge, hoping to make a small discovery or breakthrough. It means overcoming the fear of the unknown and believing in the significance of your work. \\

\noindent
Looking back, I am reminded of Steve Jobs' 2005 Stanford commencement speech about connecting the dots and doing what you love. On countless nights, when I worried about not having enough publications to graduate or feared my research themes lacked coherence, I recalled these points. As a result, I always followed my intuition to find or engage in projects that truly captivated me. Ultimately, this approach proved to be crucial. Furthermore, without the support, companionship, and encouragement of the incredible people around me, I would never have been able to complete this journey. Now, I would like to express my gratitude to them. \\

\noindent
First, I would like to thank my supervisors, Christian and Sadegh. They granted me tremendous freedom and encouraged me to pursue topics that genuinely interested me. Whenever I encountered difficulties, they consistently provided valuable advice and support, both in my work and personal life. \\

\noindent
Second, I want to thank my wonderful collaborators. Each of you has taught me something unique, and our collaboration has been a lifelong benefit. Yi-Shan, I am immensely grateful for your help with both my work and life. Your diligence and rigor have deeply inspired me. The quote on your mug, ``Every piece of work reflects oneself; anything handled by me will be a masterpiece'' perfectly encapsulates you. Our travels, meals, coffee tastings, and alcohol appreciation (though I am far from matching your expertise) are among the most cherished memories of my Ph.D. life. Andr{\'e}s, you are always full of brilliant and wild ideas, and you always manage to transform them into beautiful mathematics. Your research introduced me to a fascinating field that became the most important part of my thesis. Yevgeny, although my first impression of you was someone particularly meticulous and zen-like, as the project progressed, I discovered your keen observation of experimental data and your creative thinking in new directions. Your writing is also incredibly clear and elegant. Luis and Badr, although we have never met in person, our remote working relationship has been exceptionally smooth, and our communications have benefited me greatly. \\

\noindent
Third, I want to thank my colleagues and friends at DIKU, especially those in the DeLTA Group: Yi-Shan, Hippolyte, Chlo{\'e}, Saeed, Ola, Yuen, Yunlian, Arthur, Aymeric, Shaojie, and Oliver. Our memories span across the entire European continent, and I never imagined having such an exciting life.\\

\noindent
Fourth, I want to thank Novo Nordisk A/S for funding my Ph.D. project and all my colleagues there for their help. I especially want to thank Jakob, David, Merete, Jens, and Haocheng. You taught me the Danish way of working and living, and how to focus on long-term goals in a fun and balanced way.\\

\noindent
Fifth, I would like to thank my colleagues and mentors at the State Key Laboratory of Intelligent Technology and Systems at Tsinghua University, especially Professor Jun Zhu, Ziyu, Yuhao, Kaiwen, Tianjiao, Rosie, and Julian. I am very fortunate to have had the opportunity to visit and exchange ideas at the university I dreamed of as a child. The taste, innovation, and high standards of your work have deeply influenced me. I also want to thank Melih, Nicklas, Abdullah, and Bahareh for their help during my visit to the University of Southern Denmark.\\



\noindent
Lastly, I want to express my deepest gratitude to my family and friends, with a special mention to my parents. Your unwavering love, support, and encouragement, particularly during my most challenging moments, have been my anchor. I am incredibly fortunate to have you in my life, and you are the most invaluable part of it.

\noindent
\newpage
\tableofcontents
\newpage

\mainmatter 
\mymainpagestyle{} 

\chapter{Introduction}
\label{chap:intro}

Bayesian inference \citep{bayes1763, DBLP:books/lib/Bishop07, DBLP:journals/nature/Ghahramani15} is a cornerstone of statistical inference, providing a flexible and robust framework for updating the probability of a hypothesis as new data is acquired. Utilizing Bayes' theorem, it combines prior knowledge with observed data to produce a posterior distribution that quantifies updated beliefs. This approach allows for a full probabilistic understanding of uncertainty, offering a nuanced perspective beyond point estimates typical of classical methods. Bayesian inference is particularly powerful in complex and evolving datasets, continuously refining predictions as new data becomes available. Its ability to handle hierarchical models and incorporate various levels of uncertainty makes it highly suitable for real-world applications. Additionally, Bayesian methods facilitate model comparison and selection through marginal likelihoods and Bayes factors \citep{DBLP:conf/icml/LotfiIBGW22}, ensuring all sources of uncertainty are properly accounted for. Widely used across diverse fields such as biology \citep{article}, economics \citep{Geweke2011}, and artificial intelligence \citep{Gelman2013}, Bayesian inference integrates prior knowledge with empirical data, enhancing both exploratory data analysis and predictive modeling, making it an essential tool for scientific research and practical decision-making.

Bayesian neural networks (BNNs) \citep{DBLP:journals/corr/abs-2001-10995, Gal2016, Wang2016} extend Bayesian inference principles to deep learning, treating weights and biases as random variables with prior distributions that update to posterior beliefs based on data. This probabilistic approach contrasts with traditional neural networks, where parameters are fixed point estimates, enhancing model robustness and uncertainty quantification. BNNs incorporate prior knowledge, improve learning efficiency, and provide accurate uncertainty estimates, crucial for high-stakes applications like medical diagnosis \citep{DBLP:journals/access/AbdullahHM22}, autonomous driving \citep{DBLP:conf/ijcai/McAllisterGKWSC17}, and financial forecasting \citep{Zhang2019}. Recent advances, such as variational inference \citep{blei2017variational, DBLP:journals/ml/JordanGJS99, DBLP:journals/ftml/WainwrightJ08} and Monte Carlo methods \citep{brooks2011handbook, DBLP:conf/icml/WellingT11,zhang2019cyclical}, have made BNNs more computationally feasible, allowing for efficient approximation of intractable posterior distributions. These techniques have expanded the practical applications of BNNs, enabling their use in large-scale and complex models. BNNs excel in various fields, including computer vision \citep{DBLP:conf/cvpr/GustafssonDS20}, natural language processing \citep{DBLP:series/synthesis/2019Cohen}, and reinforcement learning \citep{DBLP:conf/uai/OsbandWADILR23}, where they handle ambiguous inputs, language variability, and model uncertainty in value functions. The Bayesian approach in deep learning through BNNs improves predictive performance, robustness, interpretability, and reliability, making it a valuable tool in developing and applying deep learning technologies.

On the other hand, probabilistic neural networks (PNNs) \citep{DBLP:journals/jmlr/Perez-OrtizRSS21, DBLP:conf/uai/DziugaiteR17, DBLP:journals/corr/abs-1908-07380}, or randomized neural networks, generalize the concept of Bayesian neural networks. Unlike Bayesian neural networks, which update distributions using Bayes' rule, PNNs can employ various methods for updating these distributions, offering greater flexibility in training and inference. Methods for obtaining probabilistic neural networks, such as standard Bayes \citep{bayes1763, DBLP:books/lib/Bishop07, DBLP:journals/nature/Ghahramani15}, Gibbs Bayes \citep{Ghosh2016, Bissiri2016, DBLP:conf/aabi/Cherief-Abdellatif19}, and power likelihood Bayes (tempered posteriors) \citep{Holmes2017,Grunwald2017,Miller2019}, can all be unified under the following optimization problem \citep{Knoblauch2022}:
\begin{align} 
    q^*(\bmtheta) = \arg\min_{q \in \Pi} \left\{ \mathbb{E}_{q(\bmtheta)} \left[ \sum_{i=1}^{n} \ell(\bmtheta, x_i) \right] + \underbrace{\dfrac{1}{\lambda}}_{\text{temperature $T$}} D(q \| \pi) \right\} = P(\ell, D, \Pi), \label{eq:intro}
\end{align}
where $\ell$ is the loss function, $D$ is the divergence measure, and $\Pi$ is the set of feasible posteriors. Most prominently, the temperature $T$, or inverse-temperature $\lambda=\dfrac{1}{T}$, which originated from statistical mechanics in physics, naturally arises in this context. Adjusting the temperature $T$ provides flexibility in probabilistic models, allowing for better control over the trade-off between fitting the data and maintaining model complexity. This flexibility is particularly beneficial when dealing with complex real-world data, where the true data-generating process is unknown and model misspecification is likely to occur. Since the optimality of Bayesian inference relies on the assumption of perfect model specification \citep{Bochkina2022, zhang2021on}, tuning the temperature to values other than 1 (standard Bayes) often leads to improved performance in practice \citep{Grunwald2017, Holmes2017,WRVS+20}.

Despite probabilistic neural networks' strong empirical success, the theoretical generalization guarantee on unseen data is still missing. This gap has led to the development of PAC-Bayesian (Probably Approximately Correct Bayesian) analysis \citep{STW97,McA98}, which provides a robust framework for deriving generalization bounds for probabilistic models. PAC-Bayesian analysis combines the principles of Bayesian inference with PAC learning theory, offering probabilistic bounds on a model's generalization error based on its empirical error and the complexity of its hypothesis class. 
This approach has been particularly useful in understanding and improving the generalization capabilities of neural networks \citep{DBLP:conf/uai/DziugaiteR17, DBLP:conf/iclr/ZhouVAAO19, DBLP:conf/nips/NegreaHDK019}, where PAC-Bayesian bounds can guide the design of training algorithms that minimize these bounds \citep{DBLP:journals/jmlr/Perez-OrtizRSS21,  DBLP:journals/corr/abs-1908-07380}, thus enhancing the models' robustness and predictive accuracy on new data. Interestingly, many PAC-Bayesian bounds, when treated as training objectives, can be represented as special cases of Equation \ref{eq:intro}, containing again the temperature parameter $T=\dfrac{1}{\lambda}$. By choosing an appropriate temperature to balance empirical performance and model complexity, PAC-Bayesian analysis offers a powerful tool for developing more reliable and theoretically sound machine learning models.

As previously discussed, the temperature parameter \( T \) plays a pivotal role in various aspects of statistical inference. Its importance has drawn considerable interest from the research community, particularly in the context of overparameterized neural networks where cold posteriors, characterized by temperatures \( T < 1 \), are employed. In this work, we aim to provide a comprehensive and unified study of recent advances in the understanding and application of cold posteriors within probabilistic neural networks. Our focus will be on dissecting the cold posterior effect and introducing a novel approach for learning cold posteriors with tight generalization guarantees. Additionally, we will examine the broader implications of varying temperature settings across different scenarios, illuminating its connections to other key topics within the field.

\section{Outline of the Thesis}
This thesis is organized into the following chapters, where Chapter \ref{chap:chap2} and Chapter \ref{chap:chap3} contain the papers \cite{zhang2024the} and \cite{wu2024recursivepacbayesfrequentistapproach} respectively.

In Chapter \ref{chap:chap2} we study the problem of cold posterior effect (CPE) \citep{WRVS+20}. The cold posterior effect (CPE) \citep{WRVS+20} in Bayesian deep learning demonstrates that posteriors with a temperature \( T < 1 \) can yield posterior predictives that outperform the standard Bayesian posterior (\( T = 1 \)). Since the Bayesian posterior is considered optimal under perfect model specification, recent studies have explored CPE as a consequence of model misspecification, originating from either the prior or the likelihood. In this research, we refine the understanding of CPE by showing that misspecification induces CPE only when the resulting Bayesian posterior underfits. We theoretically establish that without underfitting, CPE does not occur. Furthermore, we demonstrate that these tempered posteriors with \( T < 1 \) are indeed valid Bayesian posteriors, corresponding to different combinations of likelihoods and priors parameterized by \( T \). This insight justifies adjusting the temperature hyperparameter \( T \) as an effective strategy to mitigate underfitting in the Bayesian posterior. Essentially, we reveal that fine-tuning the temperature \( T \) allows for the implicit utilization of alternative Bayesian posteriors, characterized by less misspecified likelihood and prior distributions.

In Chapter \ref{chap:chap3}, we investigate the problem of sequentially updating posteriors in PAC-Bayesian analysis without losing confidence information. PAC-Bayesian analysis, inspired by Bayesian learning, traditionally struggles with maintaining confidence information during sequential updates, as the final confidence intervals only depend on data not used for prior construction, losing information related to the size of the initial dataset. This limitation hinders the benefits of sequential updates. We introduce a novel and surprisingly simple PAC-Bayesian procedure, Recursive PAC-Bayes (RPB), that overcomes this issue by decomposing the expected loss of randomized classifiers into an excess loss relative to a downscaled prior loss plus the downscaled prior loss, allowing recursive bounding. Additionally, we generalize the split-kl and PAC-Bayes-split-kl inequalities to discrete random variables for bounding excess losses. Our empirical evaluations show that this new procedure significantly outperforms state-of-the-art methods.

In Chapter \ref{chap:chap4}, we explore the connections between Recursive PAC-Bayes (RPB) and two commonly used techniques in probabilistic modeling. The first technique, cold posteriors, characterized by a temperature scaling factor less than one, significantly improves the performance of Bayesian neural networks. The second technique, KL-annealing, gradually increases the weight of the KL divergence term during training, balancing the fit to empirical data with maintaining proximity between prior and posterior distributions. By linking these methodologies with Recursive PAC-Bayes, we provide new insights into its success and practical implications. 


In Chapter \ref{chap:chap5}, we introduce a novel decomposition of the (generalized) evidence lower bound (ELBO) objective, which is a special case of Equation \ref{eq:intro}, for (generalized) mean-field variational global latent variable models. This decomposition breaks the (generalized) ELBO objective into interpretable components, with the potential to allow for finer control of the temperature \( T \) for optimization and provide valuable insights into the training dynamics of the variational posteriors.

We conclude this work with a discussion of these results in Chapter \ref{chap:chap6}.

\section{Main Contributions}
The main contributions of this work are:
\begin{itemize}
    \item We theoretically demonstrate that the presence of the CPE implies the Bayesian posterior is underfitting. In fact, we show that if there is no underfitting, there is no CPE.
    \item We show in a more general case that any tempered posterior is a proper Bayesian posterior with an alternative likelihood and prior distribution, extending \cite{zeno2021why} in the case of classification. We also provide illustrative examples in both regression and classification settings to offer an intuitive and visual understanding of our findings.
    \item We demonstrate that likelihood misspecification and prior misspecification result in the cold posterior effect (CPE) only if they also induce underfitting. Furthermore, we provide comprehensive empirical evidence in both linear regression and neural network settings, where approximate inference is applied. Additionally, we discuss the implications of model size and sample size in relation to CPE and underfitting.
    \item We show that data augmentation results in stronger CPE because it induces a stronger underfitting of the Bayesian posterior.
    \item We show that fine-tuning the temperature $T$ in tempered posteriors offers a well-founded and effective Bayesian approach to mitigate underfitting: By fine-tuning the temperature $T$, we implicitly utilize alternative Bayesian posteriors resulting from less misspecified likelihood and prior distributions.
    \item We present a novel PAC-Bayesian bound, Recursive PAC-Bayes (RPB), that allows for meaningful and beneficial sequential data processing and sequential posterior updates, without loss of confidence information.
    \item We introduce a new and simple way to decompose the loss of a randomized classifier defined by the posterior into an excess loss relative to a downscaled loss of the classifier defined by the prior plus the downscaled loss of the prior.
    \item Importantly, this excess loss can be bounded using PAC-Bayes-Empirical-Bernstein-style inequalities, and the loss of the prior can be bounded recursively, preserving confidence information on the prior.
    \item We provide a generalization of the split-kl and PAC-Bayes-split-kl inequalities from ternary to general discrete random variables, based on a novel representation of discrete random variables as a superposition of Bernoulli random variables.
    \item We develop a practical and efficient optimization procedure based on the Recursive PAC-Bayes bound, resulting in posteriors that outperform state-of-the-art methods on benchmark datasets.
    \item We explore the connections between Recursive PAC-Bayes, cold posteriors, and KL-annealing, providing new insights into Recursive PAC-Bayes's effectiveness and empirical success.
    \item We introduce a novel decomposition of the (generalized) ELBO objective for mean-field variational global latent variable models, where latent variables (typically model or network weights) are shared across all data points. By breaking down the (generalized) ELBO into interpretable components, this approach might allow for finer control of the temperature \( T \) during optimization and provide a deeper understanding of the training dynamics of variational posteriors.
\end{itemize}

\chapter{The Cold Posterior Effect Indicates Underfitting, and Cold Posteriors Represent a Fully Bayesian Method to Mitigate It}
\label{chap:chap2}

The work presented in this chapter is based on a paper that has been published as: 
\\
\\
\bibentry{zhang2024the}. 

\newpage

\section*{Abstract}
The cold posterior effect (CPE) \citep{WRVS+20} in Bayesian deep learning shows that, for posteriors with a temperature $T<1$, the resulting posterior predictive could have better performance than the Bayesian posterior ($T=1$). As the Bayesian posterior is known to be optimal under perfect model specification, many recent works have studied the presence of CPE as a model misspecification problem, arising from the prior and/or from the likelihood. In this work, we provide a more nuanced understanding of CPE as we show that \emph{misspecification leads to CPE only when the resulting Bayesian posterior underfits}. In fact, we theoretically show that if there is no underfitting, there is no CPE. Furthermore, we show that these \emph{tempered posteriors} with $T < 1$ are indeed proper Bayesian posteriors with a different combination of likelihoods and priors parameterized by $T$. This observation validates the adjustment of the temperature hyperparameter $T$ as a straightforward approach to mitigate underfitting in the Bayesian posterior. In essence, we show that by fine-tuning the temperature $T$ we implicitly utilize alternative Bayesian posteriors, albeit with less misspecified likelihood and prior distributions. The code for replicating the experiments can be found at \url{https://github.com/pyijiezhang/cpe-underfit}.

\section{Introduction}\label{sec:intro}
In Bayesian deep learning, the cold posterior effect (CPE) \citep{WRVS+20} refers to the phenomenon in which if we artificially ``temper'' the posterior by either $p(\bmtheta|D)\propto (p(D|\bmtheta)p(\bmtheta))^{1/T}$ or $p(\bmtheta|D)\propto p(D|\bmtheta)^{1/T}p(\bmtheta)$ with a temperature $T<1$, the resulting posterior enjoys better predictive performance than the standard Bayesian posterior (with $T=1$). The discovery of the CPE has sparked debates in the community about its potential contributing factors.

If the prior and likelihood are properly specified, the Bayesian solution (i.e., $T=1$) should be optimal \citep{GCSD+13}, assuming approximate inference is properly working. Hence, the presence of the CPE implies either the prior \citep{WRVS+20,FGAO+22}, the likelihood \citep{Ait21,KMIW22}, or both are misspecified. This has been, so far, the main argument of many works trying to explain the CPE.

One line of research examines the impact of the prior misspecification on the CPE \citep{WRVS+20,FGAO+22}. The priors of modern Bayesian neural networks are often selected for tractability. Consequently, the quality of the selected priors in relation to the CPE is a natural concern. Previous research has revealed that adjusting priors can help alleviate the CPE in certain cases \citep{adlam2020cold, zeno2021why}, there are instances where the effect persists despite such adjustments \citep{FGAO+22}. Some studies even show that the role of priors may not be critical \citep{IVHW21}. Therefore, the impact of priors on the CPE remains an open question.

Furthermore, the influence of likelihood misspecification on CPE has also been investigated \citep{Ait21,NRBN+21,KMIW22,FGAO+22}, and has been identified to be particularly relevant in curated datasets \citep{Ait21,KMIW22}. Several studies have proposed alternative likelihood functions to address this issue and successfully mitigate the CPE \citep{NGGA+22,KMIW22}. However, the underlying relation between the likelihood and CPE remains a partially unresolved question. Notably, the CPE usually emerges when data augmentation (DA) techniques are used \citep{WRVS+20,IVHW21,FGAO+22,NRBN+21,NGGA+22,KMIW22}. A popular hypothesis is that using DA implies the introduction of a randomly perturbed log-likelihood, which lacks a clear interpretation as a valid likelihood function \citep{WRVS+20,IVHW21}. However, \citet{NGGA+22} demonstrates that the CPE persists even when a proper likelihood function incorporating DA is defined. Therefore, further investigation is needed to fully understand their relationship. 

Other works argued that CPE could mainly be an artifact of inaccurate approximate inference methods, especially in the context of neural networks, where the posteriors are extremely high dimensional and complex \citep{IVHW21}. However, many of the previously mentioned works have also found setups where the CPE either disappears or is significantly alleviated through the adoption of better priors and/or better likelihoods with approximate inference methods. In these studies, the same approximate inference methods were used to illustrate,  for example, how using a standard likelihood function leads to the observation of CPE and how using an alternative likelihood function removes it \citep{Ait21,NRBN+21,KMIW22}. In other instances, under the same approximate inference scheme, CPE is observed when using certain types of priors but it is strongly alleviated when an alternative class of priors is utilized \citep{WRVS+20,FGAO+22}. Therefore, there is compelling evidence suggesting that approximate methods are not, at least, a necessary condition for the CPE.

This study, both theoretically and empirically, demonstrates that the presence of the cold posterior effect (CPE) implies the existence of underfitting; in other words, \textit{if there is no underfitting, there is no CPE}. Integrating this perspective with previous findings suggesting that CPE indicates misspecified likelihood, prior, or both \citep{GCSD+13}, we conclude that CPE implies both misspecification and underfitting. Consequently, mitigating CPE necessitates addressing both aspects. Notably, simplifying the issue by solely focusing on misspecification is insufficient, as misspecification can lead Bayesian methods to both underfitting and overfitting \citep{domingos2000bayesian,immer2021improving,KMIW22}; CPE only arises when underfitting occurs. This study thus offers a nuanced perspective on the factors contributing to CPE. Additionally, by building on \cite{zeno2021why}, we show how tempered posteriors represent proper Bayesian posteriors under different likelihood and prior distributions, jointly parameterized by the temperature parameter $T$. Consequently, by adjusting $T$, we effectively identify Bayesian posteriors with less misspecified likelihood and prior distributions, leading to a more accurate representation of the training data and improved generalization performance. Furthermore, we delve into the relationship between prior/likelihood misspecification, data augmentation, approximate inference, and CPE, offering insights into potential strategies for addressing these issues.

\paragraph*{Contributions}
\textbf{(i)} We theoretically demonstrate that the presence of the CPE implies the Bayesian posterior is underfitting in Section \ref{sec:CPE}. \textbf{(ii)} We show in a more general case that any tempered posterior is a proper Bayesian posterior with an alternative likelihood and prior distribution in Section \ref{sec:Bayesian}, extending \cite{zeno2021why} in the case of classification. \textbf{(iii)} We show in Section \ref{sec:modelmisspec&CPE} that likelihood misspecification and prior misspecification result in CPE only if they also induce underfitting. Furthermore, the tempered posteriors offer an effective and well-founded Bayesian mechanism to address the underfitting problem. \textbf{(iv)} Finally, we show that data augmentation results in stronger CPE because it induces a stronger underfitting of the Bayesian posterior in Section \ref{sec:data-augmentation}. In conclusion, our theoretical analysis reveals that the occurrence of the CPE signifies underfitting of the Bayesian posterior. Also, fine-tuning the temperature in tempered posteriors offers a well-founded and effective Bayesian approach to mitigate the issue. Furthermore, our work aims to settle the debate surrounding CPE and its implications for Bayesian principles, specifically within the context of deep learning. 

\section{Background}\label{sec:background}\subsection{Notation and assumptions}

Let us start by introducing basic notation. Consider a supervised learning problem with the sample space $\cal Y\times \cal X$. In this work, we consider two cases: when $\cal Y$ is a finite set, corresponding to a supervised classification problem, and when $\cal Y$ is a subset of $\R$, corresponding to a regression problem. For simplicity, we also assume that $\cal X$ is a subset of $\R^d$. Let the training set $D=(\bmY, \bmX)$, where $\bmY$ denotes the set of output entries and $\bmX$ denotes the set of input entries. If $D$ consists $n$ pairs of samples, we denote $D=\{(\bmy_i, \bmx_i)\}_{i=1}^n$.

We assume a family of probabilistic models parameterized by $\bmTheta$, where each $\bmtheta\in\bmTheta$ defines a conditional probability distribution for a sample $(\bmy,\bmx)$, denoted by $p(\bmy | \bmx, \bmtheta)$. The standard metric to measure the quality of a probabilistic model $\bmtheta$ on a sample $(\bmy,\bmx)$ is the (negative) log-loss $-\ln p(\bmy|\bmx,\bmtheta)$. The expected (or population) loss of a probabilistic model $\bmtheta$ is defined as $L(\bmtheta)=\mathbb{E}_{(\bmy,\bmx)\sim \nu}[-\ln p(\bmy|\bmx,\bmtheta)]$, where $\nu$ denotes the unknown data-generating distribution $\nu$ on $\cal Y\times \cal X$. The empirical loss of the model $\bmtheta$ on the data $D$ is defined as $\hat{L}(D,\bmtheta) = -\tfrac{1}{n}\ln p(\bmY|\bmX,\bmtheta) = -\tfrac{1}{n}\sum_{i\in[n]} \ln p(\bmy_i| \bmx_i, \bmtheta)$. In this work, we assume that the likelihood function fully factorizes, i.e., $p(\bmY|\bmX,\bmtheta)=\prod_{(\bmy,\bmx)\in D} p(\bmy|\bmx,\bmtheta)$.  
We might use the notation $p(D|\bmtheta)$ for $p(\bmY|\bmX,\bmtheta)$ in the presentation when the roles of input/output in the samples are not important in the context.
Also, if it induces no ambiguity, we use $\mathbb{E}_\nu[\cdot]$ as a shorthand for $\mathbb{E}_{(\bmy,\bmx)\sim \nu}[\cdot]$.

\subsection{(Generalized) Bayesian learning}\label{sec:bayesian-learning}

In Bayesian learning, we learn a probability distribution $\rho(\bmtheta|D)$, often called a posterior, over the parameter space $\bmTheta$ from the training data $D$. Given a new input $\bmx$, the posterior $\rho$ makes the prediction about $\bmy$ through (an approximation of) \textit{Bayesian model averaging (BMA)} $p(\bmy|\bmx,\rho)=\mathbb{E}_{\bmtheta\sim\rho}[p(\bmy|\bmx,\bmtheta)]$, where the posterior $\rho$ is used to combine the predictions of the models. Again, if it induces no ambiguity, we use $\mathbb{E}_\rho[\cdot]$ as a shorthand for $\mathbb{E}_{\bmtheta\sim \rho}[\cdot]$. The predictive performance of such BMA is usually measured by the Bayes loss, defined by 
\begin{equation}
B(\rho)=\E_\nu[-\ln \E_\rho[p(\bmy|\bmx,\bmtheta)]]\,.    
\end{equation}

For some $\lambda>0$ and a prior $p(\bmtheta)$, the so-called \textit{tempered posteriors} (or the generalized Bayes posterior) \citep{BC91,Zhang06,bissiri2016general_6,GVO17}, are defined as a probability distribution
\begin{equation}
p_\lambda(\bmtheta|D) \propto p(\bmY|\bmX,\bmtheta)^\lambda p(\bmtheta)\,. \label{eq:likelihoodTempering}
\end{equation}
\noindent 
Note that when $\lambda\neq 1$, $\int p(\bmY|\bmX,\bmtheta)^\lambda d \bmY$ might not be 1 in general. An implicit assumption is that \( p_\lambda(\bmtheta | D) \) is a \textit{proper distribution}, meaning the normalization constant is finite. In supervised classification problems, this is always the case because \( p(\bmY|\bmX,\bmtheta) \leq 1 \). Consequently, for any \(\lambda > 0\), we have
\(
1 = \int p(\bmtheta) \, d\bmtheta > \int p(\bmY|\bmX,\bmtheta)^\lambda p(\bmtheta) \, d\bmtheta.
\)
Thus, the tempered posteriors are always a proper distribution in supervised classification problems. 

Even though many works on CPE use the parameter $T=1/\lambda$ instead, we adopt $\lambda$ in the rest of the work for the convenience of derivations. Therefore, the CPE ($T<1$) corresponds to when $\lambda>1$. We also note that while some works study CPE with a full-tempering posterior, where the prior is also tempered, many works also find CPE for likelihood-tempering posterior (see \citep{WRVS+20} and the references therein). Also, with some widely chosen priors (e.g., zero-centered Gaussian priors), the likelihood-tempering posteriors are equivalent to full-tempering posteriors with rescaled prior variances \citep{Ait21,BNH22}.

When $\lambda=1$, the tempered posterior equals the (standard) Bayesian posterior. The tempered posterior can be obtained by optimizing a generalization of the so-called (generalized) ELBO objective \citep{alquier2016properties,HMPB+17}, which, for convenience, we write as follows:
\begin{eqnarray}\label{eq:likelihoodTempering:ELBO}
p_\lambda(\bmtheta|D) = \argmin_\rho \E_\rho[-\ln p(D|\bmtheta)] + \frac{1}{\lambda}\operatorname{KL}(\rho(\bmtheta|D),p(\bmtheta))\,.
\end{eqnarray}
The first term is known as the (un-normalized) \textit{reconstruction error} or the empirical Gibbs loss of the posterior $\rho$ on the data $D$, denoted as $\hat G(\rho, D)=\E_\rho[-\tfrac{1}{n}\ln p(D|\bmtheta)]$, which further equals to $\E_\rho[\hat L(D,\bmtheta)]$. Therefore, it is often used as the \textit{training loss} in Bayesian learning \citep{DBLP:conf/aistats/MorningstarAD22}. 
The second term is a Kullback-Leibler divergence between the posterior $\rho(\bmtheta|D)$ and the prior $p(\bmtheta)$ scaled by a hyper-parameter $\lambda$. 


As it induces no ambiguity, we will use $p_\lambda$ as a shorthand for $p_\lambda(\bmtheta|D)$. So, for example, $B(p_\lambda)$ would refer to the expected Bayes loss of the tempered-posterior $p_\lambda(\bmtheta|D)$. In the rest of this work, we will interpret the CPE as how changes in the parameter $\lambda$ affect the \textit{test error} and the \textit{training error} of $p_\lambda$ or, equivalently, the Bayes loss $B(p_\lambda)$ and the empirical Gibbs loss $\hat G(p_\lambda,D)$.

\section{The presence of the CPE implies underfitting}\label{sec:CPE}

\begin{Overview}
We present a definition of the Cold Posterior Effect (CPE) (Definition \ref{def:likelihoodtempering:CPE}) and show that the presence of CPE indicates the Bayesian posterior is underfitting, where both the testing loss (Definition \ref{def:likelihoodtempering:CPE}) and training loss (Proposition \ref{def:likelihoodtempering:CPE}) can be improved at the same time by decreasing the temperature $T$ (increasing $\lambda$). We also present the necessary condition of the CPE in Proposition \ref{thm:likelihoodtempering:CPE:Neccesary} and the case when Bayesian posterior is optimal in Theorem \ref{cor:likelihoodtempering:bayesposterior_optimality}.
\end{Overview}

\noindent
A standard understanding for underfitting refers to a situation when the trained model cannot properly capture the relationship between input and output in the data-generating process, resulting in high errors on both the training data and testing data. In the context of highly flexible model classes such as neural networks, underfitting refers to a scenario where the trained model exhibits (much) higher training and testing losses compared to what is achievable. Essentially, it means that there exists another model in the model class that achieves lower training and testing losses simultaneously. In the context of Bayesian inference, we argue that the Bayesian posterior is underfitting if there exists another posterior distribution with lower empirical Gibbs and Bayes losses at the same time. In fact, we will show later in Section \ref{sec:Bayesian} that such a posterior is essentially another \emph{Bayesian posterior but with a different prior and likelihood function}. Before delving into that, we focus on characterizing the cold posterior effect (CPE) and its connection to underfitting.

As previously discussed, the CPE describes the phenomenon of getting better predictive performance when we make the parameter of the tempered posterior, $\lambda$,  higher than 1. The next definition introduces a formal characterization. \textit{We do not claim this is the best possible formal characterization}. However, through the rest of the paper, we will show that this simple characterization is enough to understand the relationship between CPE and underfitting. 

\begin{definition}\label{def:likelihoodtempering:CPE}
We say there is a CPE for Bayes loss if and only if the derivative of the Bayes loss of the posterior $p_\lambda$, $B(p_\lambda)$, evaluated at $\lambda=1$ is negative. That is, 
\begin{equation}
\frac{d}{d\lambda}B(p_\lambda)_{|\lambda=1}< 0\,,
\end{equation}
where the magnitude of the derivative $\frac{d}{d\lambda} B(p_\lambda)_{|\lambda=1}$ defines the strength of the CPE. 
\end{definition}

According to the above definition, a (relatively large) negative derivative $\frac{d}{d\lambda} B(p_\lambda)_{|\lambda=1}$ implies that by making $\lambda$ slightly greater than 1, we will have a (relatively large) reduction in the Bayes loss with respect to the Bayesian posterior. Note that if the derivative  $\frac{d}{d\lambda} B(p_\lambda)_{|\lambda=1}$ is not relatively large and negative, then we can not expect a relatively large reduction in the Bayes loss and, in consequence, the CPE will not be significant. Obviously, this formal definition could also be extended to other specific $\lambda$ values different from $1$, or even consider some aggregation over different $\lambda>1$ values. We will stick to this definition because it is simpler, and the insights and conclusions extracted here can be easily extrapolated to other similar definitions involving the derivative of the Bayes loss.

Next, we present another critical observation. We postpone the proofs in this section to Appendix~\ref{app:sec:cpe-underfitting}.

\begin{proposition} \label{prop:likelihoodtempering:empiricalGibbsLoss}
The derivative of the empirical Gibbs loss of the tempered posterior $p_\lambda$ satisfies
\begin{equation}\label{eq:likelihoodtempering:empiricalgibbsgradient}
        \forall\lambda\geq 0\quad \frac{d}{d\lambda} \hat{G}(p_\lambda,D) = -\mathbb{V}_{p_\lambda}\big(\ln p(D|\bmtheta)\big) \leq 0\,,
\end{equation}
where $\mathbb{V}(\cdot)$ denotes the variance.
\end{proposition}
As shown in Proposition~\ref{prop:constant_variance} in Appendix \ref{app:sec:cpe-underfitting}, to achieve $\mathbb{V}_{p_\lambda}\big(\ln p(D|\bmtheta)\big)=0$, we need $p_\lambda(\bmtheta|D)=p(\bmtheta)$, implying that the data has no influence on the posterior. In consequence, in practical scenarios, $\mathbb{V}_{p_\lambda}\big(\ln p(D|\bmtheta)\big)$ will always be greater than zero. Thus, increasing $\lambda$ will monotonically reduce the empirical Gibbs loss $\hat{G}(p_\lambda,D)$ (i.e., the \textit{train error}) of $p_\lambda$. The next result also shows that the empirical Gibbs loss of the Bayesian posterior $\hat{G}(p_{\lambda=1})$ cannot reach its minimum to observe the CPE.

\begin{proposition}\label{thm:likelihoodtempering:CPE:Neccesary}
A necessary condition for the presence of the CPE, as defined in Definition \ref{def:likelihoodtempering:CPE}, is that
\[ \hat{G}(p_{\lambda=1},D)> \min_\bmtheta -\ln p(D|\bmtheta)\,. \]
\end{proposition}

\begin{insight} Definition~\ref{def:likelihoodtempering:CPE} in combination with Proposition~\ref{prop:likelihoodtempering:empiricalGibbsLoss} shows if the CPE is present, by making $\lambda > 1$, the test loss $B(p_\lambda)$ and the empirical Gibbs loss $\hat{G}(p_\lambda,D)$ will be reduced at the same time.  Furthermore, Proposition \ref{thm:likelihoodtempering:CPE:Neccesary} states that the Bayesian posterior \textit{still has room} to fit the training data further (e.g., by placing more probability mass on the maximum likelihood estimator). We hence deduce that the presence of CPE implies that the original Bayesian posterior ($\lambda=1$) underfits. This conclusion arises because there exists another Bayesian posterior (i.e, $p_\lambda(\bmtheta|D)$ with $\lambda>1$) that has lower training (Proposition \ref{thm:likelihoodtempering:CPE:Neccesary}) and testing (Definition~\ref{def:likelihoodtempering:CPE}) loss at the same time. Further elaboration on the nature of $p_\lambda(\bmtheta|D)$ as another Bayesian posterior will be provided later in Section \ref{sec:Bayesian}. In short, if there is CPE, the original Bayesian posterior is underfitting. Or, equivalently, if the original Bayesian posterior does not underfit, there is no CPE.
\end{insight}

However, a final question arises: when is $\lambda=1$  (the original Bayesian posterior of interest) \textit{optimal}? More precisely, when does the derivative of the Bayes loss with respect to $\lambda$ evaluated at $\lambda=1$ become zero ($\frac{d}{d\lambda} B(p_\lambda)_{|\lambda=1}=0$)? This would imply that neither (infinitesimally) increasing nor decreasing $\lambda$ changes the predictive performance. We will see that this condition is closely related to the situation that updating such a Bayesian posterior with more data does not enhance its fit to the original training data better. In other words, the extra information about the data-generation process does not provide the Bayesian posterior with better performance on the originally provided training data.

We start by denoting $\tilde p_\lambda (\bmtheta|D,(\bmy,\bmx))$ as the distribution obtained by updating the posterior $p_\lambda(\bmtheta|D)$ with one new sample $(\bmy,\bmx)$, i.e., $\tilde p_\lambda (\bmtheta|D,(\bmy,\bmx)) \propto p(\bmy|\bmx,\bmtheta)p_\lambda(\bmtheta|D)$. And we also denote $\bar{p}_\lambda$ as the distribution resulting from averaging $\tilde p_\lambda (\bmtheta|D,(\bmy,\bmx))$ over different \textit{unseen} samples from the data-generating distribution $(\bmy,\bmx)\sim\nu(\bmy,\bmx)$:
\begin{equation}\label{eq:likelihoodtempering:updatedposterior}
    \bar{p}_\lambda(\bmtheta|D) = \E_\nu\left[\tilde p_\lambda(\bmtheta|D,(\bmy,\bmx))\right].
\end{equation}

In this sense, $\bar{p}_\lambda$ represents how the posterior $p_{\lambda}$ would be, on average, after being updated with a new sample from the data-generating distribution. This updated posterior contains a bit more information about the data-generating distribution, compared to \(p_\lambda\). Using the updated posterior $\bar{p}_\lambda$, the following result introduces a characterization of the \emph{optimality} of the original Bayesian posterior.

\begin{theorem} \label{cor:likelihoodtempering:bayesposterior_optimality}
The derivative of the Bayes loss at $\lambda=1$ is null, i.e., $\frac{d}{d\lambda} B(p_\lambda)_{|\lambda=1}=0$, if and only if, 
\[  \hat{G}(p_{\lambda=1},D) = \hat{G}(\bar{p}_{\lambda=1},D)\,. \]
\end{theorem}

\begin{insight}
The original Bayesian posterior of interest is optimal if after updating it using the procedure described in Equation~\ref{eq:likelihoodtempering:updatedposterior}, or in other words, after exposing the Bayesian posterior to more data from the data-generating distribution, the empirical Gibbs loss over the initial training data remains unchanged. 
\end{insight} 
We will give examples that empirically illustrate Theorem \ref{cor:likelihoodtempering:bayesposterior_optimality} and the induced insight later in Section \ref{sec:5.5}.

\section{Tempered posteriors are Bayesian posteriors}
\label{sec:Bayesian}

\begin{Overview}
By extending \cite{zeno2021why} on classification only, we show in general that tempered posteriors are proper Bayesian posteriors with an alternative combination of likelihood and prior functions parameterized by $\lambda$. Thus, the occurrence of CPE can be explained within the Bayesian framework.
\begin{itemize}
    \item  We provide two examples to show how $\lambda$ influences the new likelihoods in Section \ref{sec:likelihood-wrt-lambda} and two examples to show how $\lambda$ influences the new priors in Section \ref{sec:prior-wrt-lambda}. 
    \item We show in Section \ref{sec:generalized-elbo-are-elbo} that the generalized ELBOs are also proper ELBOs. 
    \item We expand the discussion of the implications in Section \ref{sec:sec4-insight}.
\end{itemize}
\end{Overview}

\noindent
As previously discussed, the CPE phenomenon involves achieving improved predictive accuracy by employing a tempered posterior. A potential criticism is that this tempered posterior does not strictly adhere to the principles of a proper Bayesian posterior because the tempered likelihood, \(P(D|\bmtheta)^\lambda\) fails to meet the criteria of a proper likelihood function when $\lambda\neq 1$ (i.e., $\int P(D|\bmtheta)^\lambda dD \neq 1$ when $\lambda\neq 1$). However, as previously discussed by \cite{, zeno2021why}, this tempered posterior effectively serves as a \textit{proper Bayesian posterior} with a combination of \textit{new likelihood and prior functions}. We extend this result beyond classification to our Proposition \ref{prop:true_posterior}, proved in Appendix \ref{app:prop:true_posterior}.

Before delving into the description of the new likelihood and prior functions, it is essential to acknowledge a fundamental aspect. Given a labeled dataset \(D = (\bmX, \bmY)\) and the conditional likelihood associated to a classification model, the application of Bayes' theorem naturally results in the following Bayesian posterior: $$ p(\bmtheta|\bmX, \bmY)\propto p(\bmY|\bmX,\bmtheta) p(\bmtheta|\bmX),$$
\noindent where the prior over $\bmtheta$ is a \textit{conditional prior} \citep{marek2024can, zeno2021why} that depends on the unlabelled training data $\bmX$. However, specifying \(p(\bmtheta|\bmX)\) for a complex model, like a deep neural network, poses a significant challenge. Therefore, for practical purposes, nearly all existing works \citep{WRVS+20,FGAO+22} assume $\bmtheta$ to be independent of $\bmX$, resulting in the simplified expression $p(\bmtheta|\bmX, \bmY) \propto p(\bmY|\bmX,\bmtheta) p(\bmtheta)$, where the prior over $\bmtheta$ is now an \textit{unconditional prior}.

\begin{proposition}\label{prop:true_posterior}
    For any given dataset \(D = (\bmX, \bmY)\) such that the likelihood fully factories $p(\bmY|\bmX,\bmtheta)=\prod_{(\bmy,\bmx)\in D} p(\bmy|\bmx,\bmtheta)$, and \(\lambda > 0\), the tempered posterior defined in Equation \ref{eq:likelihoodTempering} can be expressed as a Bayesian posterior with a new prior and likelihood function as follows:
    \begin{equation}
        p_\lambda(\bmtheta|\bmX, \bmY) \propto q(\bmtheta|\bmX, \lambda) \prod_{(\bmy,\bmx)\in D} q(\bmy|\bmx,\bmtheta,\lambda), 
    \end{equation}
    where the new prior distribution \(q(\bmtheta|\bmX, \lambda)\) and likelihood function \(q(\bmy|\bmx, \bmtheta, \lambda)\) are defined as:
    \begin{equation}\label{eq:newlikelihoodprior}
    q(\bmtheta|\bmX,\lambda) \propto p(\bmtheta) \prod_{\bmx\in\bmX}  \int p(\bmy|\bmx,\bmtheta)^\lambda d\bmy, \quad\quad\quad  q(\bmy|\bmx,\bmtheta,\lambda) = \frac{p(\bmy|\bmx,\bmtheta)^\lambda}{\int p(\bmy|\bmx,\bmtheta)^\lambda d\bmy}.   
    \end{equation}
\end{proposition}

Note that the new conditional likelihood $q(\bmy|\bmx,\bmtheta,\lambda)$ and the new prior $q(\bmtheta|\bmX,\lambda)$ are both parametrized by the same $\lambda>0$, and note that the prior only depends on the unlabelled training data $\bmX$.

\cite{adlam2020cold} shows that in the specific scenario of Gaussian process regression (essentially our Bayesian linear regression example in Figures \ref{fig:blr} and \ref{fig:blr2}), any positive temperature aligns with a legitimate posterior under an adjusted unconditional prior. This can be seen as a special case of our argument.

In the rest of the section, we will discuss in Section \ref{sec:likelihood-wrt-lambda} how the new likelihoods change with respect to $\lambda$, and in Section \ref{sec:prior-wrt-lambda} how the new priors change with respect to $\lambda$. Additionally, besides demonstrating that tempered posteriors are proper Bayesian posteriors, we also show in Section \ref{sec:generalized-elbo-are-elbo} that the generalized ELBOs are also proper ELBOs. Lastly, we discuss the implications of the results in Section \ref{sec:sec4-insight}.

\subsection{How $\lambda$ influences the new likelihoods}\label{sec:likelihood-wrt-lambda}

The next result, proved in Appendix \ref{app:entropyproof}, shows that in supervised classification settings, higher $\lambda$ values induce new likelihood distributions with lower aleatoric uncertainty or, equivalently, lower Shannon entropy, denoted as $H(q(\bmy|\bmx,\bmtheta,\lambda)):=-\sum_{\bmy\in{\cal Y}} q(\bmy|\bmx,\bmtheta,\lambda)\ln q(\bmy|\bmx,\bmtheta,\lambda)$. 
\begin{proposition}\label{prop:entropy} For any $\bmtheta\in \bmTheta$, any $\bmx\in{\cal X}$, and any finite output set $\cal Y$, the entropy of the conditional likelihood $q(\bmy|\bmx,\bmtheta,\lambda)$ monotonically decreases with $\lambda>0$, i.e.,
\begin{equation} 
    \frac{d}{d\lambda} H(q(\bmy|\bmx,\bmtheta,\lambda))\leq 0\quad \forall\lambda>0\,.
\end{equation}
\end{proposition}
This result also holds for regression settings, where $\mathcal{Y}\subset \R^d$, under the differential entropy, assuming the Leibniz rule holds. See the proof for a detailed discussion on the matter.

We give two concrete examples, one in regression and one in classification, to illustrate the proposition.

\begin{figure}
    \centering 
\begin{subfigure}{0.32\textwidth}
    \centering
    \hspace*{0.5cm}\( q(\bmy|\bmx,\bmtheta,\lambda)\)\\
  \includegraphics[width=\linewidth]{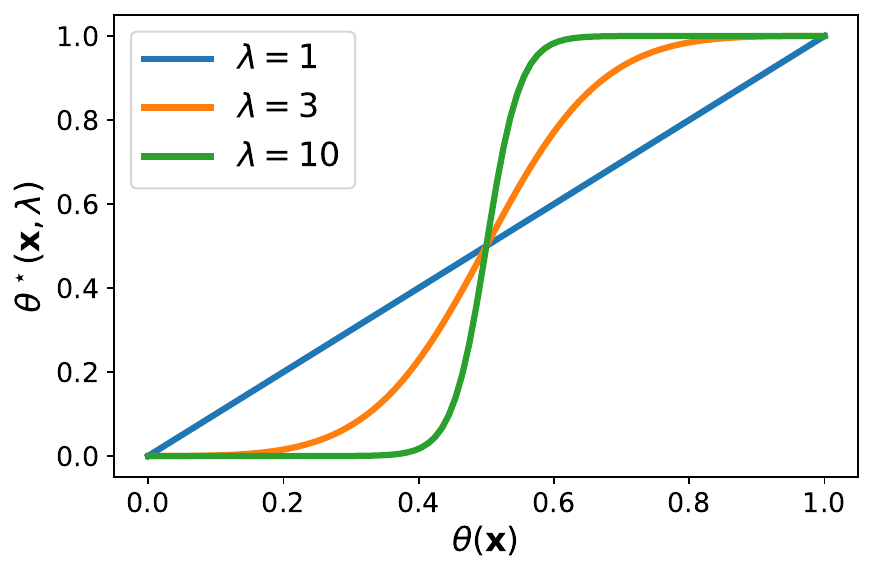}
\end{subfigure}\hfil 
\begin{subfigure}{0.32\textwidth}
\centering
\hspace*{0.5cm}\(q(\bmtheta|\bmX,\lambda)\)\\
  \includegraphics[width=\linewidth]{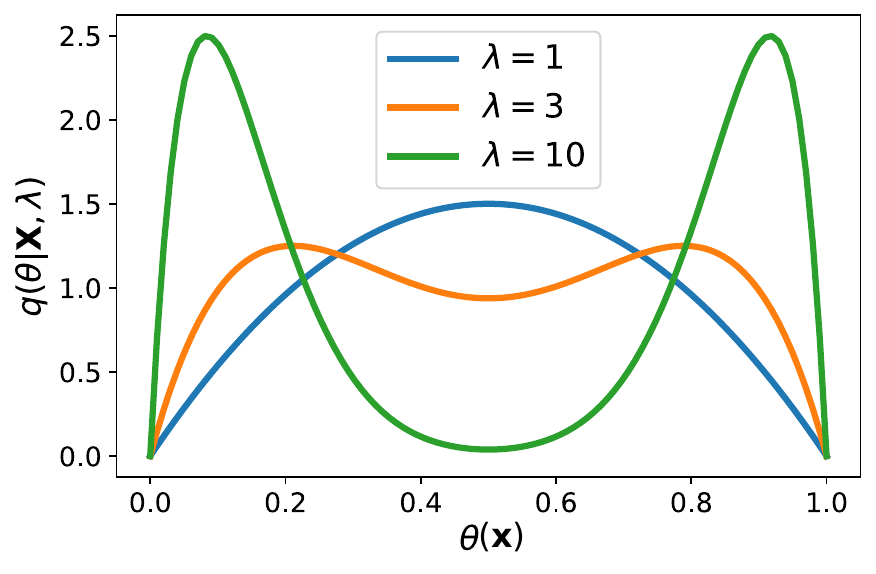}
\end{subfigure}\hfil\begin{subfigure}{0.32\textwidth}
\centering
\hspace*{0.5cm}\(q(\bmtheta|\bmX,\lambda)\)\\
  \includegraphics[width=\linewidth]{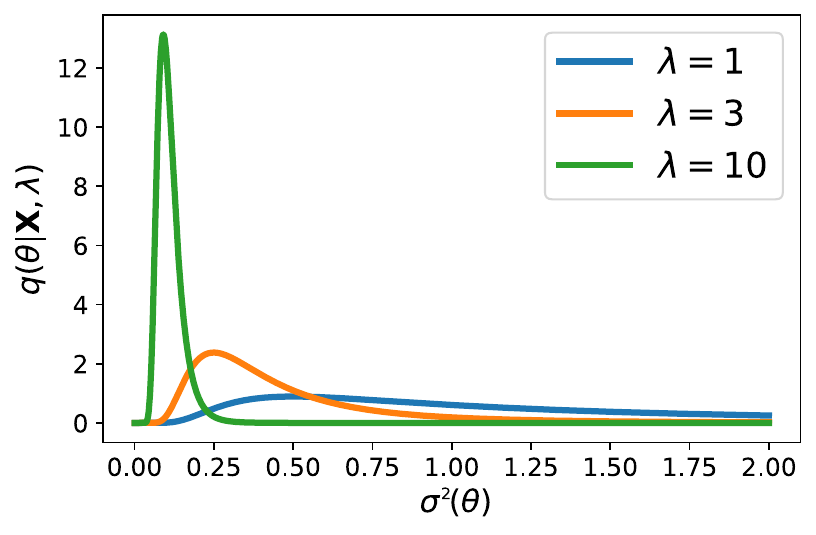}
\end{subfigure}\hfil
\caption{Illustration of the new likelihood \( q(\bmy|\bmx,\bmtheta,\lambda)\) and priors \(q(\bmtheta|\bmX,\lambda)\). In the left and middle figures, the original likelihood is in the form of the Bernoulli distribution. The left figure demonstrates the transformation from \(\theta(\bmx)\) to  \(\theta^\star(\bmx, \lambda) := \frac{\theta(\bmx)^{\lambda}}{\theta(\bmx)^{\lambda} + (1-\theta(\bmx))^\lambda}\). In the middle figure, we display a Beta-Binomial example, where the prior, initialized as a Beta distribution, is updated with a single Bernoulli-distributed sample. In the right figure, we display the new prior, initialized as an inverse-gamma prior and updated with a Gaussian likelihood with a single observation.}
\label{fig:binomial}
\end{figure}

\paragraph*{Regression example}\label{exam:likelihood-regression}\textit{Consider the case where the original likelihood is Gaussian, defined as \(p(\bmy|\bmx, \bmtheta) = \mathcal{N}(\mu(\bm{x}, \bmtheta), \sigma^2(\bmtheta))\), where the variance is input-independent, as typically seen in many regression problems. Then, following Equation \ref{eq:newlikelihoodprior}, the new likelihood corresponds to a scaling in the variance, given by \(q(\bmy|\bmx, \bmtheta, \lambda) = \mathcal{N}(\mu(\bm{x}, \bmtheta), \tfrac{\sigma^2(\bmtheta)}{\lambda^2})\). Thus, as $\lambda$ increases, the tempered likelihood \(q(\bmy|\bmx, \bmtheta, \lambda)\) induces a proper Gaussian likelihood with reduced variance, i.e., a new likelihood with lower aleatoric uncertainty.} 

\paragraph*{Classification example}\label{exam:likelihood-classification}\textit{Consider the case of a binary classification problem where the original conditional likelihood is Bernoulli, defined as \(p(\bmy|\bmx, \bmtheta) = \theta(\bmx)^{y}(1-\theta(\bmx))^{1-y}\) with $y\in\{0,1\}$ and the input-dependent parameter function $\theta(\bmx)\in [0,1]$, which is usually implemented by a neural network with a softmax activation function in the last layer. Then, following Equation \ref{eq:newlikelihoodprior}, the new conditional likelihood \(q(\bmy|\bmx,\bmtheta,\lambda) = \theta^*(\bmx,\lambda)^{y}(1-\theta^*(\bmx,\lambda))^{1-y}\) also follows a Bernoulli distribution with a different parameter function \(\theta^*(\bmx, \lambda) = \frac{\theta(\bmx)^{\lambda}}{\theta(\bmx)^{\lambda} + (1-\theta(\bmx))^\lambda}\in[0,1]\). The function  \(\theta^*(\bmx, \lambda)\) is displayed in Figure \ref{fig:binomial} (left). When $\lambda$ increases, the parameter function that defines the new Bernoulli likelihood becomes more extreme, resulting in a new likelihood with lower aleatoric uncertainty. }

In both cases, we see that as suggested by Proposition \ref{prop:entropy}, as $\lambda$ increases, the new conditional likelihoods $q(\bmy|\bmx,\bmtheta,\lambda)$ have lower entropy, i.e., lower aleatoric uncertainty. In Sections \ref{sec:5.2} and \ref{sec:5.3}, we will further explore the implications of this finding and its connection to existing literature.

\subsection{How $\lambda$ influences the new priors}\label{sec:prior-wrt-lambda}

On the other hand, according to Proposition \ref{prop:true_posterior}, using the tempered posteriors implies implicitly using the prior $q(\bmtheta|\bmX,\lambda)$. Such prior depends on the \emph{unlabelled training data} $\bmX$. On top of that, the functional form of the likelihood function is defined by the probabilistic model family through the term $ \int p(\bmy|\bmx,\bmtheta)^\lambda d\bmy$ for $\bmx\in\bmX$. Hence, models $\bmtheta$ that yield a large value for this term across most of the training data $\bmx\in\bmX$ will be assigned larger probability mass by the new prior. We will showcase this effect in both regression and binary classification problems. Moreover, we will see how the new prior $q(\bmtheta|\bmX,\lambda)$ with $\lambda>1$ \textit{favors those models within the model class that yield likelihoods with lower aleatoric uncertainty on the training data $\bmX$}. 

\paragraph*{Regression example}\label{exam:prior-Regression} \emph{Consider the case where the original likelihood is Gaussian, defined as \(p(\bmy|\bmx, \bmtheta) = \mathcal{N}(\mu(\bm{x}, \bmtheta), \sigma^2(\bmtheta))\), where the variance is input-independent, as typically seen in many regression problems. A common parametrization involves $\bmtheta=(\bmw,\gamma)$, where $\bmw$ refer to the weights of the neural network defining the function $\mu(\bm{x},\bmtheta)$ and $\gamma>0$ is a parameter encoding the variance of the Gaussian likelihood such that $\sigma^2(\bmtheta)=\gamma$. The prior $p(\bmtheta)$ is then defined as $p(\bmtheta)=p(\bmw)p(\gamma)$, where $p(\bmw)$ is usually a  Gaussian distribution with a diagonal covariance matrix, and $p(\gamma)$ is usually defined in terms of an inverse-gamma distribution. Following Equation \ref{eq:newlikelihoodprior}, the new prior would be expressed as \(q(\bmtheta|\bmX,\lambda)=q(\bmw|\bmX,\lambda)q(\gamma|\bmX,\lambda)\), where each term}
    \[q(\bmw|\bmX,\lambda)=p(\bmw)\quad,\quad q(\gamma|\bmX,\lambda) \propto p(\gamma)/\gamma^{n(\lambda-1)}\,.\]
    \emph{Figure~\ref{fig:binomial} (right) plots the density of $q(\gamma|\bmX,\lambda)$ when only one data is observed, with various $\lambda>1$ values when $p(\gamma)$ is an inverse-gamma prior.  For larger $\lambda$ values, this new prior will assign more probability mass to models defining a likelihood with smaller variance or, equivalently, smaller aleatoric uncertainty.} 

\paragraph*{Classification example}\label{exam:prior-Classification}\textit{Consider another case where the original conditional likelihood is Bernoulli, defined as \(p(\bmy|\bmx, \bmtheta) = \theta(\bmx)^{y}(1-\theta(\bmx))^{1-y}\) with $y\in\{0,1\}$ and $\theta(\bmx)\in [0,1]$, as commonly used in binary classification problems. Also, take any prior $p(\bmtheta)$. Then, following Equation \ref{eq:newlikelihoodprior}, the new prior is expressed as
        \[q(\bmtheta|\bmX,\lambda) \propto p(\bmtheta) \prod_{\bmx\in\bmX} \Big(\theta(\bmx)^\lambda + (1-\theta(\bmx))^\lambda\Big)\,. \]
Figure \ref{fig:binomial} (middle) illustrates the transformation of the prior for a Beta-Binomial model with a single training sample. Initially, the prior $p(\bmtheta)$ is taken as a Beta distribution, while the likelihood of this single data is Bernoulli. As $\lambda\ge 1$ increases, the new prior assigns more probability mass to models where $\theta(\bmx)$ is close to either $1$ or $0$. In other words, this new prior assigns more probability mass to models that assign more extreme probabilities to the training data (i.e., models with lower aleatoric uncertainty). Note that the prior does not consider how accurately these models classify the training data, but only the extremity of the probabilities assigned to the training data.}

We will discuss in Section \ref{sec:5.4} further implications of this finding and how it relates to the literature.

\subsection{Generalized ELBOs are also proper ELBOs}\label{sec:generalized-elbo-are-elbo}
Generalized ELBOs, characterized by scaling the KL divergence term using a hyper-parameter $\lambda$, have found widespread application in many studies \citep{WRVS+20}. This popularity stems from the demonstrated ability to adjust $\lambda$ to improve the predictive accuracy of variational approximations:
\begin{equation}\label{eq:pseudoELBO}
q^\star_\lambda := \argmin_{r\in \Pi} \E_r[-\ln p(D|\bmtheta)] + \frac{1}{\lambda}\operatorname{KL}(r(\bmtheta),p(\bmtheta))\,,
\end{equation}
\noindent where $\Pi$ defines the variational family. Critics have pointed out a flaw in the above generalized ELBO when $\lambda$ deviates from $1$, as it no longer functions as a true lower bound for the marginal likelihood. However, Proposition \ref{prop:true_posterior} can be used to justify that such a variational posterior $q^\star_\lambda$ still emerges from minimizing a valid ELBO. Specifically, it is constructed based on the revised likelihood and prior functions as follows: 
\begin{equation}\label{eq:pseudoELBO:fixed}
q^\star_\lambda = \argmin_{r\in \Pi} \E_r[-\ln q(\bmY|\bmX,\bmtheta,\lambda)] + \operatorname{KL}(r(\bmtheta),q(\bmtheta|\bmX, \lambda)) \,.
\end{equation}
Consequently, this analysis shows that using generalized ELBOs as Equation \ref{eq:pseudoELBO} perfectly adheres to variational and Bayesian principles. 

\subsection{Insights and implications from the section}\label{sec:sec4-insight}

In this section, we show that employing tempered posteriors seamlessly fits within a Bayesian framework, which streamlines and enriches the use of diverse likelihood and prior functions.

With the characterization in Section \ref{sec:CPE}, observing CPE implies that the tempered posterior, which implicitly employs the new likelihood and priors defined in Equation \ref{eq:newlikelihoodprior} is better specified in comparison to the original Bayesian posterior. The alignment with the underlying data-generating distribution is easily achieved by tempering. Consequently, tempered posteriors offer a simple, computationally efficient, and theoretically sound approach to mitigate the underfitting problem often encountered in contemporary Bayesian deep learning methods.

Furthermore, as discussed in Section \ref{sec:likelihood-wrt-lambda} and Section \ref{sec:prior-wrt-lambda}, increasing $\lambda$ results in likelihoods with lower aleatoric uncertainty and priors that favor models yielding such likelihoods on the training data $X$. Therefore, the occurrence of CPE in contemporary Bayesian deep learning indicates that the models currently employed in the field often underfit the data by assuming models with too high aleatoric uncertainty. This strengthens our understanding of the CPE as a consequence of underfitting, resulting from poorly specified likelihood and prior functions. 

We will further expand on and discuss how these implications relate to the literature in Section \ref{sec:modelmisspec&CPE}.

\section{Likelihood misspecification, prior misspecification and the CPE}\label{sec:modelmisspec&CPE}

\begin{Overview}
We relate our analysis in previous sections to the main arguments of CPE from the literature.
\begin{itemize}
    \item Section \ref{sec:5.1}: we demonstrate with Bayesian linear regression that exact inference can also bring CPE, showing CPE is not merely a side effect of approximate inference in NNs.
    \item Section \ref{sec:5.2}: using the same regression examples, we show that model misspecification can lead to underfitting or overfitting. CPE arises specifically when misspecified likelihoods or priors cause underfitting, not just from misspecification alone.
    \item Section \ref{sec:5.3}: likelihood misspecification is often identified as a source of CPE in practice. We show it is because the standard softmax likelihood (high aleatoric uncertainty) is misspecified and underfits the data-generating process (curated data with low aleatoric uncertainty).
    \item Section \ref{sec:5.4}: we show prior misspecification leads to CPE if it induces underfitting. Using tempered posteriors implicitly defines better-specified conditional priors that alleviate it.
    \item Section \ref{sec:5.5}: we show that larger models have more flexibility to fit data, thereby mitigating underfitting and CPE. Conversely, with small models and abundant data, the Bayesian posterior may already fit the data optimally, thereby exhibiting minimal underfitting (by our definition) and CPE.
\end{itemize}
\end{Overview}

\subsection{CPE, approximate inference, and NNs}\label{sec:5.1}
As mentioned in the introduction, several works have discussed that CPE is an artifact of inappropriate approximate inference methods, especially in the context of the highly complex posterior that emerge from neural networks \citep{WRVS+20}. There are occasions suggesting that if the approximate inference method is accurate enough, the CPE disappears \citep{IVHW21}. However, Proposition \ref{prop:likelihoodtempering:empiricalGibbsLoss} shows that when $\lambda$ is made larger than 1, the \textit{training loss} of the exact Bayesian posterior decreases; if the \textit{test loss} decreases too, the exact Bayesian posterior underfits. It means that even if the inference method is accurate, we can still observe the CPE due to underfitting. In fact, Figures \ref{fig:blr} and \ref{fig:blr2} shows examples of a Bayesian linear regression model learned on synthetic data. Here, the exact Bayesian posterior can be computed, and it is clear from Figures~\ref{fig:1c} and~\ref{fig:1d} that the CPE can occur in Bayesian linear regression with exact inference.
Although simple, the setting is articulated specifically to mimic the classification tasks using BNNs where CPE was observed. In particular, the linear model has more parameters than observations (i.e., it's overparameterized). We also note that \cite{adlam2020cold} presents similar findings and observations for Gaussian process regression.

\begin{figure}[h]
    \centering 
\begin{subfigure}{0.33\textwidth}
  \includegraphics[width=\linewidth]{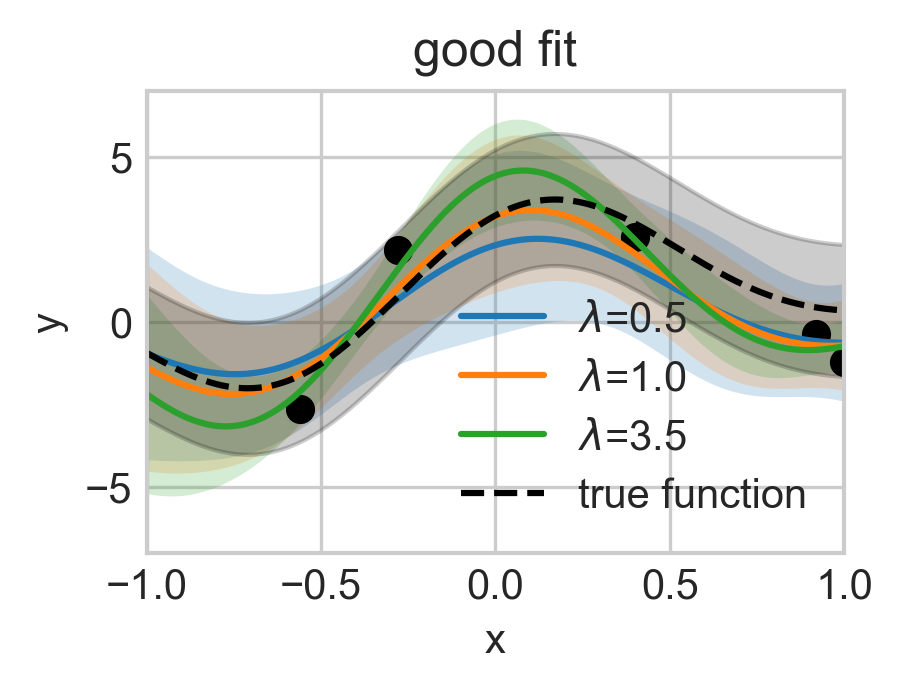}
\end{subfigure}\hfil 
\begin{subfigure}{0.33\textwidth}
  \includegraphics[width=\linewidth]{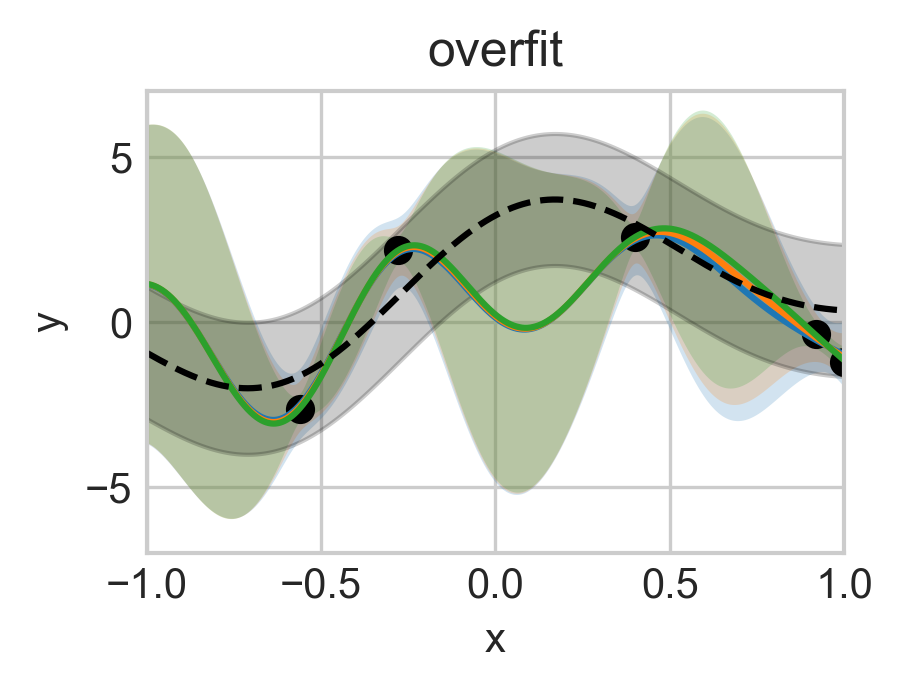}
\end{subfigure}\hfil 
\begin{subfigure}{0.33\textwidth}
  \includegraphics[width=\linewidth]{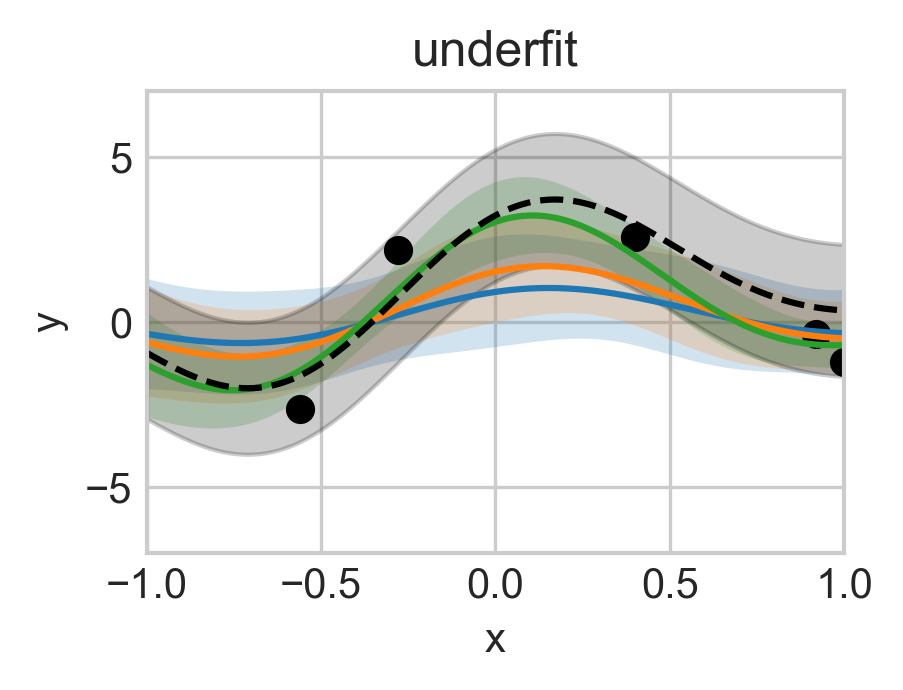}
\end{subfigure}

\begin{subfigure}[b]{0.33\textwidth}
  \includegraphics[width=\linewidth]{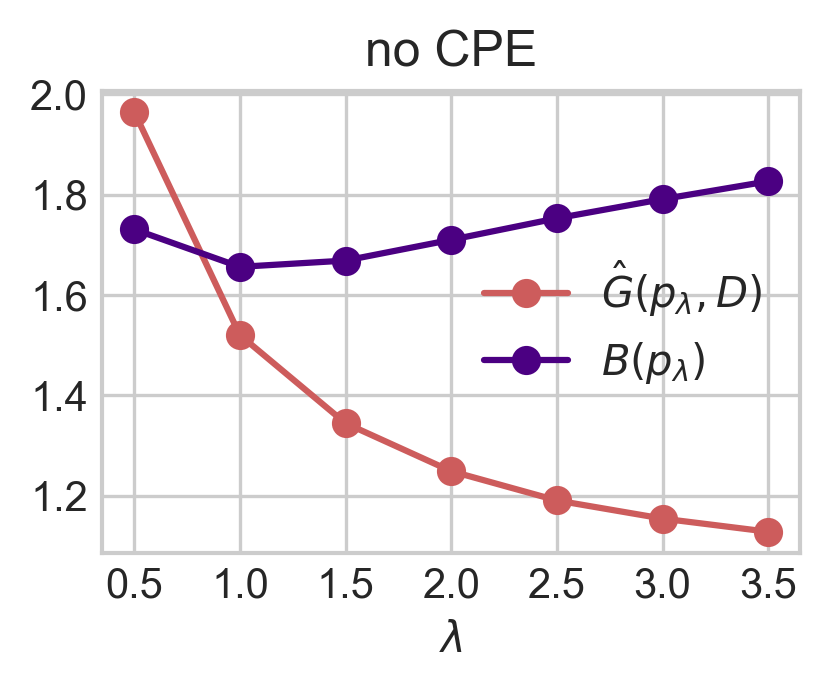}
  \captionsetup{format=hang, justification=centering}
  \caption{No likelihood or prior misspecification}
  \label{fig:1a}
\end{subfigure}\hfil 
\begin{subfigure}[b]{0.315\textwidth}
  \includegraphics[width=\linewidth]{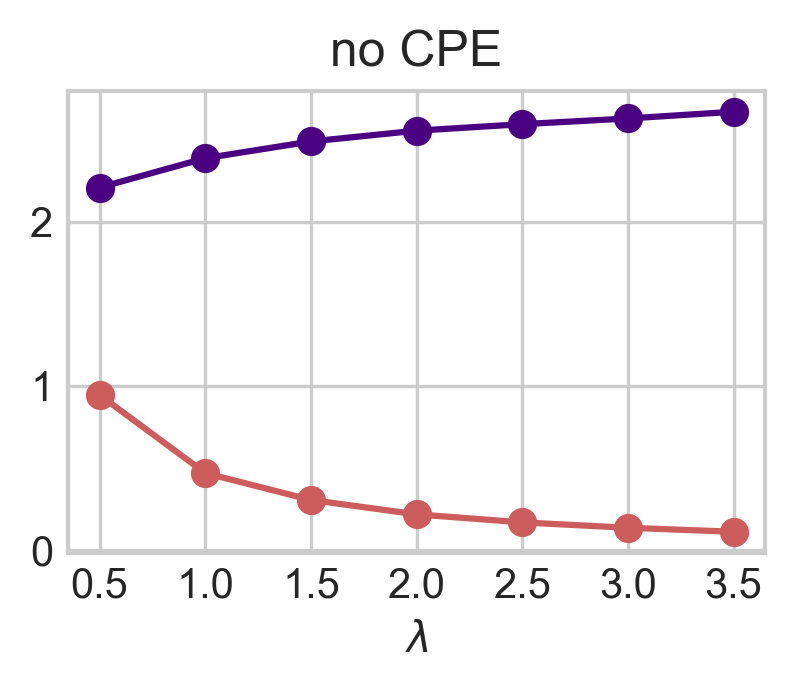}
  \captionsetup{format=hang, justification=centering}
  \caption{Likelihood misspecification case I}
  \label{fig:1b}
\end{subfigure}\hfil 
\begin{subfigure}[b]{0.335\textwidth}
  \includegraphics[width=\linewidth]{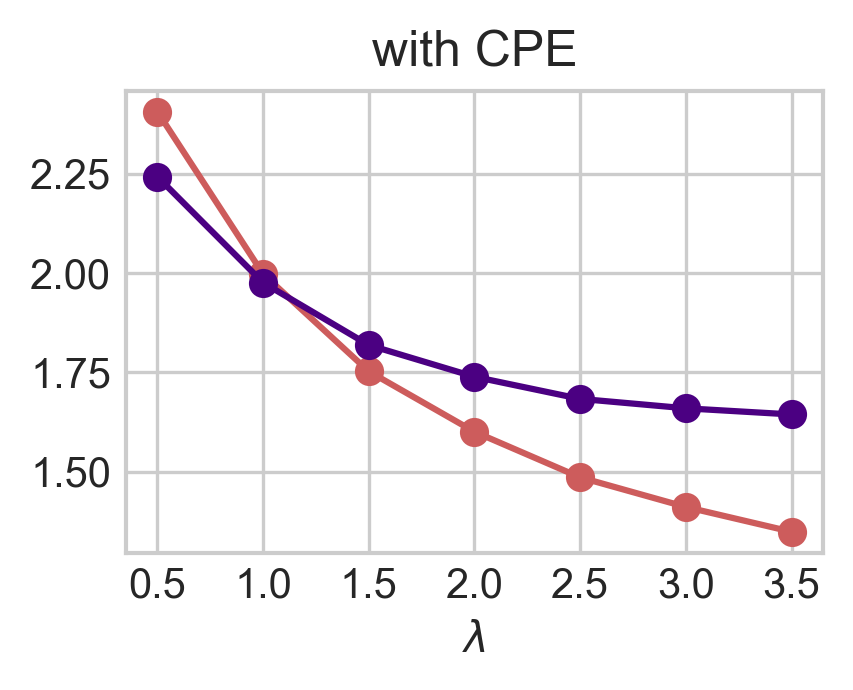}
  \captionsetup{format=hang, justification=centering}
  \caption{Only with prior misspecification}
  \label{fig:1d}
\end{subfigure}

\caption{\textbf{1. The CPE occurs in Bayesian linear regression with exact inference. 2. Model misspecification can lead to overfitting and to a ``warm'' posterior effect (WPE).} Every column displays a specific setting, as indicated in the caption. The first row shows exact Bayesian posterior predictive fits for three different values of the tempering parameter $\lambda$. The second row shows the Gibbs loss $\hat G(p_\lambda, D)$ (aka training loss) and the Bayes loss $B(p_\lambda)$ (aka testing loss) with respect to $\lambda$. The experimental details are given in Appendix \ref{app:sec:experiment}.}
\label{fig:blr}
\end{figure}

\begin{figure}[h]
    \centering 
\begin{subfigure}{0.33\textwidth}
  \includegraphics[width=\linewidth]{imgs/chap2/blm/fit_perfect.png}
\end{subfigure}\hfil 
\begin{subfigure}{0.33\textwidth}
  \includegraphics[width=\linewidth]{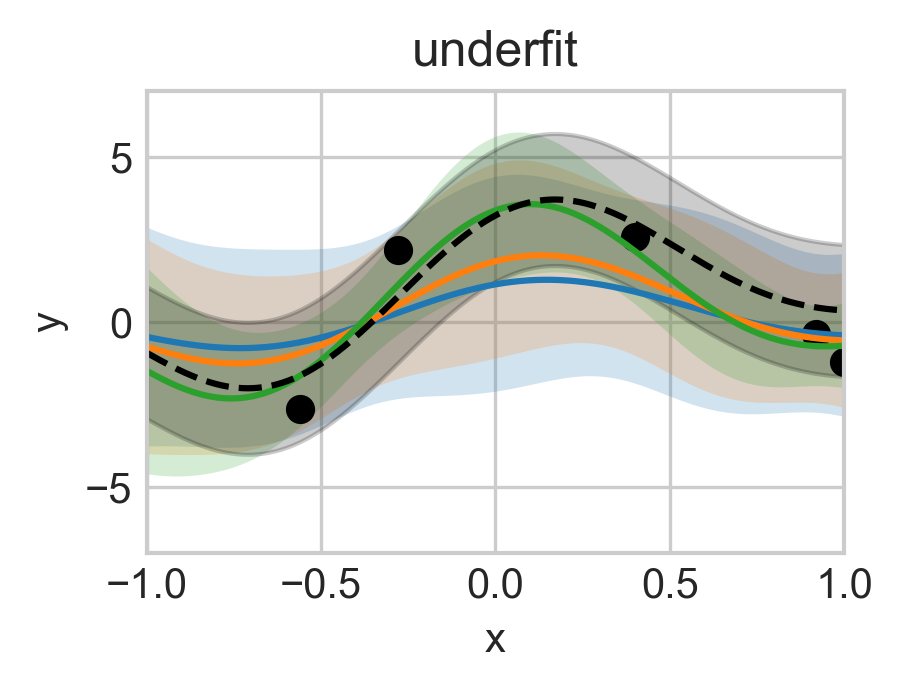}
\end{subfigure}\hfil 
\begin{subfigure}{0.33\textwidth}
  \includegraphics[width=\linewidth]{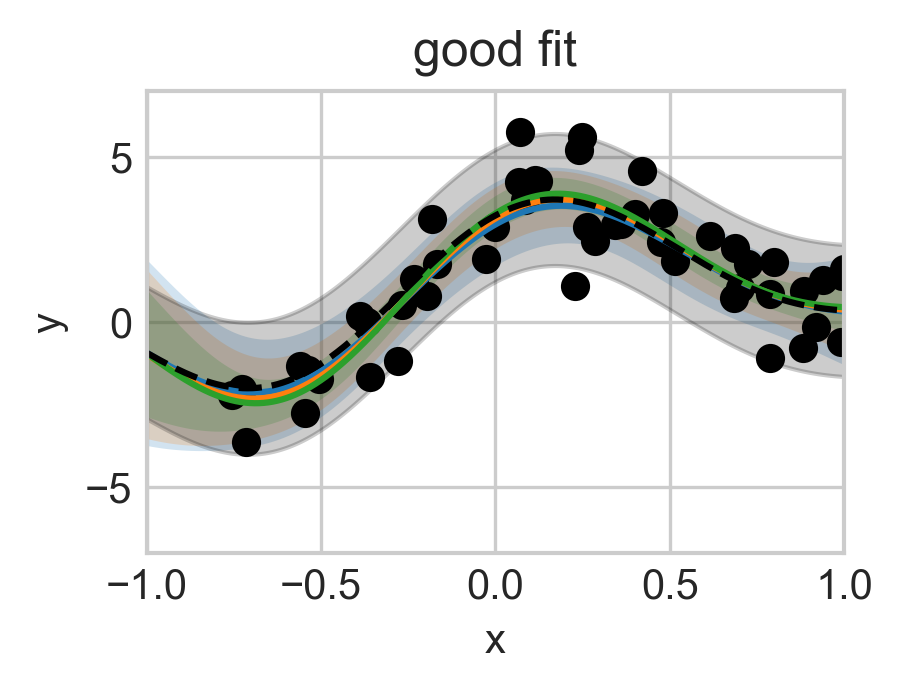}
\end{subfigure}

\begin{subfigure}[b]{0.33\textwidth}
  \includegraphics[width=\linewidth]{imgs/chap2/blm/nll_perfect.png}
  \captionsetup{format=hang, justification=centering}
  \caption{No likelihood or prior misspecification}
\end{subfigure}\hfil 
\begin{subfigure}[b]{0.33\textwidth}
  \includegraphics[width=\linewidth]{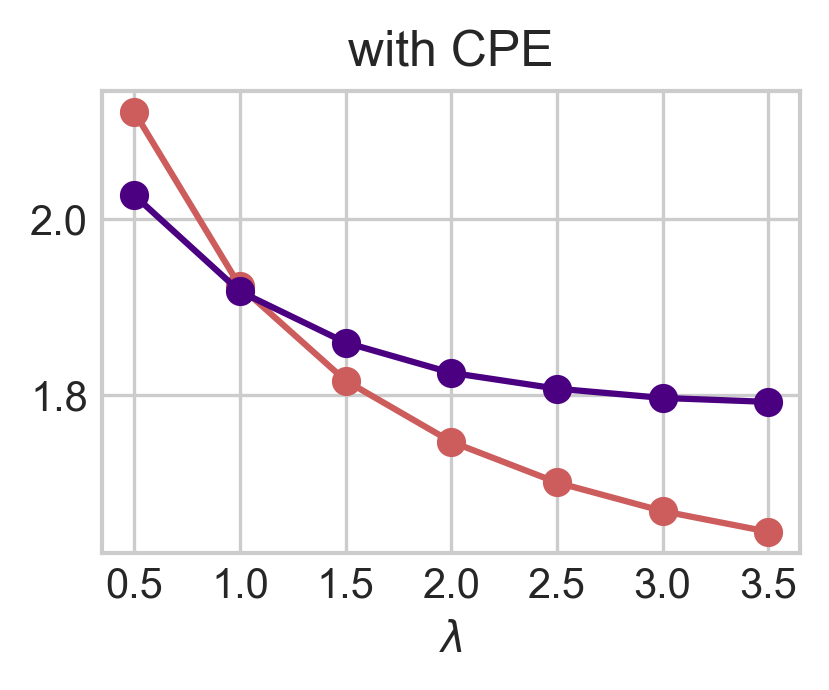}
  \captionsetup{format=hang, justification=centering}
  \caption{Likelihood misspecification case II}
  \label{fig:1c}
\end{subfigure}\hfil 
\begin{subfigure}[b]{0.33\textwidth}
  \includegraphics[width=\linewidth]{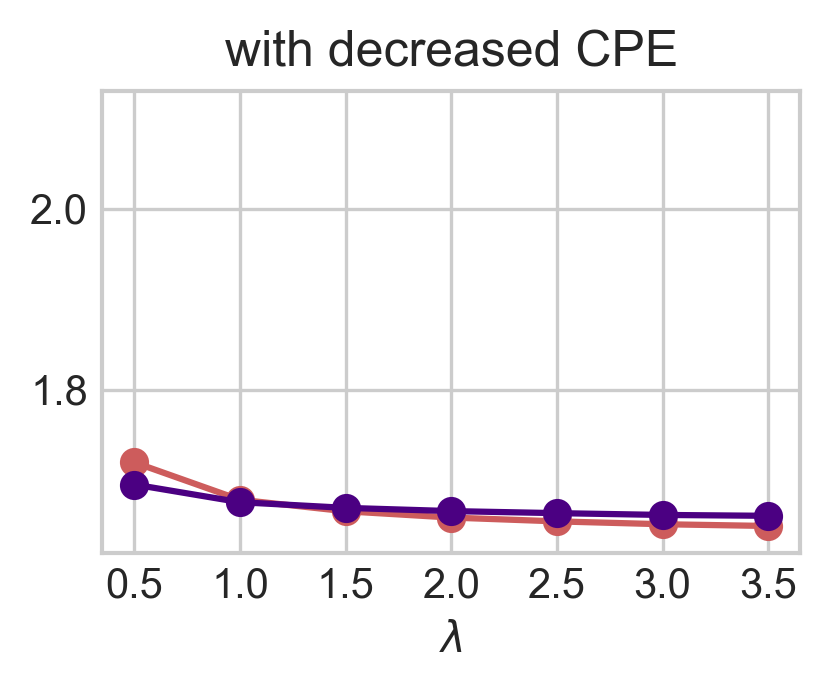}
  \captionsetup{format=hang, justification=centering}
  \caption{Same as (b) but with 50 data points}
  \label{fig:1e}
\end{subfigure}

\caption{\textbf{Extended results of Figure \ref{fig:blr} with more configurations of model misspecification.}}
\label{fig:blr2}
\end{figure}

\subsection{Model misspecification, CPE, and underfitting} \label{sec:5.2}

Prior and/or likelihood misspecification can lead Bayesian methods to both underfitting and overfitting, as widely discussed in the literature \citep{domingos2000bayesian,immer2021improving,KMIW22}. We illustrate this using a Bayesian linear regression model: Figures \ref{fig:1c} and~\ref{fig:1d} show how the Bayesian posterior underfits due to likelihood and prior misspecification, respectively. On the other hand, Figure~\ref{fig:1b} showcases a scenario where likelihood misspecification can perfectly lead to overfitting as well, giving rise to what we term a ``warm'' posterior effect (WPE), i.e., there exist other posteriors ($p_\lambda$ with $\lambda<1$) with lower testing loss, which, at the same time, have higher training loss due to Proposition \ref{prop:likelihoodtempering:empiricalGibbsLoss}. As a result, to describe CPE merely as a model misspecification issue without acknowledging underfitting offers a narrow interpretation of the problem.

The examples presented in Figures \ref{fig:blr} and \ref{fig:blr2} help illustrate the results of Proposition \ref{prop:true_posterior} and provide concrete demonstrations of the theoretical insights discussed: when CPE shows up, tuning $\lambda$ is akin to finding another Bayesian posterior with a less misspecified likelihood and prior.However, we note that in this particular Bayesian linear regression setup, the new prior $q(\bmtheta|\bmX,\lambda)$ is always equal to initial prior $p(\bmtheta)$ because the variance of the likelihood is assumed to be constant. Therefore, the analysis of the  \hyperref[exam:prior-Regression]{regression case} in Section \ref{sec:prior-wrt-lambda} does not directly apply here.

In the discussion regarding likelihood, we refer to the \hyperref[exam:likelihood-regression]{regression example} in Section \ref{sec:likelihood-wrt-lambda}. Let's first have a look at Figure \ref{fig:1c}, where the Gaussian likelihood model has a larger variance than the true data-generating process. By increasing $\lambda$, we obtain a likelihood model with a smaller variance (divided by $\lambda^2$, as shown in the \hyperref[exam:likelihood-regression]{regression example} in Section \ref{sec:likelihood-wrt-lambda}), i.e., we induce a new likelihood with lower aleatoric uncertainty (Proposition \ref{prop:entropy}). Such a new model is closer to the true data-generating distribution and less misspecified, thus enjoying better performance. The opposite can be seen in Figure \ref{fig:1b}, where the Gaussian likelihood model has a lower variance than the true data-generating distribution and the WPE occurs.

\subsection{The likelihood misspecification argument} \label{sec:5.3}

Likelihood misspecification has also been identified as a cause of CPE, especially in cases where the dataset has been \textit{curated} \citep{Ait21,KMIW22}. Data curation often involves carefully selecting samples and labels to improve the quality of the dataset. As a result, the curated data-generating distribution typically presents very low aleatoric uncertainty, meaning that $\nu(\bmy|\bmx)$ usually takes values very close to either $1$ or $0$. However, as previously discussed in \citep{Ait21,KMIW22}, the standard likelihoods used in deep learning for image classification, like softmax or sigmoid, tend to allocate more spread-out probabilities to the outcomes, implicitly reflecting a higher level of aleatoric uncertainty.  Therefore, their use in curated datasets that exhibit low uncertainty made them misspecified \citep{KMIW22,FGAO+22}. To address this issue, alternative likelihood functions like the Noisy-Dirichlet model \citep[Section 4]{KMIW22} have been proposed, which better align with the characteristics of the curated data. On the other hand, introducing noise labels also alleviates the CPE, as demonstrated in \citet[Figure 7]{Ait21}. By introducing noise labels, we intentionally increase aleatoric uncertainty in the data-generating distribution, which aligns better with the high aleatoric uncertainty assumed by the standard Bayesian deep networks \citep{KMIW22}. Consequently, according to these works, the CPE can be strongly alleviated when the likelihood misspecification is addressed.

Our theoretical analysis aligns with these findings in Sections \ref{sec:likelihood-wrt-lambda} and \ref{sec:sec4-insight}: fitting low aleatoric uncertainty data-generating distributions, e.g., $\nu(y|\bmx)\in\{0.01,0.99\}$, with high aleatoric uncertainty likelihood functions e.g., $p(y|\bmx,\bmtheta)\in[0.2,0.8]$, induces underfitting, and thus, CPE.  The presence of underfitting is not mentioned at all by any of these previous works \citep{Ait21,KMIW22}. On top of that, using Propositions \ref{prop:true_posterior} and \ref{prop:entropy}, our work explains why the likelihood implicitly used by the tempered posterior with $\lambda>1$ provides better generalizaton performance. Because, in this case, we are using a likelihood $q(\bmtheta|\bmX,\lambda)$ (Equation \ref{eq:newlikelihoodprior}) with lower aleatoric uncertainty, which better aligns with the low aleatoric uncertainty data-generating distribution induced by curated datasets, thus reducing the degree of model misspecification.

Figures \ref{fig:CPE:NarrowPriorSmall} and \ref{fig:CPE:SoftmaxSmall}, along with Figures \ref{fig:CPE:NarrowPriorLarge} and \ref{fig:CPE:SoftmaxLarge}, illustrate this point through a regular multi-class classification task on a curated benchmark dataset. Both scenarios utilize the same narrow prior. The distinction in Figure \ref{fig:CPE:SoftmaxSmall} lies in the adoption of a tempered softmax likelihood, defined as \(p(y|x,\theta) = \left(1+\exp\left(-\gamma \, \text{{logits}}(x,\theta)\right)\right)^{-1}\), with \(\gamma=2\), compared to \(\gamma=1\) in Figure \ref{fig:CPE:NarrowPriorSmall}. This tempered softmax likelihood, more closely aligned with the dataset's low aleatoric uncertainty as outlined by \citep{DBLP:conf/icml/GuoPSW17}, leads to a reduced incidence of CPE in Figure \ref{fig:CPE:SoftmaxSmall} compared to Figure \ref{fig:CPE:NarrowPriorSmall}. From the perspective of Proposition \ref{prop:true_posterior} and specifically Proposition \ref{prop:entropy}, the intrinsic lower aleatoric uncertainty of the likelihood used in Figure \ref{fig:CPE:SoftmaxSmall} (softmax with  $\gamma=2$) makes the potential for improvement through increasing \(\lambda\) somewhat limited, resulting in a less pronounced CPE compared to Figure \ref{fig:CPE:NarrowPriorSmall}. It is, however, important to highlight the critical interaction between the likelihood and the prior, as we dicuss next.

\begin{figure}[t]
  \begin{subfigure}{.32\linewidth}
    \includegraphics[width=\linewidth]{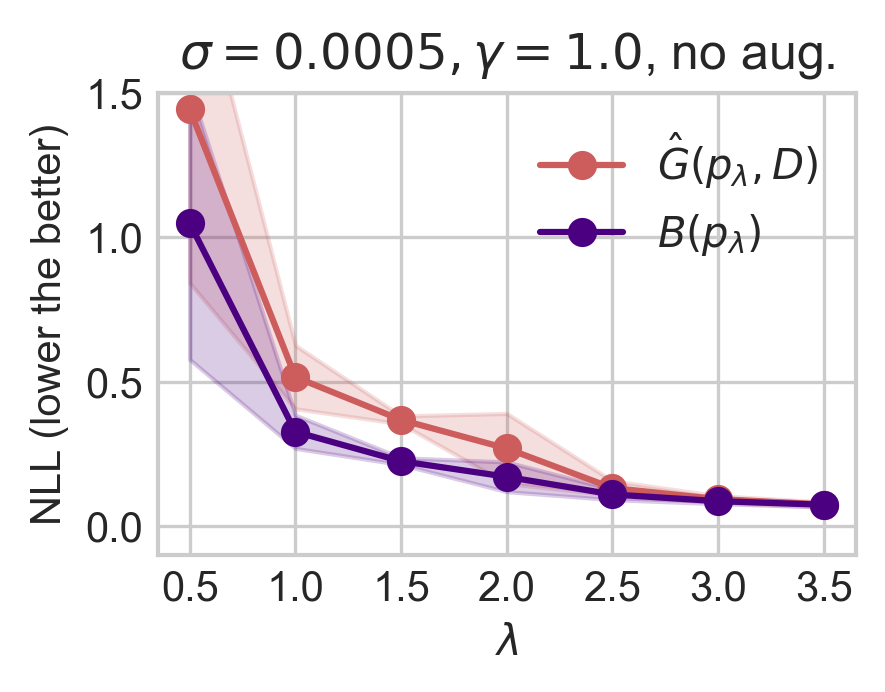}
      \captionsetup{format=hang, justification=centering}
    \caption{Narrow prior and standard softmax}
    \label{fig:CPE:NarrowPriorSmall}
  \end{subfigure}
  \hfill
  \begin{subfigure}{.32\linewidth}
    \includegraphics[width=\linewidth]{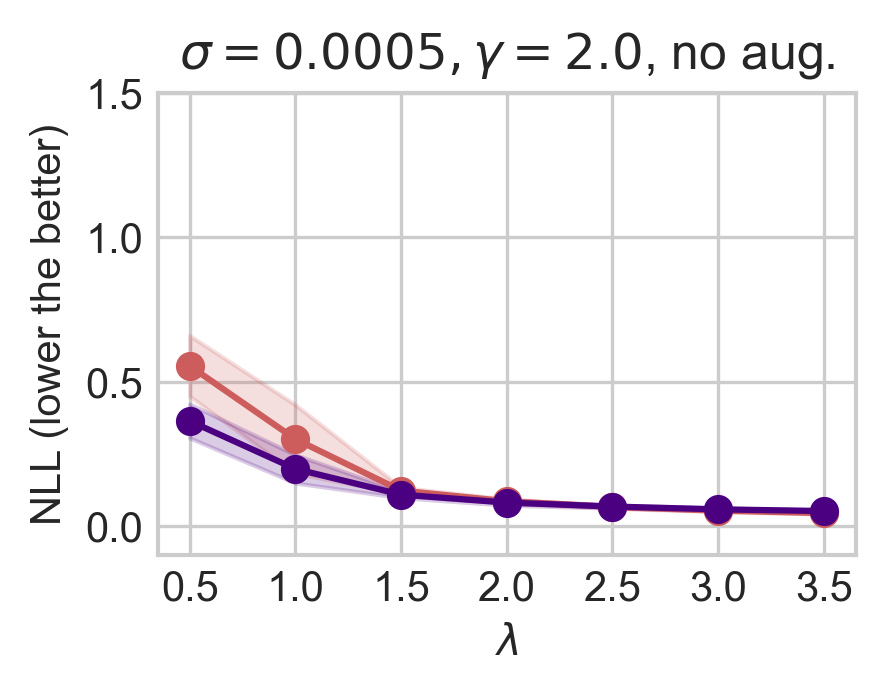}
      \captionsetup{format=hang, justification=centering}
    \caption{Narrow prior and tempered softmax}
    \label{fig:CPE:SoftmaxSmall}
  \end{subfigure}
  \hfill
  \begin{subfigure}{.32\linewidth}
    \includegraphics[width=\linewidth]{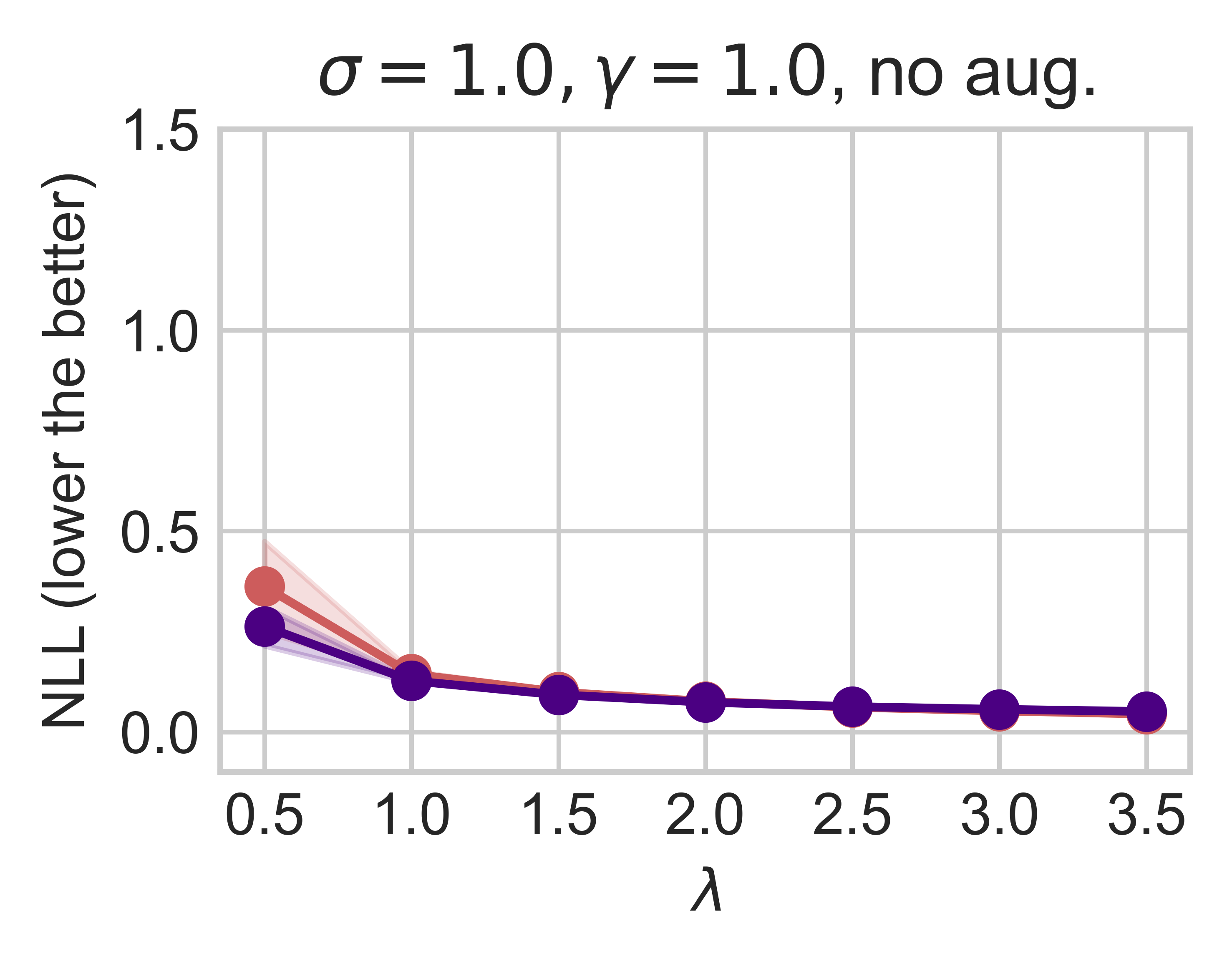}
      \captionsetup{format=hang, justification=centering}
    \caption{Standard prior and standard softmax}
    \label{fig:CPE:StandardPriorSmall}
  \end{subfigure}
    \caption{\textbf{Experimental illustrations for the arguments in Section \ref{sec:modelmisspec&CPE} using small CNN via SGLD on MNIST. We show similar results on Fashion-MNIST with small CNN and CIFAR-10(0) with ResNet-18 in Appendix \ref{app:sec:experiment-approx}}. Figures \ref{fig:CPE:NarrowPriorSmall} to \ref{fig:CPE:StandardPriorSmall} illustrate the arguments in Section~\ref{sec:modelmisspec&CPE}. Figure \ref{fig:CPE:StandardPriorSmall} uses the standard prior ($\sigma=1$) and the standard softmax ($\gamma=1$) for the likelihood without applying DA. Figure \ref{fig:CPE:NarrowPriorSmall} follows a similar setup except for using a narrow prior. Figure \ref{fig:CPE:SoftmaxSmall} uses a narrow prior as in Figure \ref{fig:CPE:NarrowPriorSmall} but with a tempered softmax that results in a lower aleatoric uncertainty. We report the training loss $\hat G(p_\lambda,D)$ and the testing losses, $B(p_\lambda)$ and $G(p_\lambda)$, from 10 samples of the small Convolutional neural network (CNN) via Stochastic Gradient Langevin Dynamics (SGLD). We show the mean and standard error across three different seeds. For additional experimental details, please refer to Appendix \ref{app:sec:experiment-approx}. 
    }
    \label{fig:CPE:small}
\end{figure}

\begin{figure}[t]

    \begin{subfigure}{.32\linewidth}
      \includegraphics[width=\linewidth]{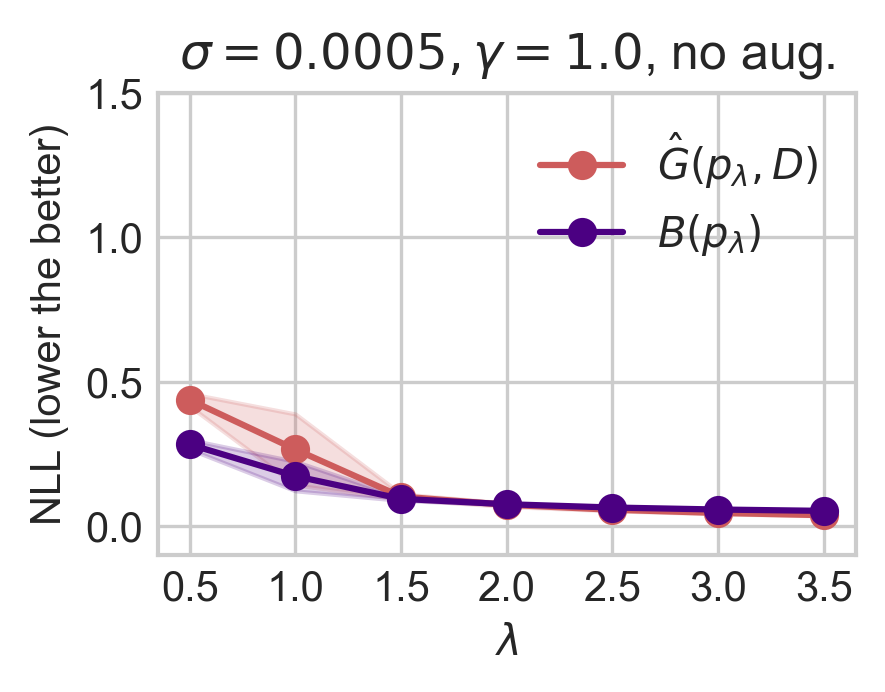}
      \captionsetup{format=hang, justification=centering}
      \caption{Narrow prior and standard softmax}
      \label{fig:CPE:NarrowPriorLarge}
    \end{subfigure}
    \begin{subfigure}{.32\linewidth}
      \includegraphics[width=\linewidth]{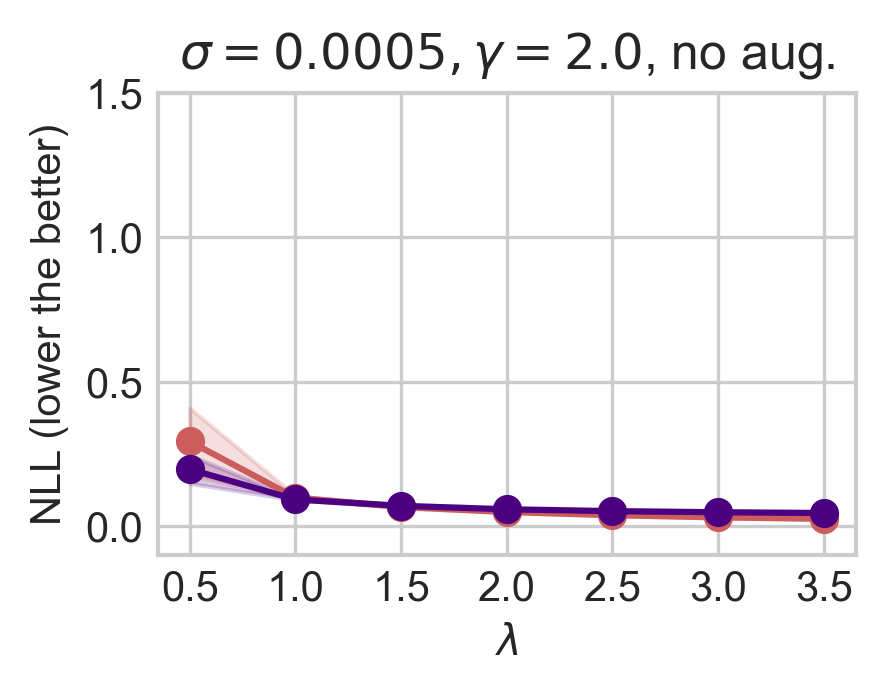}
      \captionsetup{format=hang, justification=centering}
      \caption{Narrow prior and tempered softmax}
      \label{fig:CPE:SoftmaxLarge}
    \end{subfigure}
    \begin{subfigure}{.32\linewidth}
      \includegraphics[width=\linewidth]{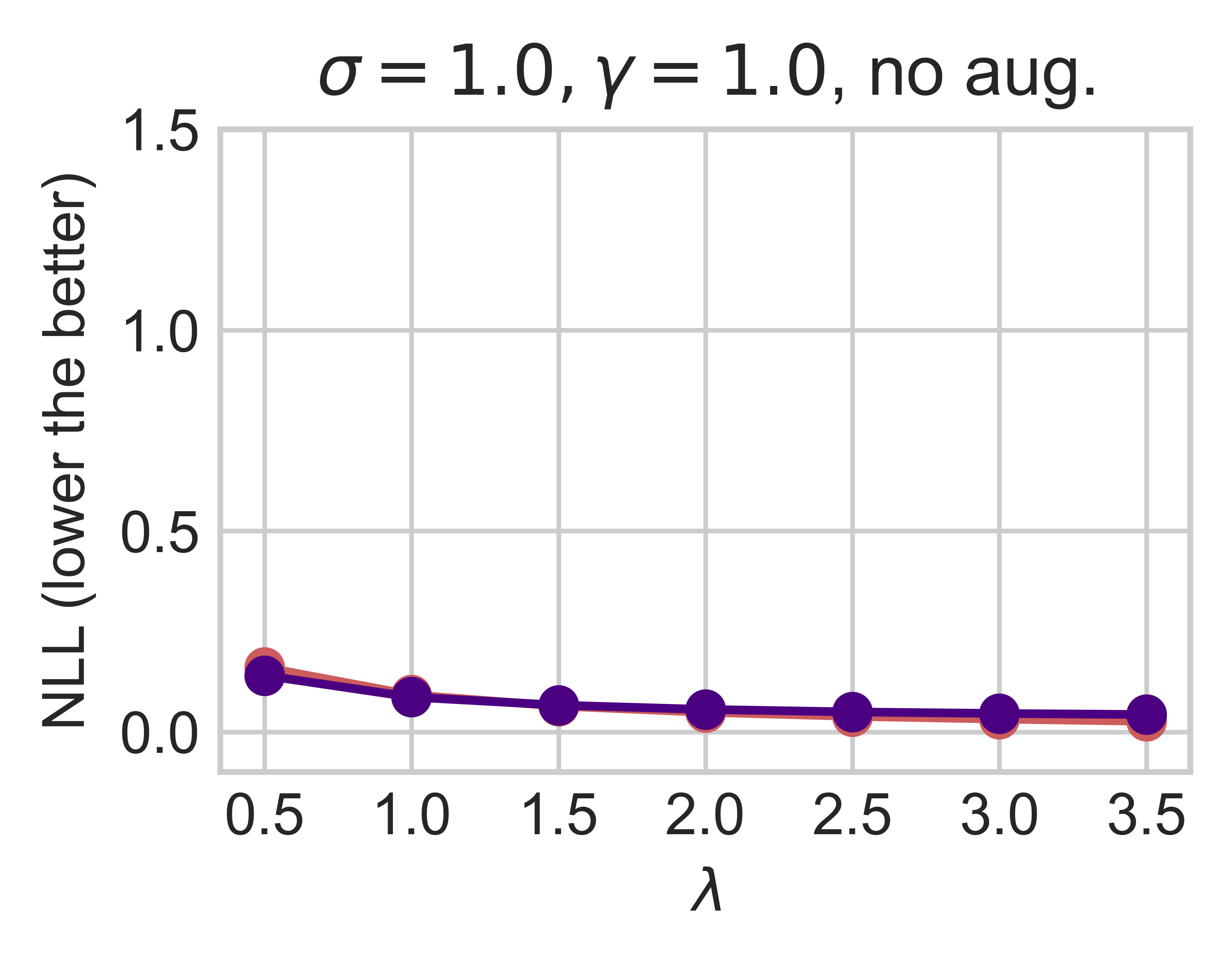}
      \captionsetup{format=hang, justification=centering}
      \caption{Standard prior and standard softmax}
      \label{fig:CPE:StandardPriorLarge}
    \end{subfigure}
      \caption{\textbf{Experimental illustrations for the arguments in Section~\ref{sec:modelmisspec&CPE} using large CNN via SGLD on MNIST. We show similar results on Fashion-MNIST with large CNN and CIFAR-10(0) with ResNet-50 in Appendix \ref{app:sec:experiment-approx}.} The experiment setup is similar to the setups in Figure~\ref{fig:CPE:small} but with a large CNN. Please refer to Appendix \ref{app:sec:experiment-approx} for further details on the model. 
      }
      \label{fig:CPE:large}
  \end{figure}

\subsection{The prior misspecification argument} \label{sec:5.4}

As highlighted in previous works, such as in \citet{WRVS+20,FGAO+22}, isotropic Gaussian priors are commonly chosen in modern Bayesian neural networks for the sake of tractability in approximate Bayesian inference rather than chosen based on their alignment with our actual beliefs. Given that the presence of the CPE implies that either the likelihood and/or the prior are misspecified, and given that neural networks define highly flexible likelihood functions, there are strong reasons for thinking these commonly used priors are misspecified. Notably, the experiments conducted by \citet{FGAO+22} demonstrate that the CPE can be mitigated in fully connected neural networks when using heavy-tailed prior distributions that better capture the weight characteristics typically observed in such networks. However, such priors were found to be ineffective in addressing the CPE in convolutional neural networks \citep{FGAO+22}, indicating the challenges involved in designing effective Bayesian priors within this context.

Our theoretical analysis provides a deeper insight into these observations. As discussed in Section \ref{sec:CPE}, the absence of underfitting means the absence of CPE. This suggests flexible likelihood functions may still result in posteriors that underfit, due to the prior's tendency to overly regularize. This may incur CPE when the strong prior fails to allocate enough probability to models that both fit the training data well and exhibit good generalization capabilities. As detailed in Section \ref{sec:prior-wrt-lambda}, employing tempered posteriors with $\lambda>1$ effectively defines a conditional prior $q(\bmtheta|\bmX,\lambda)$ that favors models with lower aleatoric uncertainty. If such models align better with the training data compared to the original prior, then we may observe CPE. Also, the conditional prior $q(\bmtheta|\bmX,\lambda)$ with $\lambda>1$ can be considered better specified than the original prior $p(\bmtheta)$.

Figures \ref{fig:CPE:NarrowPriorSmall} and \ref{fig:CPE:StandardPriorSmall} exemplify this situation. The prior in the case of Figures \ref{fig:CPE:NarrowPriorSmall} is very narrow ($\sigma=0.0005$), inducing strong regularization. Such a narrow prior results in a posterior that severely underfits the training data, evident from the high empirical Gibbs loss that deviates significantly from zero. Additionally, we observe a strong CPE. On the other hand, with a flatter prior in the case of Figure~\ref{fig:CPE:StandardPriorSmall}, the CPE is considerably diminished. According to the discussion above and Sections \ref{sec:prior-wrt-lambda} and \ref{sec:sec4-insight}, we know that the flatter prior allocates more probability mass to preferred models. Also, such preferred models have lower aleatoric uncertainty than the ones assigned initially by the narrow prior in the former case. To elaborate further, in the former case, the new prior $q(\bmtheta|\bmX,\lambda)$ with $\lambda>1$ would place much more probability mass to models with lower aleatoric uncertainty than the narrow prior and strongly alleviating underfitting. In the second case, since the flatter prior already distributes probability mass more broadly across the model class, the room to shift probability mass to models with lower aleatoric uncertainty is more limited than that from a narrower prior, resulting in a milder CPE. 

\subsection{Model size, sample size in relation to CPE and underfitting} \label{sec:5.5}

Larger models have the capacity to fit data more effectively, while smaller models are more likely to underfit. As we have argued that if there is no underfitting, there is no CPE, we expect that the size of the model has an impact on the strength of CPE as well. We demonstrate in Figure \ref{fig:CPE:small} and Figure \ref{fig:CPE:large}. Specifically, in our experiments presented in Figure \ref{fig:CPE:small}, we use a relatively small convolutional neural network (CNN), which has a more pronounced underfitting behavior, and this indeed corresponds to a stronger CPE. On the other hand, we employ a larger CNN in Figure \ref{fig:CPE:large}, which has less underfitting, and we see the CPE is strongly alleviated. Actually, this effect can be directly inferred from Theorem \ref{cor:likelihoodtempering:bayesposterior_optimality}. For an extremely flexible model capable of perfectly fitting both the original training samples and new samples, this theorem suggests that a CPE should not be expected, as the model's fit on the original data remains perfect, even when new examples are introduced. 

Theorem \ref{cor:likelihoodtempering:bayesposterior_optimality} can also be used to understand why small models in the presence of large training data sets do not exhibit CPE. We empirically illustrate this point in Figures \ref{fig:blr} and \ref{fig:blr2}. In particular, Figure \ref{fig:1c} and Figure \ref{fig:1e} use the same (small) regression models and settings where the only difference is that Figure \ref{fig:1c} uses 5 data points while Figure \ref{fig:1e} uses 50 data points. In situations where a model possesses limited flexibility and the training set is large, including additional examples should barely affect the fit of the original training data because the Bayesian posterior is highly concentrated and will be barely affected by a single extra sample. Then, as predicted by Theorem \ref{cor:likelihoodtempering:bayesposterior_optimality}, CPE in Figure \ref{fig:1e} is much less significant than in Figure \ref{fig:1c}.

Finally, it's worth noting that Figure 11 in \cite{WRVS+20} shows the opposite effect, where larger models exhibit much stronger CPE compared to shallower or narrower versions of the same architectures. However, it's important to recognize that \cite{WRVS+20} studied full-tempering, whereas our work focuses on likelihood-tempering. For full-tempering, Proposition \ref{prop:likelihoodtempering:empiricalGibbsLoss} does not necessarily hold. Intuitively, since $\lambda$ operates on both the likelihood (data) and the prior (regularization) simultaneously, the effect of increasing $\lambda$ is mixed, not necessarily improving the fit on the training data. Consequently, the CPE brought by full-tempering as $\lambda$ increases does not necessarily coincide with better training loss, as the training loss may not be improvable. As a result, the CPE observed with full-tempering cannot be interpreted solely as underfitting. Therefore, for full-tempering, increasing model capacity may not achieve a lower degree of CPE, unlike the behavior we observed in our case focusing on likelihood-tempering.

\section{Data augmentation (DA) and the CPE}\label{sec:data-augmentation}

\begin{Overview}
We show conditions under which data augmentation exacerbates the CPE. 
\begin{itemize}
    \item Section \ref{sec:6.1}: starting with the Gibbs loss for clarity, 
    we show that data augmentation induces a stronger CPE on the Gibbs loss if the augmented data provides more information about the data-generating process, increasing the correlation between the expected and empirical losses.
    \item Section \ref{sec:6.2}: extending the above idea, we show analogous conditions where data augmentation exacerbates CPE on the Bayes loss.
\end{itemize}
\end{Overview}

\noindent
Machine learning is applied to many different fields and problems. In many of them, the data-generating distribution is known to have properties that can be exploited to generate new data samples \citep{shorten2019survey} artificially. This is commonly known as \emph{data augmentation (DA)} and relies on the property that for a given set of transformations \(\daset=\{h:\cal X\rightarrow \cal X\}\), the data-generating distribution satisfies \(\nu(\bmy|\bmx) = \nu(\bmy|h(\bmx))\) for all \(h \in \daset\). In practice, not all the transformations are applied to every single data. Instead, a probability distribution (usually uniform) $\mu_\daset$ is defined over $\daset$, and augmented samples are drawn accordingly. As argued in \citet{NGGA+22}, the use of data augmentation when training Bayesian neural networks implicitly targets the following (pseudo) log-likelihood, denoted $\hat{L}_{\text{\tiny{DA}}}(D,\bmtheta)$ and defined as
\begin{equation}\label{eq:DA:likelihood}
    \hat{L}_{\text{\tiny{DA}}}(D,\bmtheta) = \frac{1}{n}\sum_{i\in[n]} \E_{h\sim \mu_\daset}\left[-\ln p(\bmy_i|h(\bmx_i),\bmtheta)\right]\,,
\end{equation}
where data augmentation provides unbiased estimates of the expectation under the set of transformations using \textit{Monte Carlo samples} (i.e., random data augmentations).

Although some argue that this data-augmented \textit{(pseudo) log-likelihood} ``does not have a clean interpretation as a valid likelihood function'' \citep{WRVS+20,IVHW21}, we do not need to enter into this discussion to understand why the CPE emerges when using the generalized Bayes posterior \citep{bissiri2016general_6} associated to this \textit{(pseudo) log-likelihood}, which is the main goal of this section. We call this posterior the DA-tempered posterior and is denoted by $p_\lambda^{\text{\tiny{DA}}}(\bmtheta|D)$. The DA-tempered posterior can be expressed as the global minimizer of the following learning objective,
\begin{equation}\label{eq:likelihoodTempering:DA:ELBO}
p_\lambda^{\text{\tiny{DA}}}(\bmtheta|D) = \argmin_\rho \E_\rho[n\hat{L}_{\text{\tiny{DA}}}(D,\bmtheta)] + \frac{1}{\lambda}\operatorname{KL}(\rho(\bmtheta|D),p(\bmtheta))\,.
\end{equation}
This is similar to Equation~\ref{eq:likelihoodTempering:ELBO} but now using $\hat{L}_{\text{\tiny{DA}}}(D,\bmtheta)$ instead of $\hat{L}(D,\bmtheta)$, where we recall the notation $\hat{L}(D,\bmtheta) =-\tfrac{1}{n}\ln p(D|\bmtheta)$. Hence, the resulting DA-tempered posterior is given by $p^{\text{\tiny{DA}}}_\lambda(\bmtheta|D)\propto e^{-n\lambda \hat{L}_{\text{\tiny{DA}}}(D,\bmtheta)} p(\bmtheta)$. In comparison, the tempered posterior $p_\lambda(\bmtheta|D)$ in Equation~\ref{eq:likelihoodTempering} can be similarly expressed as $e^{-n\lambda \hat L(D,\bmtheta)}p(\bmtheta)$.

There is large empirical evidence that DA induces a stronger CPE \citep{WRVS+20,IVHW21,FGAO+22}. Indeed, many of these studies show that if CPE is not present in our Bayesian learning settings, using DA makes it appear. According to our previous analysis, this means that the use of DA induces a stronger underfitting. To understand why this is case, we will take a step back and begin analyzing the impact of DA in the so-called Gibbs loss of the DA-Bayesian posterior \(p_{\lambda=1}^{\text{\tiny{DA}}}\) rather than the Bayes loss, as this will help us in understanding this puzzling phenomenon. 

\subsection{Data augmentation and CPE on the Gibbs loss} \label{sec:6.1}

The expected Gibbs loss of a given posterior $\rho$, denoted  \(G(\rho)\), is a commonly used metric in the theoretical analysis of the \textit{generalization performance} of Bayesian methods \citep{germain2016pac,masegosa2020learning}. The Gibbs loss represents the average of the expected log-loss of individual models under the posterior $\rho$, that is,
\[G(\rho) = \E_\rho[L(\bmtheta)] = \E_\rho[\E_{\nu}[-\ln [p(\bmy|\bmx, \bmtheta)]]\,. \]
In fact, Jensen's inequality confirms that the expected Gibbs loss serves as an upper bound for the Bayes loss, i.e., $G(\rho)\geq B(\rho)$. This property supports the expected Gibbs loss to act as a proxy of the Bayes loss, which justifies its usage in gaining insights into how DA impacts the CPE.

We will now study whether data augmentation can cause a CPE on the Gibb loss. In other words, we will examine whether increasing the parameter $\lambda$ of the DA-tempered posterior leads to a reduction in the Gibbs loss. This can be formalized by extending Definition \ref{def:likelihoodtempering:CPE} to the expected Gibbs loss by considering its derivative $\frac{d}{d\lambda} G(p_{\lambda})$ at $\lambda=1$, which can be represented as follows: 
\begin{equation}\label{eq:Gibbs_cpe}
    \frac{d}{d\lambda} G(p_{\lambda})_{|\lambda=1}= - \mathrm{COV}_{p_{\lambda=1}} \big(n \hat{L}(D,\bmtheta),L(\bmtheta)\big)\,.
\end{equation}
Where $\mathrm{COV}(X,Y)$ denotes the covariance of $X$ and $Y$. Again, due to the page limit, we postpone the necessary proofs in this section to Appendix \ref{app:sec:data-augmentation}. 

With this extended definition, if Equation~\ref{eq:Gibbs_cpe} is negative, we can infer the presence of CPE for the Gibbs loss as well. Based on this, we say that DA induces a stronger CPE if the derivative of the expected Gibbs loss for the DA-tempered posterior exhibits a more negative trend at $\lambda=1$, i.e., if $\frac{d}{d\lambda} G(p_\lambda^{\text{\tiny{DA}}})_{|\lambda=1}<\frac{d}{d\lambda} G(p_{\lambda})^{|\lambda=1}$.This condition can be equivalently stated as  
\begin{equation}\label{eq:da:cpe:gibbs}
    \mathrm{COV}_{p_{\lambda=1}^{\text{\tiny{DA}}}}\big(n\hat{L}_{\text{\tiny{DA}}}(D,\bmtheta), L(\bmtheta)\big)>\mathrm{COV}_{p_{\lambda=1}}\big(n\hat{L}(D,\bmtheta), L(\bmtheta)\big)>0\,.    
\end{equation}
The inequality presented above helps characterize and understand the occurrence of a stronger CPE when using DA. A stronger CPE arises if the expected Gibbs loss of a model $L(\bmtheta)$ is more \textit{correlated} with the empirical Gibbs loss of this model on the augmented training dataset $\hat{L}_{\text{\tiny{DA}}}(D,\bmtheta)$ than on the non-augmented dataset $\hat{L}(D,\bmtheta)$. This observation suggests that, if we empirically observe that the CPE is stronger when using an augmented dataset, the set of transformations $\daset$ used to generate the augmented dataset are introducing \textit{valuable information} about the data-generating process.

\begin{figure}[h]
  \begin{subfigure}{.32\linewidth}
    \includegraphics[width=\linewidth]{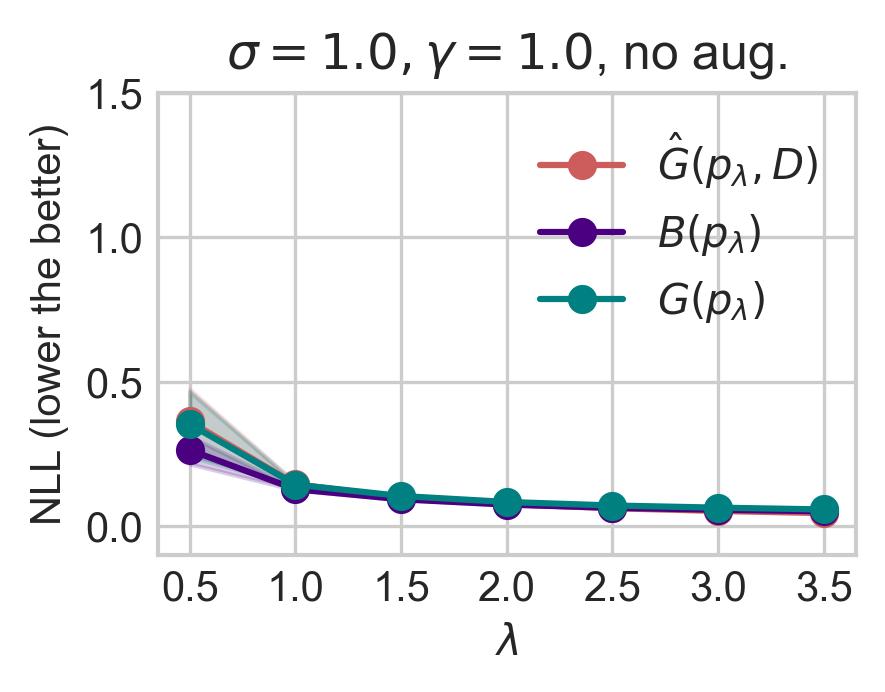}
      \captionsetup{format=hang, justification=centering}
    \caption{Standard prior and standard softmax}
    \label{fig:CPE:StandardPriorSmallNoAug}
  \end{subfigure}
  \hfill
  \begin{subfigure}{.32\linewidth}
    \includegraphics[width=\linewidth]{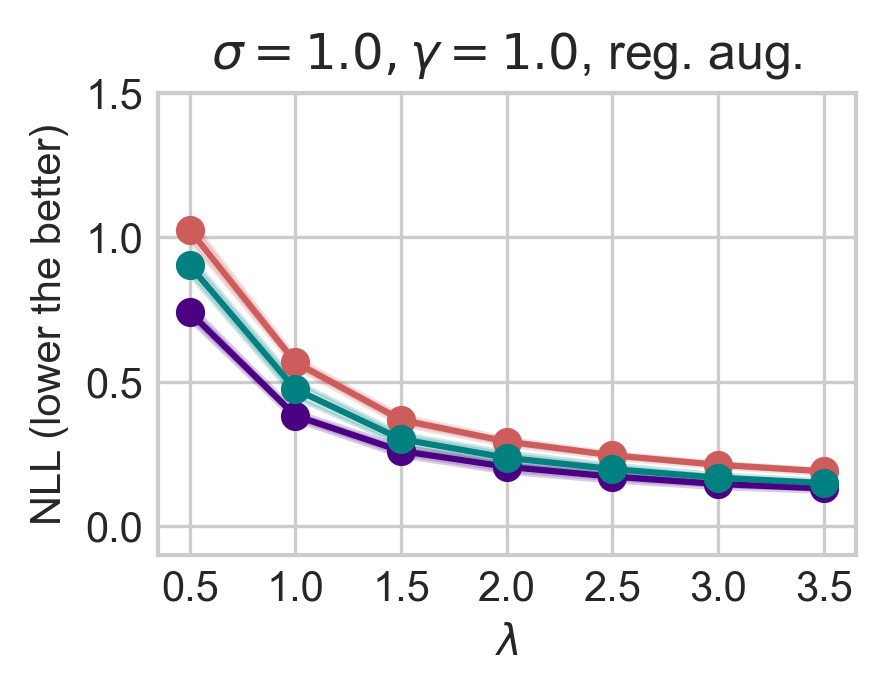}
      \captionsetup{format=hang, justification=centering}
    \caption{Random crop and horizontal flip}
    \label{fig:DA:CPE:augSmall}
  \end{subfigure}
  \hfill
  \begin{subfigure}{.32\linewidth}
    \includegraphics[width=\linewidth]{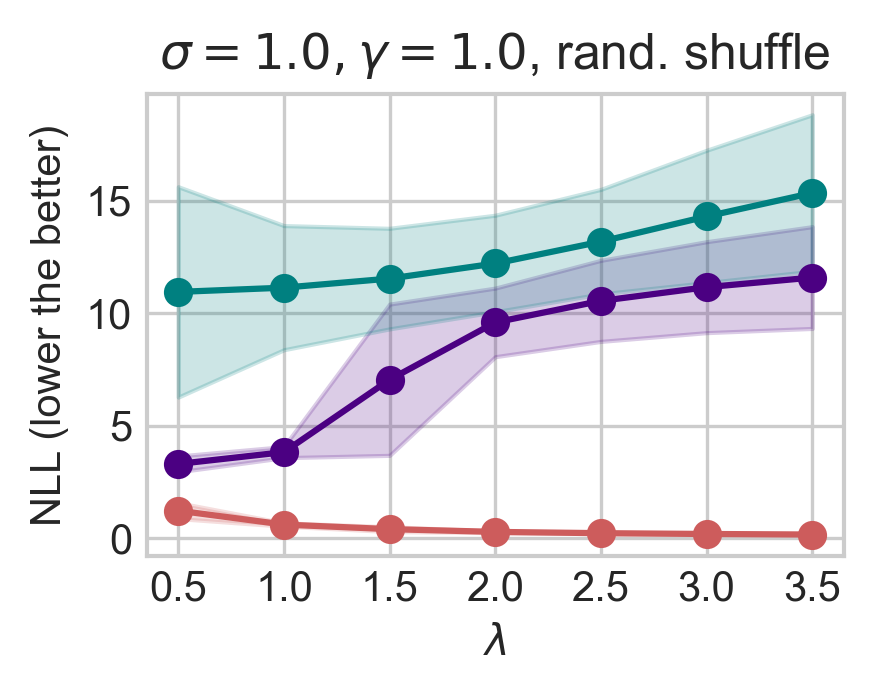}
      \captionsetup{format=hang, justification=centering}
    \caption{Image pixels randomly shuffled}
    \label{fig:DA:CPE:permSmall}
  \end{subfigure}
    \caption{\textbf{Experimental illustrations for the arguments in Section \ref{sec:data-augmentation} using small CNN via SGLD on MNIST. We show similar results on Fashion-MNIST with small CNN and CIFAR-10(0) with ResNet-18 in Appendix \ref{app:sec:experiment-approx}}. Figures \ref{fig:CPE:StandardPriorSmallNoAug} to \ref{fig:DA:CPE:permSmall} illustrate the arguments in Section~\ref{sec:data-augmentation}. Figure \ref{fig:CPE:StandardPriorSmallNoAug} uses the standard prior ($\sigma=1$) and the standard softmax ($\gamma=1$) for the likelihood without applying DA. Figure \ref{fig:DA:CPE:augSmall} follows the setup as in Figure \ref{fig:CPE:StandardPriorSmallNoAug} but with standard DA applied, while Figure \ref{fig:DA:CPE:permSmall} uses fabricated DA. We report the training loss $\hat G(p_\lambda,D)$ and the testing losses $B(p_\lambda)$ and $G(p_\lambda)$ from 10 samples of the small Convolutional neural network (CNN) via Stochastic Gradient Langevin Dynamics (SGLD). We show the mean and standard error across three different seeds. For additional experimental details, please refer to Appendix \ref{app:sec:experiment-approx}. 
    }
    \label{fig:CPE:smallAug}
\end{figure}

\begin{figure}[h]
      \begin{subfigure}{.32\linewidth}
      \includegraphics[width=\linewidth]{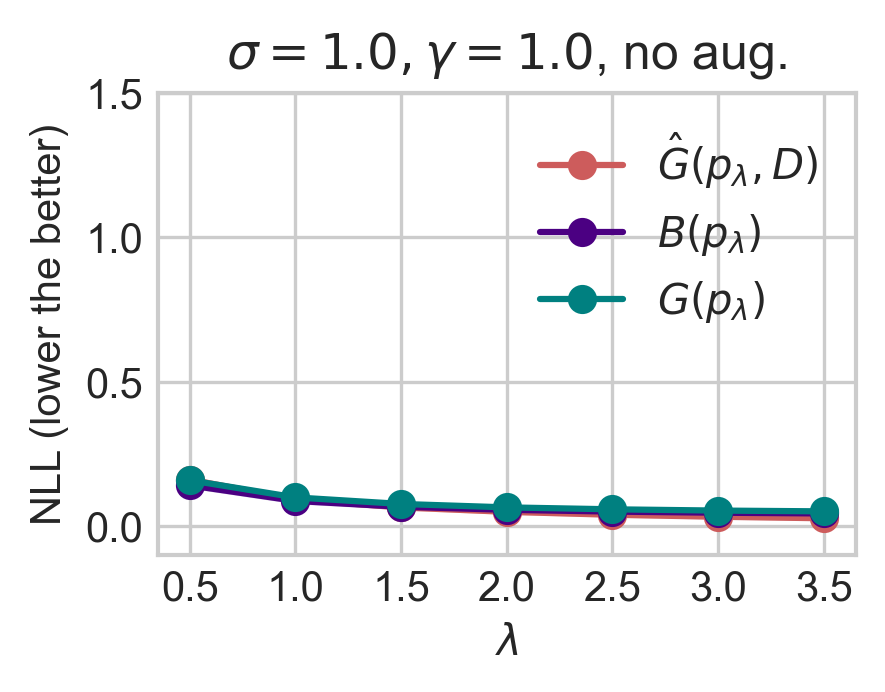}
      \captionsetup{format=hang, justification=centering}
      \caption{Standard prior and standard softmax}
      \label{fig:CPE:StandardPriorLargeNoAug}
    \end{subfigure}
    \begin{subfigure}{.32\linewidth}
      \includegraphics[width=\linewidth]{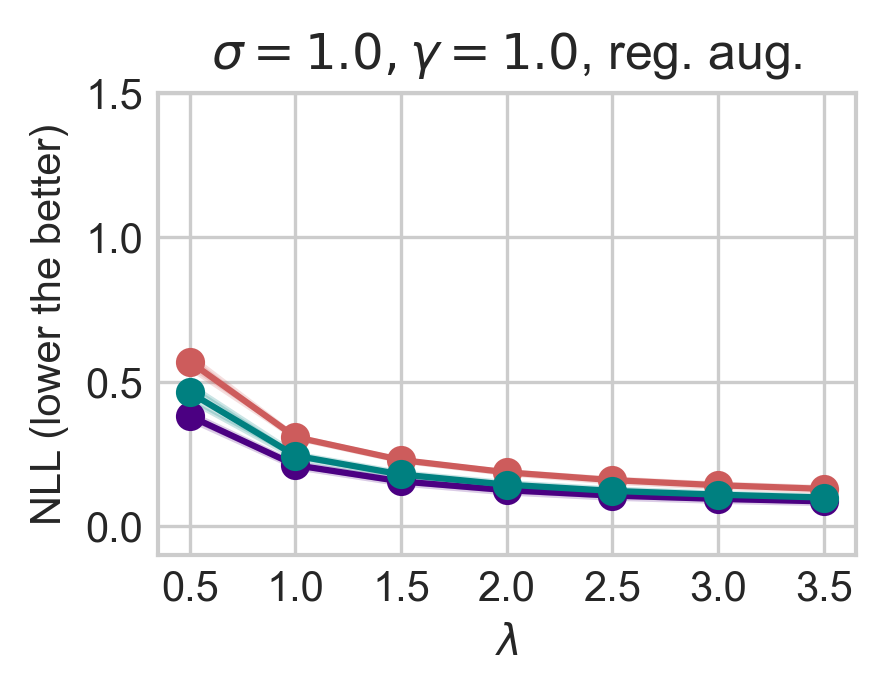}
      \captionsetup{format=hang, justification=centering}
      \caption{Random crop and horizontal flip}
      \label{fig:DA:CPE:augLarge}
    \end{subfigure}
    \begin{subfigure}{.32\linewidth}
      \includegraphics[width=\linewidth]{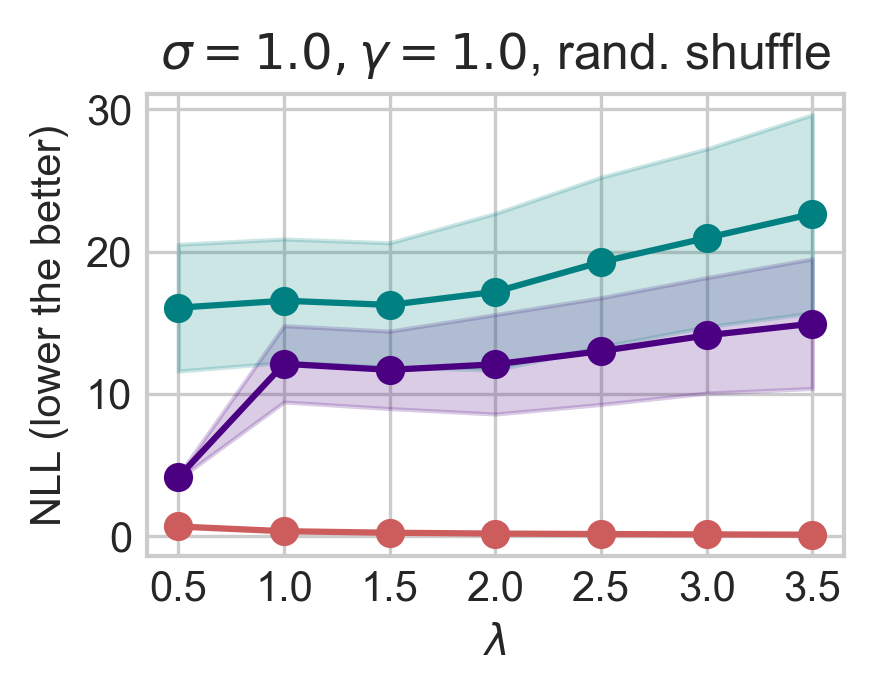}
      \captionsetup{format=hang, justification=centering}
      \caption{Image pixels randomly shuffled}
      \label{fig:DA:CPE:permLarge}
    \end{subfigure}
      \caption{\textbf{Experimental illustrations for the arguments in Section \ref{sec:data-augmentation} using large CNN via SGLD on MNIST. We show similar results on Fashion-MNIST with large CNN and CIFAR-10(0) with ResNet-50 in Appendix \ref{app:sec:experiment-approx}.} The experiment setup is similar to the setups in Figure~\ref{fig:CPE:smallAug} but with a large CNN. Please refer to Appendix \ref{app:sec:experiment-approx} for further details on the model. 
      }
      \label{fig:CPE:largeAug}
  \end{figure}

Figure \ref{fig:CPE:smallAug} clearly illustrates such situations. Figure \ref{fig:DA:CPE:augSmall} shows that, compared to Figure \ref{fig:CPE:StandardPriorSmallNoAug}, the standard DA, which makes use of the invariances inherent in the data-generating distribution, induces a CPE on the Gibbs loss. Thus, the condition in Equation~\ref{eq:da:cpe:gibbs} holds by definition. On the other hand, Figure \ref{fig:DA:CPE:permSmall} uses a fabricated DA, where the same permutation is applied to the pixels of the images in the training dataset, which destroys low-level features present in the data-generating distribution. In this case, the derivative of the Gibb loss is positive, and Equation~\ref{eq:da:cpe:gibbs} holds in the opposite direction. These findings align perfectly with the explanations provided above, showing that DA induces a stronger underfitting.

\subsection{Data augmentation and CPE on the Bayes loss} \label{sec:6.2}

Now, we step aside of the Gibbs loss and focus back to the Bayes loss. The derivative of the Bayes loss at \(\lambda=1\) can also be written as, 
\begin{equation}\label{eq:likelihoodtempering:bayesgradient2}
    \frac{d}{d\lambda} B(p_{\lambda})_{|\lambda=1} = -  \mathrm{COV}_{p_{\lambda=1}}\left( n\hat{L}(D,\bmtheta), S_{p_{\lambda=1}}(\bmtheta) \right)\,,
\end{equation}
where for any posterior $\rho$, \(S_{\rho}(\bmtheta)\) is a (negative) performance measure defined as
\begin{equation}
    S_\rho(\bmtheta) = -\E_{\nu}\left[\frac{ p(\bmy|\bmx,\bmtheta)}{\E_{\rho}[p(\bmy|\bmx,\bmtheta)]} \right]\,.
\end{equation}
This function measures the relative performance of a model parameterized by $\bmtheta$ compared to the average performance of the models weighted by $\rho$. Such measure is conducted on samples from the data-generating distribution $\nu(\bmy,\bmx)$. Specifically, if the model $\bmtheta$ outperforms the average, we have \(S_\rho(\bmtheta) < -1\), and if the model performs worse than the average, we have \(S_p(\bmtheta) > -1\) (i.e., the lower the better). The derivations of the above equations are given in Appendix \ref{app:sec:data-augmentation}.

According to Definition \ref{def:likelihoodtempering:CPE} and Equation~\ref{eq:likelihoodtempering:bayesgradient2}, DA will induce a stronger CPE if and only if the following condition is satisfied: 
\begin{equation}\label{eq:likelihoodtempering:bayesgradient3} 
    \mathrm{COV}_{p_{\lambda=1}^{\text{\tiny{DA}}}}\left(n \hat{L}_{\text{\tiny{DA}}}(D,\bmtheta), S_{p_{\lambda=1}^{\text{\tiny{DA}}}}(\bmtheta) \right) > \mathrm{COV}_{p_{\lambda=1}}\left(n \hat{L}(D,\bmtheta), S_{p_{\lambda=1}}(\bmtheta) \right)\,.
\end{equation}

The previous analysis on the Gibbs loss remains applicable in this context, with the use of \(S_\rho(\bmtheta)\) as a metric for the expected performance on the true data-generating distribution instead of \(L(\bmtheta)\). While these metrics are slightly different, it is reasonable to assume that the same arguments we presented to explain the CPE under data augmentation for the Gibbs loss also apply here. The theoretical analysis aligns with the behavior of the Bayes loss as depicted in Figure \ref{fig:CPE:smallAug}.

Finally, comparing Figure \ref{fig:DA:CPE:augSmall}  with Figure \ref{fig:DA:CPE:augLarge}, we also notice that using a larger neural network enables us to mitigate the CPE because we reduce the underfitting introduced by DA.

\subsection{Related work of the data augmentation argument} 
The relation between data augmentation and CPE is an active topic of discussion \citep{WRVS+20,IVHW21,NRBN+21,NGGA+22}. Some studies suggest that CPE is an artifact of DA because turning off data augmentation is enough to eliminate the CPE \citep{IVHW21,FGAO+22}. Our study shows that this is \textit{much more} than an artifact, as also argued in \cite{NGGA+22}. As discussed, the (pseudo) log-likelihood induced by standard DA is a better proxy of the expected log-loss, as precisely defined by Equation~\ref{eq:da:cpe:gibbs} and Equation~\ref{eq:likelihoodtempering:bayesgradient3}. 

Some argue that when using DA, we are not using a proper likelihood function \citep{IVHW21}, which could be a problem. Recent works \citep{NGGA+22} have developed principle likelihood functions that integrate DA-based approaches, hoping to remove CPE. However, they find that CPE still persists. Another widely accepted viewpoint regarding the interplay between the CPE and DA is that DA increases the effective sample size \citep{IVHW21,NRBN+21}: ``intuitively, data augmentation increases the amount of data observed by the model, and should lead to higher posterior contraction'' \citep{IVHW21}.

Our analysis provides a more nuance understanding of this interplay between CPE and DA. First, we show that, when the augmented data  
provide extra information about the data-generating process, there is a stronger CPE, as shown in Equations \ref{eq:da:cpe:gibbs} and \ref{eq:likelihoodtempering:bayesgradient3}. This, in turn, leads to higher posterior concentration. But, we also show that higher posterior concentration in the context of non-meaningful DA does not improve performance; as discussed before, Figure \ref{fig:DA:CPE:permSmall} illustrates this situation. Using the analysis given in Section \ref{sec:Bayesian}, we can also add that tempering the posterior under DA is again a way to define alternative Bayesian posteriors that addresses this stronger underfitting, i.e., they better fit the training data and improve generalization.

\section{Conclusions and limitations}

Our research contributes to understanding the cold posterior effect (CPE) and its implications for Bayesian deep learning in several ways. Firstly, we theoretically demonstrate that the presence of the CPE implies that the Bayesian posterior is underfitting. Secondly, by building on \cite{zeno2021why}, we show that, in general, any tempered posterior can be considered as a proper Bayesian posterior with an alternative likelihood and prior distribution jointly parametrized by $T$, beyond merely the case of classification. Hence, fine-tuning the temperature parameter $T$ serves as an effective and theoretically sound approach to addressing the underfitting of the Bayesian posterior. Furthermore, we comprehensively discuss the interplay between several factors and CPE, including the use of approximate versus exact inference, model misspecification, and the size of the model and samples. Finally, our analysis in Section \ref{sec:data-augmentation} reveals that data augmentation exacerbates the CPE by intensifying underfitting. This occurs because augmented data provides richer and more reliable information, enhancing the capacity for fitting.

Overall, our theoretical analysis underscores the significance of the CPE as an indicator of underfitting within the Bayesian framework and promotes the fine-tuning of the temperature $T$ in tempered posteriors as a principled approach to mitigate this issue. Furthermore, by dissecting the nature of the CPE and its effect on the Bayesian principle, our work aims to resolve ongoing debates and clarify the role of cold posteriors in enhancing the predictive performance of Bayesian deep learning models.

As a limitation of this work, we want to highlight that the characterization of CPE proposed here is defined only as the local change of Bayes loss at \(\lambda = 1\). This approach does not account for scenarios where significant decreases in Bayes loss at other \(\lambda\) values might also indicate the presence of CPE. We believe that our theoretical analysis could be expanded to include these cases as well.

\section{Appendix}
\subsection{Proofs for Section~\ref{sec:CPE}}\label{app:sec:cpe-underfitting}
In this section, we provide the proofs for Section \ref{sec:CPE} in the following order. We first prove the derivative of the empirical Gibbs loss in Proposition \ref{prop:likelihoodtempering:empiricalGibbsLoss}. Then, we show in Proposition \ref{prop:constant_variance} that for meaningful posteriors (depends on training data), the derivative won't be zero. Before proving Proposition \ref{thm:likelihoodtempering:CPE:Neccesary} and Theorem \ref{cor:likelihoodtempering:bayesposterior_optimality}, we first provide Proposition \ref{thm:likelihoodtempering:bayesgradient}, stating an alternative expression of the derivative of the Bayes loss. The proofs of Proposition \ref{thm:likelihoodtempering:CPE:Neccesary} and Theorem \ref{cor:likelihoodtempering:bayesposterior_optimality} then follow from that.

\subsubsection{Proof of Proposition~\ref{prop:likelihoodtempering:empiricalGibbsLoss}}
We first show a slightly more general result of $\frac{d}{d\lambda}\E_{p_\lambda}[f(\bmtheta)]$ for any function $f(\bmtheta)$ that is independent of $\lambda$. Recall that the posterior $p_\lambda(\bmtheta|D) \propto p(D|\bmtheta)^\lambda p(\bmtheta)$. With the fact that $\frac{d}{d\lambda} \left(p(D|\bmtheta)^\lambda p(\bmtheta)\right) =  \ln (p(D|\bmtheta))p(D|\bmtheta)^\lambda p(\bmtheta)$,
the derivative    
\begin{equation}\label{eq:grad-gibbs-f-likelihood}
    \frac{d}{d\lambda} \E_{p_\lambda}[f(\bmtheta)]=\E_{p_\lambda}[ \ln p(D|\bmtheta)f(\bmtheta)]-\E_{p_\lambda}[ \ln p(D|\bmtheta)]\E_{p_\lambda}[f(\bmtheta)]=\mathrm{COV}_{p_\lambda}( \ln p(D|\bmtheta),f(\bmtheta))\,,
\end{equation}
where we denote $\mathrm{COV}(X,Y)$ as the covariance of $X$ and $Y$. Hence, the derivative of the empirical Gibbs loss
\begin{align*}
    \frac{d}{d\lambda} \hat G(p_\lambda,D)
    &=\frac{d}{d\lambda} \E_{p_\lambda}[-\ln p(D|\bmtheta)]\\
    &=\mathrm{COV}_{p_\lambda}(\ln p(D|\bmtheta),
    -\ln p(D|\bmtheta))\\
    &=-\mathbb{V}_{p_\lambda}(\ln p(D|\bmtheta))\,.
\end{align*}

\subsubsection{Proposition \ref{prop:constant_variance}}
\begin{proposition}\label{prop:constant_variance}
    For any \(\lambda > 0\) and \(D \neq \emptyset\), if the tempered posterior \(p_\lambda(\bmtheta|D) \propto p(D|\bmtheta)^\lambda p(\bmtheta)\) satisfies \(\mathbb{V}_{p_\lambda}(\ln P(D|\bmtheta)) = 0\), then, \(p_\lambda(\bmtheta|D) = p(\bmtheta)\).
\end{proposition}
\begin{proof}
    First of all, note that the tempered posterior is defined as
    \begin{equation*}
        p_\lambda(\bmtheta|D) = \frac{p(D|\bmtheta)^\lambda p(\bmtheta)}{\int_\bmtheta p(D|\bmtheta)^\lambda p(\bmtheta)}\,.
    \end{equation*}
    Then,
    \begin{equation*}
        \mathbb{V}_{p_\lambda}(\ln p(D|\bmtheta)) = 0  \implies \int_{\bmtheta}p_\lambda(\bmtheta|D) \left(\ln p(D|\bmtheta) - \E_{p_\lambda}[\ln p(D|\bmtheta)]\right)^2  = 0   
    \end{equation*}
    Thus, for any \(\bmtheta \in \text{supp}(p_\lambda)\), it verifies that
    \begin{equation*}
        \ln p(D|\bmtheta) = \E_{p_\lambda}[\ln p(D|\bmtheta)]\,.
    \end{equation*}
    That is, \(\ln p(D|\bmtheta)\) is constant in the support of \(p_\lambda\). Let \(c\) denote such constant, then
    \begin{equation*}
        p_\lambda(\bmtheta|D) = \frac{e^{c\lambda} p(\bmtheta)}{\int_\bmtheta e^{c\lambda} p(\bmtheta)} =  \frac{e^{c\lambda} p(\bmtheta)}{e^{c\lambda} \int_\bmtheta  p(\bmtheta)} = p(\bmtheta)\,.
    \end{equation*}
\end{proof}

\subsubsection{Proof of Proposition \ref{thm:likelihoodtempering:CPE:Neccesary} and Theorem \ref{cor:likelihoodtempering:bayesposterior_optimality}}
In order to prove Proposition \ref{thm:likelihoodtempering:CPE:Neccesary} and Theorem \ref{cor:likelihoodtempering:bayesposterior_optimality}, we first show in Proposition \ref{thm:likelihoodtempering:bayesgradient} that the derivative of the Bayes loss of the tempered posterior $p_\lambda$ can be expressed by the difference between the empirical Gibbs loss of $\bar p_\lambda$ and the empirical Gibbs loss of $p_\lambda$.

\begin{proposition}\label{thm:likelihoodtempering:bayesgradient}
The  derivative of the Bayes loss of the tempered posterior $p_\lambda$ can be expressed by 
\begin{equation}\label{eq:likelihoodtempering:bayesgradient}
\frac{d}{d\lambda} B(p_\lambda)  = \hat{G}(\bar{p}_\lambda,D) - \hat{G}(p_\lambda,D)\,.
\end{equation}
\end{proposition}

\begin{proof}
By definition,
\begin{equation*}
    \frac{d}{d\lambda} B(p_\lambda) = \frac{d}{d\lambda} \E_\nu[{-\ln\E_{p_\lambda}[p(\bmy|\bmx,\bmtheta)] }] = -\E_\nu\left[{\frac{d}{d\lambda} \ln\E_{p_\lambda}[p(\bmy|\bmx,\bmtheta)]}\right],
\end{equation*}
where
\begin{equation*}
    \frac{d}{d\lambda} \ln\E_{p_\lambda}[p(\bmy|\bmx,\bmtheta) ] = \frac{\frac{d}{d\lambda} \E_{p_\lambda}[p(\bmy|\bmx,\bmtheta)]}{\E_{p_\lambda}[p(\bmy|\bmx,\bmtheta)]} = \frac{\mathrm{COV}_{p_\lambda}(\ln p(D|\bmtheta),p(\bmy|\bmx,\bmtheta)) }{\E_{p_\lambda}[p(\bmy|\bmx,\bmtheta)]}
\end{equation*}
due to Equation~\ref{eq:grad-gibbs-f-likelihood}. By expanding the covariance, the above formula further equals to
\begin{align*}
    &\frac{\E_{p_\lambda}[\ln p(D|\bmtheta)p(\bmy|\bmx,\bmtheta)]-\E_{p_\lambda}[\ln p(D|\bmtheta)]\E_{p_\lambda}[p(\bmy|\bmx,\bmtheta)]}{\E_{p_\lambda}[p(\bmy|\bmx,\bmtheta)]} \\
    &= \E_{\tilde p_\lambda}[\ln p(D|\bmtheta)] - \E_{p_\lambda}[\ln p(D|\bmtheta)]\,,
\end{align*}
where the probability distribution $\tilde p_\lambda(\bmtheta|D,(\bmy,\bmx))\propto p_\lambda(\bmtheta|D) p(\bmy|\bmx,\bmtheta)$. Put everything together, we have
\begin{equation}\label{eq:grad-bayes}
    \frac{d}{d\lambda} B(p_\lambda)=\E_{p_\lambda}[\ln p(D|\bmtheta)]-\E_\nu[\E_{\tilde p_\lambda}[\ln p(D|\bmtheta)]] = \E_{p_\lambda}[\ln p(D|\bmtheta)]-\E_{\bar p_\lambda}[\ln p(D|\bmtheta)]\,,
\end{equation}
where
\[\bar p_\lambda(\bmtheta|D)=\E_\nu [\tilde p_\lambda(\bmtheta|D,(\bmy,\bmx))]=\E_\nu\left[\frac{ p_\lambda(\bmtheta|D)p(\bmy|\bmx,\bmtheta)}{\E_{p_\lambda}[p(\bmy|\bmx,\bmtheta)]}\right].\] 
The last equality is because
\begin{align*}
    \E_\nu[\E_{\tilde p_\lambda}[\ln p(D|\bmtheta)]] &= \int_{(\bmy,\bmx)}\nu(\bmy,\bmx)\int_{\bmtheta} \tilde p_\lambda(\bmtheta|D,(\bmy,\bmx)) \ln p(D|\bmtheta)\, d\bmtheta \, d(\bmy,\bmx)\\
    &=\int_\bmtheta \int_{(\bmy,\bmx)}\nu(\bmy,\bmx)\tilde p_\lambda(\bmtheta|D,(\bmy,\bmx))\, d(\bmy,\bmx) \ln p(D|\bmtheta)\, d\bmtheta\\
    &=\int_\bmtheta \E_\nu [\tilde p_\lambda(\bmtheta|D,(\bmy,\bmx))] \ln p(D|\bmtheta)\, d\bmtheta\\
    &=\E_{\bar p_\lambda}[\ln p(D|\bmtheta)]\,.
\end{align*}
The last expression in Equation \ref{eq:grad-bayes} further equals to $\hat G(\bar p_\lambda,D) - \hat G(p_\lambda,D)$ by definition.
\end{proof}

\paragraph{Proof of Proposition \ref{thm:likelihoodtempering:CPE:Neccesary}}

Note that for any distribution $\rho$, we have  $\hat{G}(\rho,D):=\E_\rho{-\ln p(D|\bmtheta)}\geq \min_\bmtheta -\ln p(D|\bmtheta)$. On the other hand, Proposition \ref{thm:likelihoodtempering:bayesgradient} together with Definition \ref{def:likelihoodtempering:CPE} give that the CPE takes place if and only if
\[
\frac{d}{d\lambda} B(p_\lambda)_{|\lambda=1} =  \hat G(\bar p_{\lambda=1},D) - \hat G(p_{\lambda=1},D) <0\,.
\] 
Therefore, it is not possible to have $\hat{G}(p_{\lambda=1},D)\not> \min_\bmtheta -\ln p(D|\bmtheta)$ and, at the same time, $\hat{G}(\bar{p}^{\lambda=1},D) < \hat{G}(p_{\lambda=1},D)$ because  $\hat{G}(\bar{p}^{\lambda=1},D)\geq \min_\bmtheta -\ln p(D|\bmtheta)$.

\paragraph{Proof of Theorem \ref{cor:likelihoodtempering:bayesposterior_optimality}}
It's easy to see from Proposition \ref{thm:likelihoodtempering:bayesgradient} that
\[
\frac{d}{d\lambda} B(p_\lambda)_{|\lambda=1} =  \hat G(\bar p_{\lambda=1},D) - \hat G(p_{\lambda=1},D) =0
\]
if and only if $\hat G(\bar p_{\lambda=1},D) = \hat G(p_{\lambda=1},D)$.

\subsection{Proofs for Section~\ref{sec:Bayesian}}\label{app:sec:bayesian}

\subsubsection{Proof of Proposition~\ref{prop:true_posterior}}\label{app:prop:true_posterior}

    First of all, by the definition in Equation \ref{eq:likelihoodTempering}, and assuming a data-independent prior $p(\bmtheta|\bmX)=p(\bmtheta)$, the tempered posterior is given by
    \begin{equation*}
    p_\lambda(\bmtheta|\bmX,\bmY) \propto p(\bmY|\bmX,\bmtheta)^\lambda p(\bmtheta)\,,
    \end{equation*}
    where the tempered likelihood fully factorizes as $p(\bmY|\bmX,\bmtheta)^\lambda=\prod_{(\bmy,\bmx)\in (\bmY,\bmX)} p(\bmy|\bmx,\bmtheta)^\lambda$. Let a similar but $\bmy$-independent function $k(\bmtheta,\bmX,\lambda)=\prod_{\bmx\in\bmX}\int p(\bmy|\bmx,\bmtheta)^\lambda \,d\bmy$.
    Therefore, $p(\bmY|\bmX,\bmtheta)^\lambda p(\bmtheta)=\frac{p(\bmY|\bmX,\bmtheta)^\lambda}{k(\bmtheta,\bmX,\lambda)}  \left(k(\bmtheta,\bmX,\lambda) p(\bmtheta)\right)$
    , where we can let the new prior
    \begin{equation*}
        q(\bmtheta|\bmX,\lambda) \propto p(\bmtheta) k(\bmtheta,\bmX,\lambda)= p(\bmtheta)\prod_{\bmx\in\bmX} \int p(\bmy|\bmx,\bmtheta)^\lambda \,d \bmy\,,
    \end{equation*}
    and the new posterior 
    \begin{equation*}
        q(\bmY|\bmX,\bmtheta,\lambda)=\frac{p(\bmY|\bmX,\bmtheta)^\lambda}{k(\bmtheta,\bmX,\lambda)} = \frac{\prod_{(\bmy,\bmx)\in(\bmY,\bmX)}p(\bmy|\bmx,\bmtheta)^\lambda}{\prod_{\bmx\in\bmX}\int p(\bmy|\bmx,\bmtheta)^\lambda \,d\bmy}=\prod_{(\bmy,\bmx)\in(\bmY,\bmX)} q(\bmy|\bmx,\bmtheta)\,.
    \end{equation*}

\subsubsection{Proof of Proposition~\ref{prop:entropy}}\label{app:entropyproof}

The proof is made using differential entropy, i.e. assuming continuous target values \(\bm{y}\). The only assumption is that Leibniz integral rule holds for \(q(\bmy|\bmx,\bmtheta,\lambda) \ln q(\bmy|\bmx,\bmtheta,\lambda))\), verifying that
\[
\frac{d}{d \lambda} \int  (q(\bmy|\bmx,\bmtheta,\lambda) \ln q(\bmy|\bmx,\bmtheta,\lambda)) \ d\bmy = \int \frac{d}{d \lambda} (q(\bmy|\bmx,\bmtheta,\lambda) \ln q(\bmy|\bmx,\bmtheta,\lambda)) \ d\bmy\,.
\]
In the case of supervised classification problems, we adopt the Shanon entropy, where equality holds naturally 
\[
\frac{d}{d \lambda} \sum_{\bm{y} \in \mathcal{Y}} (q(\bmy|\bmx,\bmtheta,\lambda) \ln q(\bmy|\bmx,\bmtheta,\lambda)) = \sum_{\bm{y} \in \mathcal{Y}} \frac{d}{d \lambda}  (q(\bmy|\bmx,\bmtheta,\lambda) \ln q(\bmy|\bmx,\bmtheta,\lambda)) \,.
\]
From the definition of differential entropy, we got that
\[
H(q(\bmy|\bmx,\bmtheta,\lambda)) = -\int q(\bmy|\bmx,\bmtheta,\lambda) \ln q(\bmy|\bmx,\bmtheta,\lambda) \ d\bmy\,.
\]
Thus, taking derivative w.r.t. \(\lambda\) and exchanging derivative and integral leads to the following expression
\begin{align*}
    \frac{d}{d \lambda} H(q(\bmy|\bmx,\bmtheta,\lambda)) &= -\int \frac{d}{d \lambda} (q(\bmy|\bmx,\bmtheta,\lambda) \ln q(\bmy|\bmx,\bmtheta,\lambda)) \ d\bmy \\
    &= -\int (\ln q(\bmy|\bmx,\bmtheta,\lambda) + 1)\frac{d}{d \lambda}  q(\bmy|\bmx,\bmtheta,\lambda) \ d\bmy\,.
\end{align*}
Using that \(\int \frac{d}{d \lambda}  q(\bmy|\bmx,\bmtheta,\lambda) d\bmy = \frac{d}{d \lambda}  \int q(\bmy|\bmx,\bmtheta,\lambda)  d\bmy = 0\), simplifies the expression as
\[
\frac{d}{d \lambda} H(q(\bmy|\bmx,\bmtheta,\lambda)) = -\int \ln q(\bmy|\bmx,\bmtheta,\lambda)\frac{d}{d \lambda}  q(\bmy|\bmx,\bmtheta,\lambda) \ d\bmy\,.
\]
Let us consider now the second term inside the integral. Using the derivative of the quotient rule leads to the following:
\begin{align*}
    &\frac{d}{d \lambda}  q(\bmy|\bmx,\bmtheta,\lambda) \\
    &= \frac{d}{d \lambda} \frac{p(\bmy|\bmx, \bmtheta)^\lambda}{\int p(\bmy|\bmx, \bmtheta)^\lambda \ d\bmy} \\
    &= \frac{p(\bmy|\bmx, \bmtheta)^\lambda \ln p(\bmy|\bmx, \bmtheta)}{\int p(\bmy|\bmx, \bmtheta)^\lambda \ d\bmy} - \frac{p(\bmy|\bmx, \bmtheta)^\lambda \int p(\bmy|\bmx, \bmtheta)^\lambda \ln p(\bmy|\bmx, \bmtheta) \ d\bmy}{(\int p(\bmy|\bmx, \bmtheta)^\lambda \ d\bmy)^2}\,.
\end{align*}
Where, using the definition of \(q(\bmy|\bmx, \bmtheta, \lambda)\), we got that
\[
\frac{p(\bmy|\bmx, \bmtheta)^\lambda \ln p(\bmy|\bmx, \bmtheta)}{\int p(\bmy|\bmx, \bmtheta)^\lambda \ d\bmy} = q(\bmy|\bmx, \bmtheta, \lambda)\ln p(\bmy|\bmx, \bmtheta)\,,
\]
and
\begin{align*}
    &\frac{p(\bmy|\bmx, \bmtheta)^\lambda \int p(\bmy|\bmx, \bmtheta)^\lambda \ln p(\bmy|\bmx, \bmtheta) \ d\bmy}{(\int p(\bmy|\bmx, \bmtheta)^\lambda \ d\bmy)^2}\\
    &= q(\bmy|\bmx, \bmtheta, \lambda) \int q(\bmy|\bmx, \bmtheta, \lambda) \ln p(\bmy|\bmx, \bmtheta) \ d\bmy\\
    &= q(\bmy|\bmx, \bmtheta, \lambda) \E_q[\ln p(\bmy|\bmx, \bmtheta)]\,.
\end{align*}
As a result, we got that
\begin{align*}
    &\int \ln q(\bmy|\bmx,\bmtheta,\lambda) \frac{d}{d \lambda}  q(\bmy|\bmx,\bmtheta,\lambda) \ d\bmy \\
    &= \E_{q}[\ln p(\bmy|\bmx, \bmtheta)\ln q(\bmy|\bmx,\bmtheta,\lambda)] - \E_{q}[\ln q(\bmy|\bmx,\bmtheta,\lambda)]\E_{q}[\ln p(\bmy|\bmx, \bmtheta)]
\end{align*}
Using \(q(\bmy|\bmx,\bmtheta,\lambda)\) definition again:
\begin{align*}
    &\int \ln q(\bmy|\bmx,\bmtheta,\lambda) \frac{d}{d \lambda}  q(\bmy|\bmx,\bmtheta,\lambda) \ d\bmy \\
    &= \E_{q}[\ln p(\bmy|\bmx, \bmtheta)\ln \frac{p(\bmy|\bmx,\bmtheta)^\lambda}{\int p(\bmy|\bmx,\bmtheta)^\lambda}]  - \E_{q}[\ln\frac{p(\bmy|\bmx,\bmtheta)^\lambda}{\int p(\bmy|\bmx,\bmtheta)^\lambda}]\E_{q}[\ln p(\bmy|\bmx, \bmtheta)]
\end{align*}
Where, expanding the logarithms the denominators cancel each other, leading to
\begin{align*}
    &\int \ln q(\bmy|\bmx,\bmtheta,\lambda) \frac{d}{d \lambda}  q(\bmy|\bmx,\bmtheta,\lambda) \ d\bmy \\
    & = \E_{q}[\ln p(\bmy|\bmx, \bmtheta)\ln p(\bmy|\bmx,\bmtheta)^\lambda]  - \E_{q}[\ln p(\bmy|\bmx,\bmtheta)^\lambda]\E_{q}[\ln p(\bmy|\bmx, \bmtheta)]\\
    & = \lambda \mathbb{V}(\ln p(\bmy|\bmx,\bmtheta))\geq 0
\end{align*}
As a result, the entropy is negative.

\subsection{Proofs for Section~\ref{sec:data-augmentation}}\label{app:sec:data-augmentation}

\subsubsection{Proof of Equation \ref{eq:Gibbs_cpe}}
Note that 
\[
\frac{d}{d\lambda} G(p_{\lambda})= \frac{d}{d\lambda} \E_{p_\lambda}[L(\bmtheta)] = \mathrm{COV}_{p_\lambda}(\ln p(D|\bmtheta), L(\bmtheta)) = \mathrm{COV}_{p_\lambda}(-\hat L(D,\bmtheta), L(\bmtheta)),
\]
where the second equality is by applying Equation~\ref{eq:grad-gibbs-f-likelihood}. By taking $\lambda=1$, we obtain the desired derivative.

\subsubsection{Proof of Equation \ref{eq:likelihoodtempering:bayesgradient2}}

Recall from the proof of Theorem \ref{thm:likelihoodtempering:bayesgradient} that
\[
\frac{d}{d\lambda} B(p_\lambda) = \E_{p_\lambda}[\ln p(D|\bmtheta)]-\E_{\bar p_\lambda}[\ln p(D|\bmtheta)] = \E_{\bar p_\lambda}[\hat L(D,\bmtheta)] - \E_{p_\lambda}[\hat L(D,\bmtheta)],
\]
where $\bar{p}_\lambda(\bmtheta|D) = \E_\nu\left[\tilde p_\lambda(\bmtheta|D,(\bmy,\bmx))\right]$ (Equation~\ref{eq:likelihoodtempering:updatedposterior}), and $\tilde p_\lambda (\bmtheta|D,(\bmy,\bmx)) \propto p_\lambda(\bmtheta|D) p(\bmy|\bmx,\bmtheta)$ is the distribution obtained by updating the posterior $p_\lambda$ with one new sample $(\bmy,\bmx)$.
Therefore,
\begin{align*}
    \E_{\bar p_\lambda}[\hat L(D,\bmtheta)] = \E_\nu \E_{\tilde p_\lambda}[\hat L(D,\bmtheta)] = \E_\nu\left[\E_{p_\lambda}\left[\frac{p(\bmy|\bmx,\bmtheta)}{\E_{p_\lambda}[p(\bmy|\bmx,\bmtheta)]} \hat L(D,\bmtheta) \right]\right].
\end{align*}
By Fubini's theorem, the above formula further equals to
\begin{align*}
    \E_{p_\lambda}\left[\E_\nu\left[\frac{p(\bmy|\bmx,\bmtheta)}{\E_{p_\lambda}[p(\bmy|\bmx,\bmtheta)]} \hat L(D,\bmtheta) \right]\right] 
    &= \E_{p_\lambda}\left[\E_\nu\left[\frac{p(\bmy|\bmx,\bmtheta)}{\E_{p_\lambda}[p(\bmy|\bmx,\bmtheta)]}\right] \hat L(D,\bmtheta) \right] \\
    &= \E_{p_\lambda}\left[-S_{p_\lambda}(\bmtheta)\cdot\hat L(D,\bmtheta)\right].
\end{align*}
On the other hand, since
\[
\E_{p_\lambda}[-S_{p_\lambda}(\bmtheta)] = \E_{p_\lambda}\left[\E_\nu\left[\frac{p(\bmy|\bmx,\bmtheta)}{\E_{p_\lambda}[p(\bmy|\bmx,\bmtheta)]}\right]\right] = \E_\nu\left[\E_{p_\lambda}\left[\frac{p(\bmy|\bmx,\bmtheta)}{\E_{p_\lambda}[p(\bmy|\bmx,\bmtheta)]}\right]\right]=1\,,
\]
we have
\[
\E_{p_\lambda}[\hat L(D,\bmtheta)] = \E_{p_\lambda}[\hat L(D,\bmtheta)] \E_{p_\lambda}[-S_{p_\lambda}(\bmtheta)]\,.
\]
By putting them altogether,
\begin{align*}
    \frac{d}{d\lambda} B(p_\lambda) &= \E_{p_\lambda}\left[-S_{p_\lambda}(\bmtheta)\cdot\hat L(D,\bmtheta)\right]-\E_{p_\lambda}[\hat L(D,\bmtheta)] \E_{p_\lambda}[-S_{p_\lambda}(\bmtheta)] \\
    &= -\mathrm{COV}\left(\hat L(D,\bmtheta), S_{p_\lambda}(\bmtheta) \right)\,.
\end{align*}

\subsection{Experiment details for Bayesian linear regression on synthetic data with exact inference}\label{app:sec:experiment}
In this section we detail the settings of the toy experiment using synthetic data and exact Bayesian linear regression in Figures \ref{fig:blr} and \ref{fig:blr2}. We also show extra results of the derivative of Gibbs loss and Bayes loss w.r.t to $\lambda$ approximated by samples.  

To begin, we will outline the data-generating process for the synthetic data used in the experiment shown in Figures \ref{fig:blr} and \ref{fig:blr2} and Figures \ref{fig:blr_appendix} and \ref{fig:blr2_appendix}. We sample $x$ uniformly from the $[-1,1]$ interval and pass it through a Fourier transformation to construct the input of the data. That is, for a sampled $x$, the input $\bmx$ is constructed by a 10-dimensional Fourier basis function $\bmphi(x)=[g_1(x), ..., g_K(x)]^T$ for $K=10$, where the basis functions are defined as follows: $g_1(x)=\dfrac{1}{\sqrt{2\pi}}$, and for other odd values of $k$, $g_k(x)=\dfrac{1}{\sqrt{\pi}}\sin{(kx)}$, whereas for even values of $k$, $g_k(x)=\dfrac{1}{\sqrt{\pi}}\cos{(kx)}$. The distribution of the output $y\in\mathbb{R}$ given an input $\bmx$, denoted as $\nu(y|\bmx)$, follows a Normal distribution with mean $\mathbf{1}^T\bmx$ and variance $1.0$, where $\mathbf{1}$ is an all-ones vector. That is, $\nu(y|\bmx)=\mathcal{N}(\mathbf{1}^T\bmx,1.0)$.

\begin{figure}[h]
  \centering 

\begin{subfigure}{0.33\textwidth}
\includegraphics[width=\linewidth]{imgs/chap2/blm/nll_perfect.png}
\end{subfigure}\hfil 
\begin{subfigure}{0.315\textwidth}
\includegraphics[width=\linewidth]{imgs/chap2/blm/nll_overfit.png}
\end{subfigure}\hfil 
\begin{subfigure}{0.33\textwidth}
\includegraphics[width=\linewidth]{imgs/chap2/blm/nll_prior_bad.png}
\end{subfigure}

\begin{subfigure}[b]{0.33\textwidth}
\includegraphics[width=\linewidth]{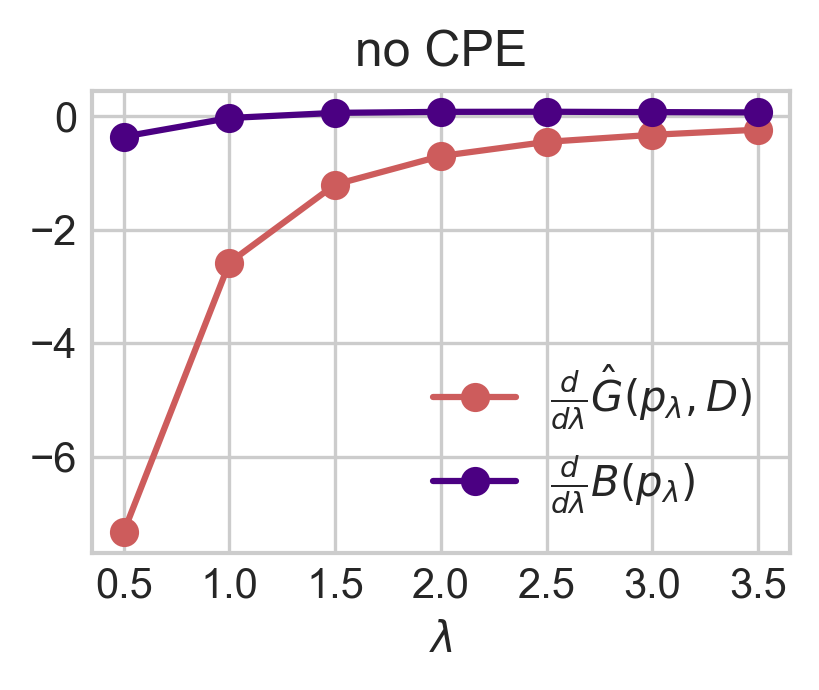}
\captionsetup{format=hang, justification=centering}
\caption{No likelihood or prior misspecification}
\end{subfigure}\hfil 
\begin{subfigure}[b]{0.345\textwidth}
\includegraphics[width=\linewidth]{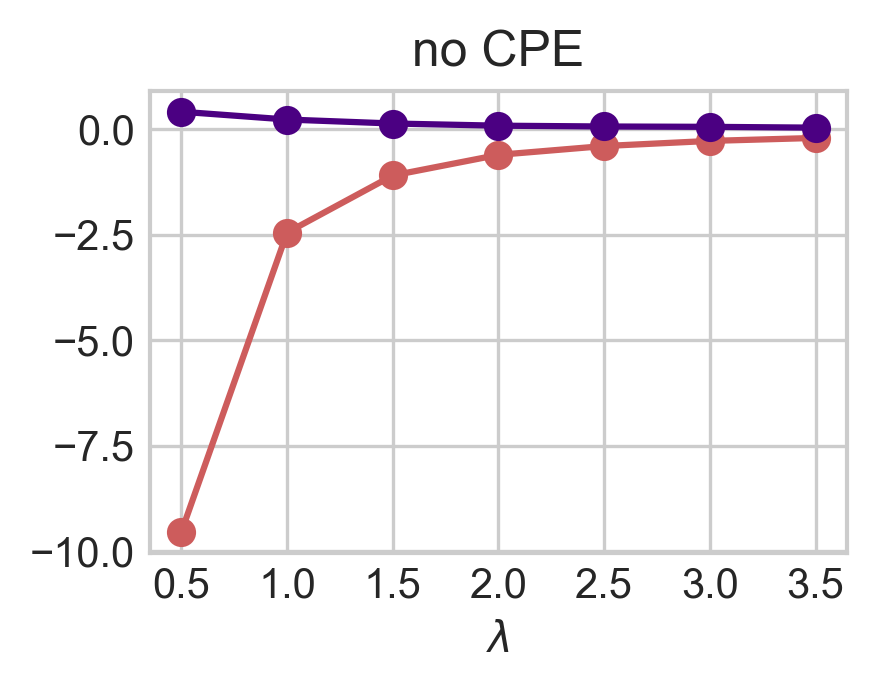}
\captionsetup{format=hang, justification=centering}
\caption{Likelihood misspecification case I}
\end{subfigure}\hfil 
\begin{subfigure}[b]{0.32\textwidth}
\includegraphics[width=\linewidth]{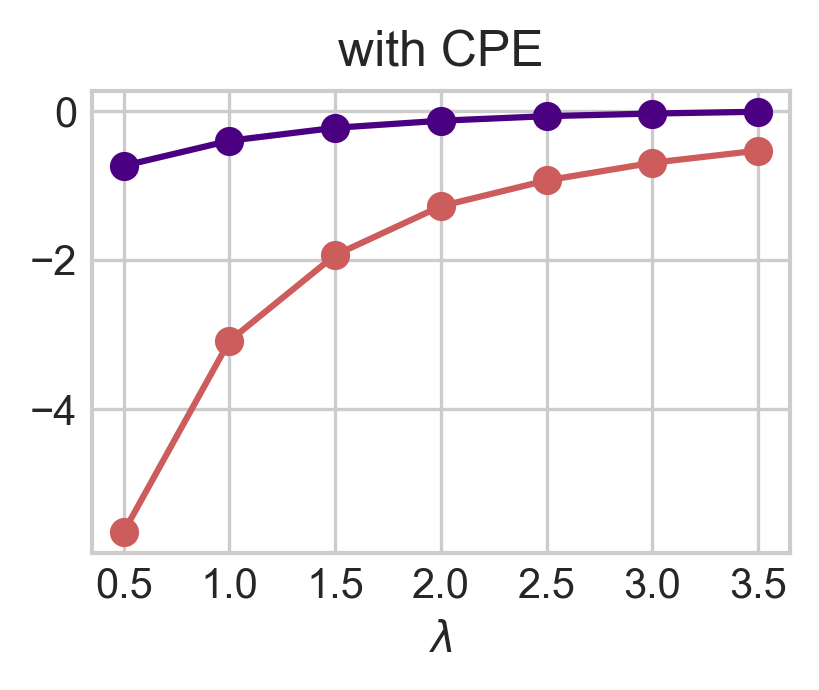}
\captionsetup{format=hang, justification=centering}
\caption{Only with prior misspecification}
\end{subfigure}

\caption{\textbf{The derivatives $\frac{d}{d\lambda} \hat G(p_\lambda, D)$ (Equation \ref{eq:likelihoodtempering:empiricalgibbsgradient}) and $\frac{d}{d\lambda} B(p_\lambda)$ (Equation \ref{eq:likelihoodtempering:bayesgradient}) characterize the Gibbs loss $\hat G(p_\lambda, D)$ and the Bayes loss $B(p_\lambda)$ perfectly.}}
\label{fig:blr_appendix}
\end{figure}

In our experiment, the likelihood model and the prior model are defined differently for the four settings in Figures \ref{fig:blr} and \ref{fig:blr2}. To enable exact inference, both the likelihood and the prior are Gaussian, which gives a closed-form solution for the posterior predictive. This choice also provides convenience when studying the CPE: different values of $\lambda$ on the likelihood term can be naturally absorbed into the Gaussian densities by adjusting the variance (dividing by $\lambda$) without hindering the exact inference step.  We describe them in detail in the following.
\begin{enumerate}
    \item No likelihood or prior misspecification: likelihood $p(y | \bmx, \bmtheta)=\mathcal{N}(\bmtheta^T\bmx,1.0)$, prior $p(\bmtheta)=\mathcal{N}(0,2)$. This is the baseline for comparison. 
    \item Likelihood misspecification case I: likelihood $p(y | \bmx, \bmtheta)=\mathcal{N}(\bmtheta^T\bmx,0.15)$ (the order of Fourier transformation is $K=20$, however note that it still contains the $K=5$ data-generating process in its solution space), prior $p(\bmtheta)=\mathcal{N}(0,2)$. In this case, the model is misspecified in a way that it has a smaller variance than the data-generating process.  
    \item Likelihood misspecification case II: likelihood $p(y | \bmx, \bmtheta)=\mathcal{N}(\bmtheta^T\bmx,3.0)$, prior $p(\bmtheta)=\mathcal{N}(0,2)$. In this case, the model is misspecified in a way that it has a larger variance than the data-generating process. This is similar to one of the scenarios where CPE was found: the curated data has a lower aleatoric uncertainty than the model \citep{Ait21}.
    \item Only with prior misspecification: likelihood $p(y | \bmx, \bmtheta)=\mathcal{N}(\bmtheta^T\bmx,1.0)$, prior $p(\bmtheta)=\mathcal{N}(0,0.5)$. The prior is poorly specified in a way that it is tightly centered at 0 while the best $\bmtheta$ should be 1. 
\end{enumerate}

\begin{figure}[H]
  \centering 

\begin{subfigure}{0.33\textwidth}
\includegraphics[width=\linewidth]{imgs/chap2/blm/nll_perfect.png}
\end{subfigure}\hfil 
\begin{subfigure}{0.33\textwidth}
\includegraphics[width=\linewidth]{imgs/chap2/blm/nll_model_bad.png}
\end{subfigure}\hfil 
\begin{subfigure}{0.33\textwidth}
\includegraphics[width=\linewidth]{imgs/chap2/blm/nll_model_bad_more.png}
\end{subfigure}

\begin{subfigure}[b]{0.33\textwidth}
\includegraphics[width=\linewidth]{imgs/chap2/blm/grad_perfect.png}
\captionsetup{format=hang, justification=centering}
\caption{No likelihood or prior misspecification}
\end{subfigure}\hfil 
\begin{subfigure}[b]{0.33\textwidth}
\includegraphics[width=\linewidth]{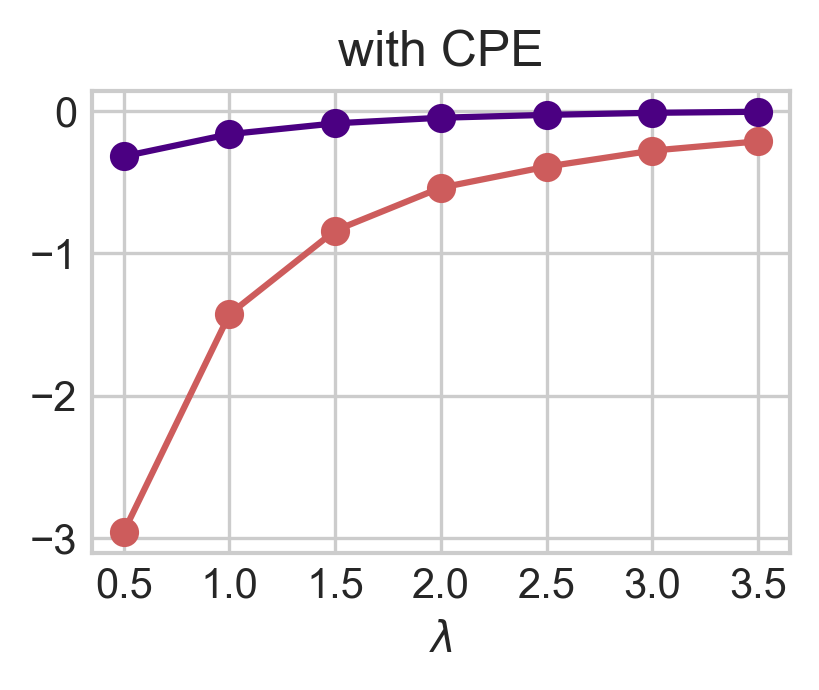}
\captionsetup{format=hang, justification=centering}
\caption{Likelihood misspecification case II}
\end{subfigure}\hfil 
\begin{subfigure}[b]{0.33\textwidth}
\includegraphics[width=\linewidth]{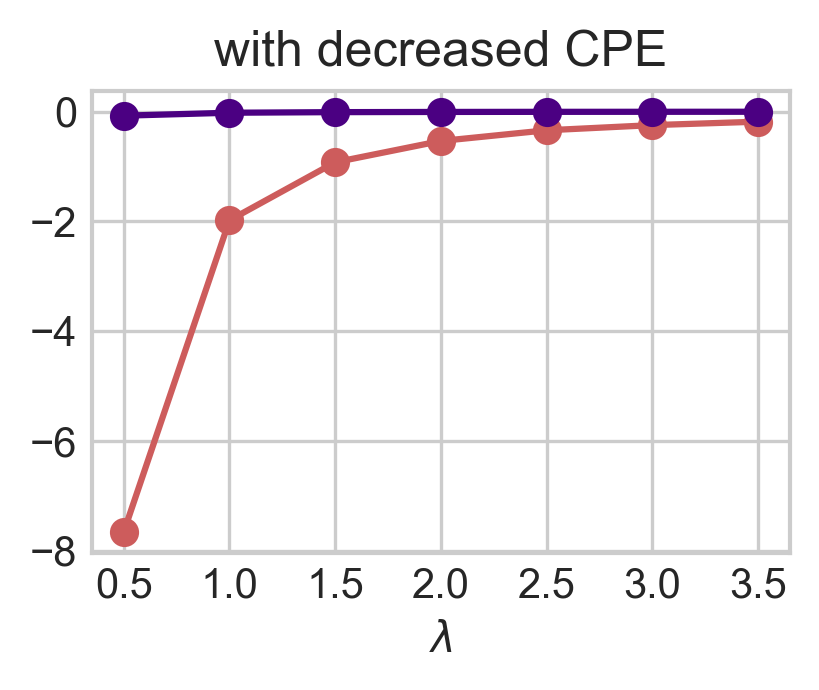}
\captionsetup{format=hang, justification=centering}
\caption{Same as (b) but with 50 data points}
\end{subfigure}

\caption{\textbf{Extended results of Figure \ref{fig:blr_appendix} with more configurations of model misspecification.}}
\label{fig:blr2_appendix}
\end{figure}

In all the experiments, every training set consists of only 5 samples. Since there are more parameters than the number of training data points, our setting falls within the ``overparameterized'' regime where CPE has been observed in Bayesian deep learning \citep{WRVS+20}.

Continuing from Figures \ref{fig:blr} and \ref{fig:blr2}, where we show the Gibbs loss $\hat G(p_\lambda, D)$ (training) and the Bayes loss $B(p_\lambda)$ (testing) with respect to $\lambda$, we now show their derivatives $\frac{d}{d\lambda} \hat G(p_\lambda, D)$ (Equation \ref{eq:likelihoodtempering:empiricalgibbsgradient}) and $\frac{d}{d\lambda} B(p_\lambda)$ (Equation \ref{eq:likelihoodtempering:bayesgradient}) respectively in Figures \ref{fig:blr_appendix} and \ref{fig:blr2_appendix}. Here the losses are included for a clearer depiction of the derivatives. To approximate the Bayes loss for generating the plot, we use 10000 data points sampled from the data-generating distribution. Also, the derivatives are approximated using 10000 samples from the exact posteriors. From Figures \ref{fig:blr_appendix} and \ref{fig:blr2_appendix}, we could clearly see that the derivatives perfectly characterize the losses in all four settings. 

\subsection{Experiment details for Bayesian neural networks on image data with approximate inference}\label{app:sec:experiment-approx}
In this section, we first present in Appendix \ref{app:sec:architecture} the architectures of the small and large CNNs used in this paper. 
As promised in the main text, we then provide results on additional image datasets trained with Stochastic Gradient Langevin Dynamics (SGLD) \citep{DBLP:conf/icml/WellingT11} in Appendix \ref{app:sec:exp_sgld}. 

\subsubsection{Architectures of small/large CNN}\label{app:sec:architecture}
\paragraph{Small CNN}
The small CNN is similar to LeNet-5, but with 107786 parameters in total:
\begin{enumerate}
    \item Convolutional layer 1. Input channels: 1 (assuming grayscale images), output channels: 6, kernel size: 5x5, padding: 2, activation: ReLU.

    \item Average pooling layer 1. Kernel size: 2x2, stride: 2.

    \item Convolutional layer 2. Input channels: 6, output channels: 16, kernel size: 5x5, padding: 2, activation: ReLU.

    \item Average pooling layer 2. Kernel size: 2x2, stride: 2.

    \item Flattening layer. Flattens the output from the previous layers.

    \item Fully connected layer 1. Input features: 784 (16 channels * 7 * 7), output features: 120, activation: ReLU.

    \item Fully connected layer 2. Input features: 120, output features: 84, activation: ReLU.

    \item Fully connected layer 3 (output layer). Input features: 84, output features: num\_classes (specified during instantiation).
\end{enumerate}
\paragraph{Large CNN}
The large CNN is similar to the small CNN, but with 545546 parameters in total:
\begin{enumerate}
    \item Convolutional layer 1. Input channels: 1 (assuming grayscale images), output channels: 6, kernel size: 5x5, padding: 2, activation: ReLU.

    \item Average pooling layer 1. Kernel size: 2x2, stride: 2.

    \item Convolutional layer 2. Input channels: 6, output channels: 16, kernel size: 5x5, padding: 2, activation: ReLU.

    \item Average pooling layer 2. Kernel size: 2x2, stride: 2.

    \item Convolutional layer 3. Input channels: 16, output channels: 120, kernel size: 5x5, padding: 2, activation: ReLU.

    \item Flattening layer. Flattens the output from the previous layers.

    \item Fully connected layer 1. Input features: 5880 (120 channels \(\times\) 7 \(\times\) 7), output features: 84, activation: ReLU.

    \item Fully connected layer 2 (output layer). Input features: 84, output features: num\_classes (specified during instantiation).
\end{enumerate}

In all the convolutional layers, no stride \(= 1\) and padding is set to \emph{same}.

\subsubsection{Stochastic gradient Langevin dynamics (SGLD)} \label{app:sec:exp_sgld}
Our experiments using SGLD are categorized into 4 groups:
\begin{enumerate}
    \item Bayesian CNNs (small and large) on MNIST (Figures \ref{fig:CPE:small} - \ref{fig:CPE:largeAug} in the main text)
    \item Bayesian CNNs (small and large) on Fashion-MNIST (Appendix \ref{app:sec:sgld-Fmnist})
    \item Bayesian ResNets (18 and 50) on CIFAR-10 (Appendix \ref{app:sec:sgld-cifar10})
    \item Bayesian ResNets (18 and 50) on CIFAR-100 (Appendix \ref{app:sec:sgld-cifar100})
\end{enumerate}
where each group evaluates the effect of underfitting on a small model and a large model. Note that as we follow the standard ResNet-18 and ResNet-50, the details of the architectures are omitted. They have around 11 million and 23 million parameters, respectively. We implement with PyTorch \citep{DBLP:conf/nips/PaszkeGMLBCKLGA19} and train the model using cyclical learning rate SGLD (cSGLD) \citep{zhang2019cyclical} for 1000 epochs. We set the learning rate to 1e-6 with a momentum term of 0.99. We run cSGLD for 10 trials and collect 10 samples for each trial. Experiments were conducted on NVIDIA A100 GPU, with each trial taking around 30 hours.

\newpage

\paragraph{Small and large CNNs via SGLD on Fashion-MNIST}\label{app:sec:sgld-Fmnist}

\centerline
\centerline
\centerline

\begin{figure}[H]
  \begin{subfigure}{.32\linewidth}
      \includegraphics[width=\linewidth]{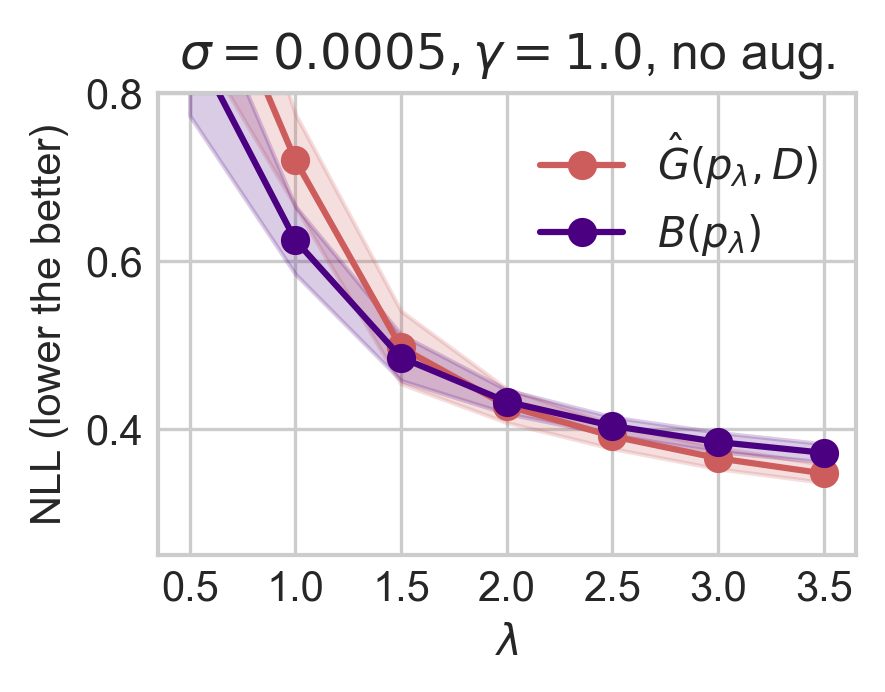}
      \captionsetup{format=hang, justification=centering}
      \caption{Narrow prior and standard softmax}
      \label{fig:CPE:NarrowPriorsmall_fmnist}
  \end{subfigure}
  \hfill
  \begin{subfigure}{.32\linewidth}
      \includegraphics[width=\linewidth]{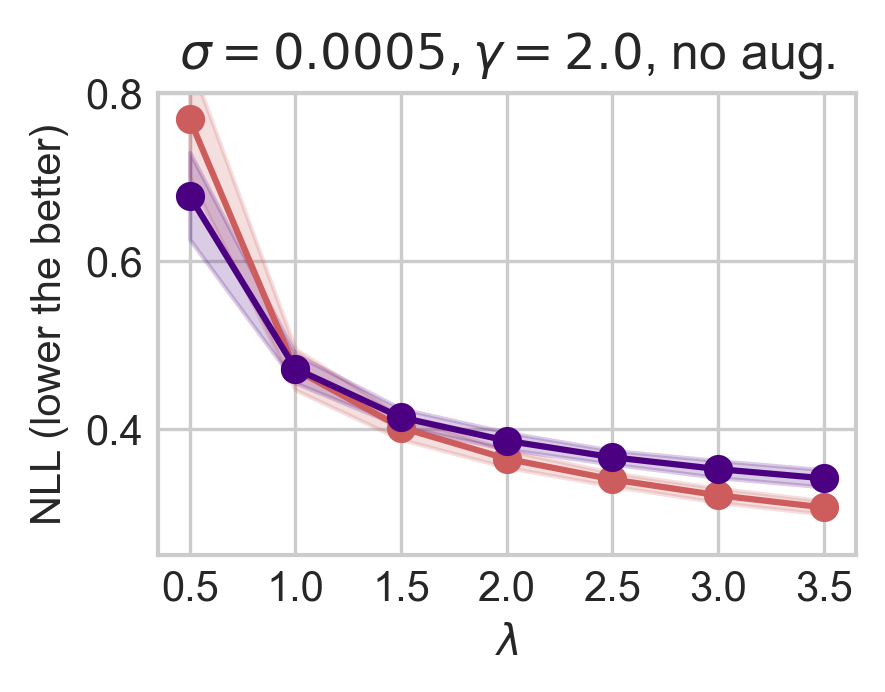}
      \captionsetup{format=hang, justification=centering}
      \caption{Narrow prior and tempered softmax}
      \label{fig:CPE:Softmaxsmall_fmnist}
  \end{subfigure}
  \hfill
  \begin{subfigure}{.32\linewidth}
      \includegraphics[width=\linewidth]{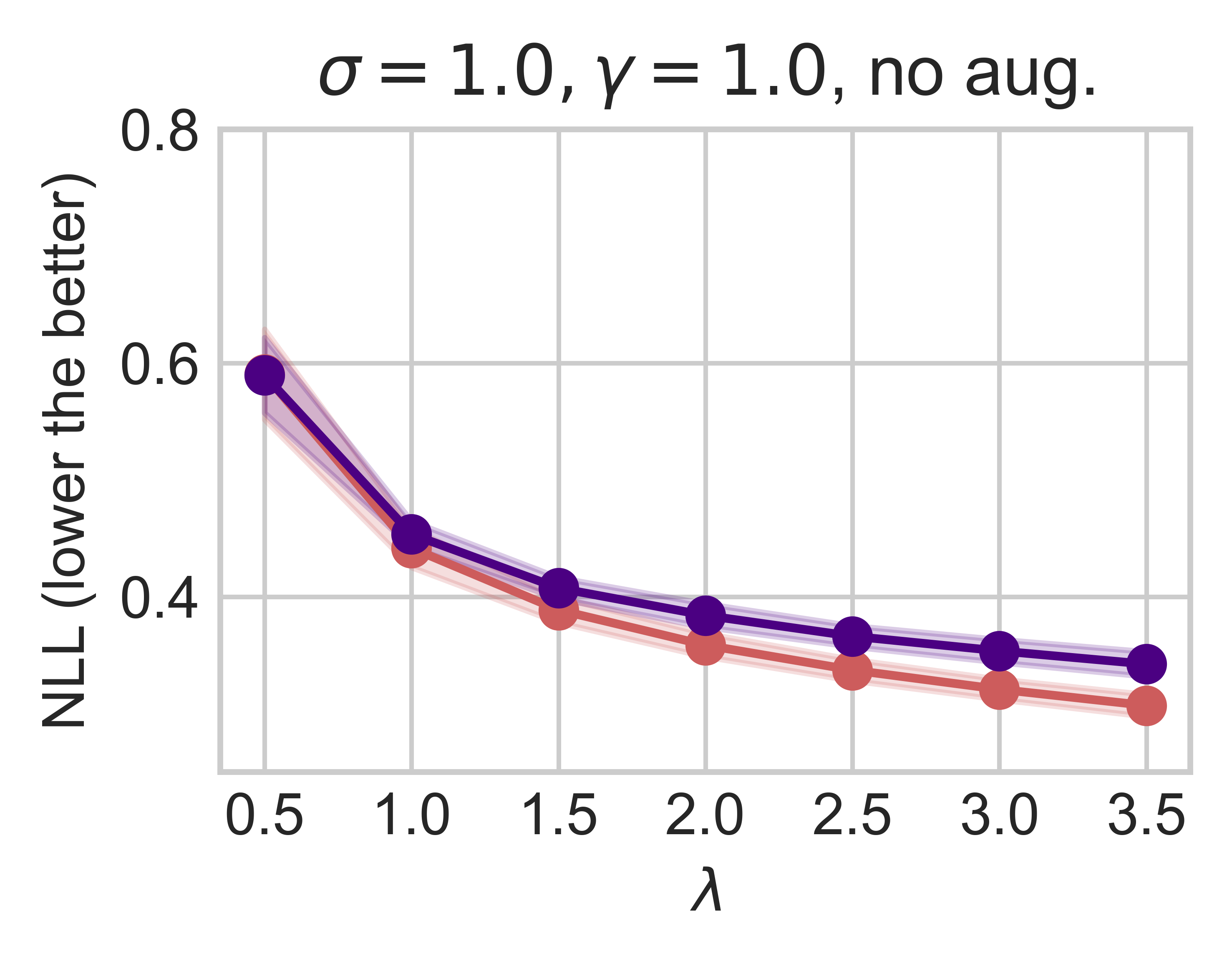}
      \captionsetup{format=hang, justification=centering}
      \caption{Standard prior and standard softmax}
      \label{fig:CPE:StandardPriorsmall_fmnist}
  \end{subfigure}
  \caption{Extended results of Figure \ref{fig:CPE:small} using small CNN via SGLD on Fashion-MNIST.}
  \label{fig:CPE:small_fmnist}
\end{figure}

\centerline
\centerline
\centerline

\begin{figure}[H]

  \begin{subfigure}{.32\linewidth}
      \includegraphics[width=\linewidth]{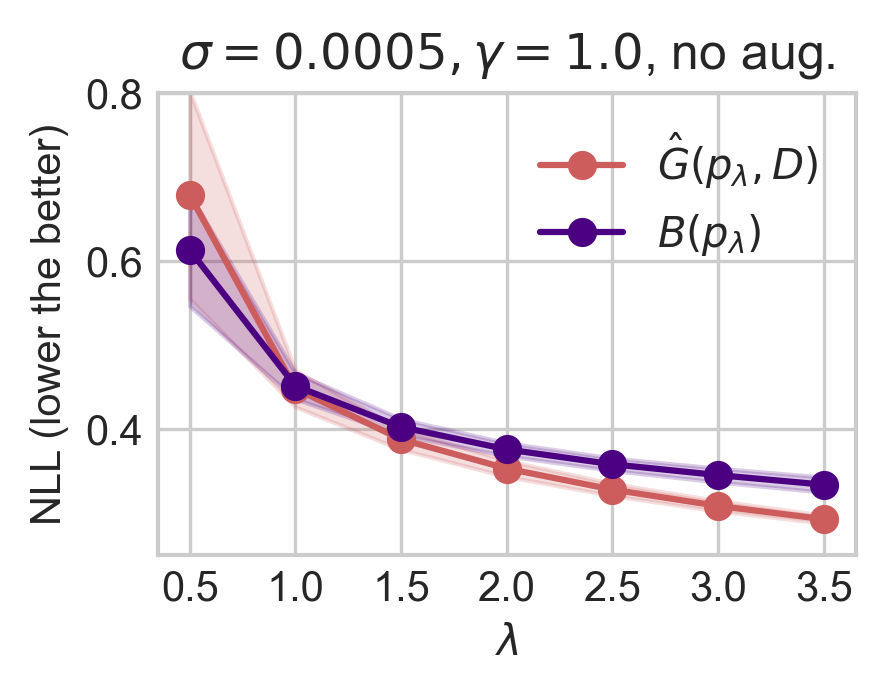}
      \captionsetup{format=hang, justification=centering}
      \caption{Narrow prior and standard softmax}
      \label{fig:CPE:NarrowPriorlarge_fmnist}
  \end{subfigure}
  \begin{subfigure}{.32\linewidth}
      \includegraphics[width=\linewidth]{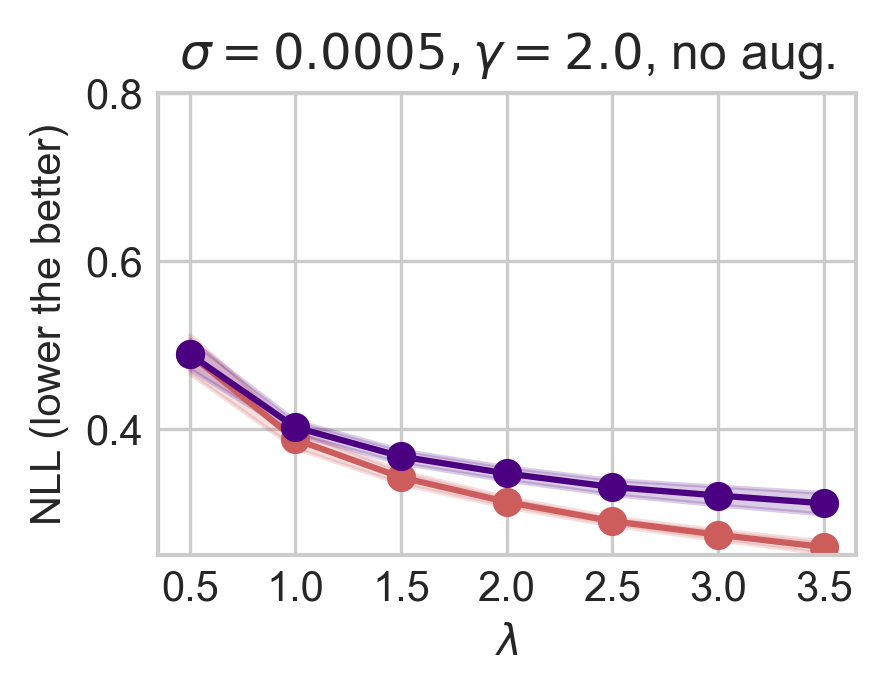}
      \captionsetup{format=hang, justification=centering}
      \caption{Narrow prior and tempered softmax}
      \label{fig:CPE:Softmaxlarge_fmnist}
  \end{subfigure}
  \begin{subfigure}{.32\linewidth}
      \includegraphics[width=\linewidth]{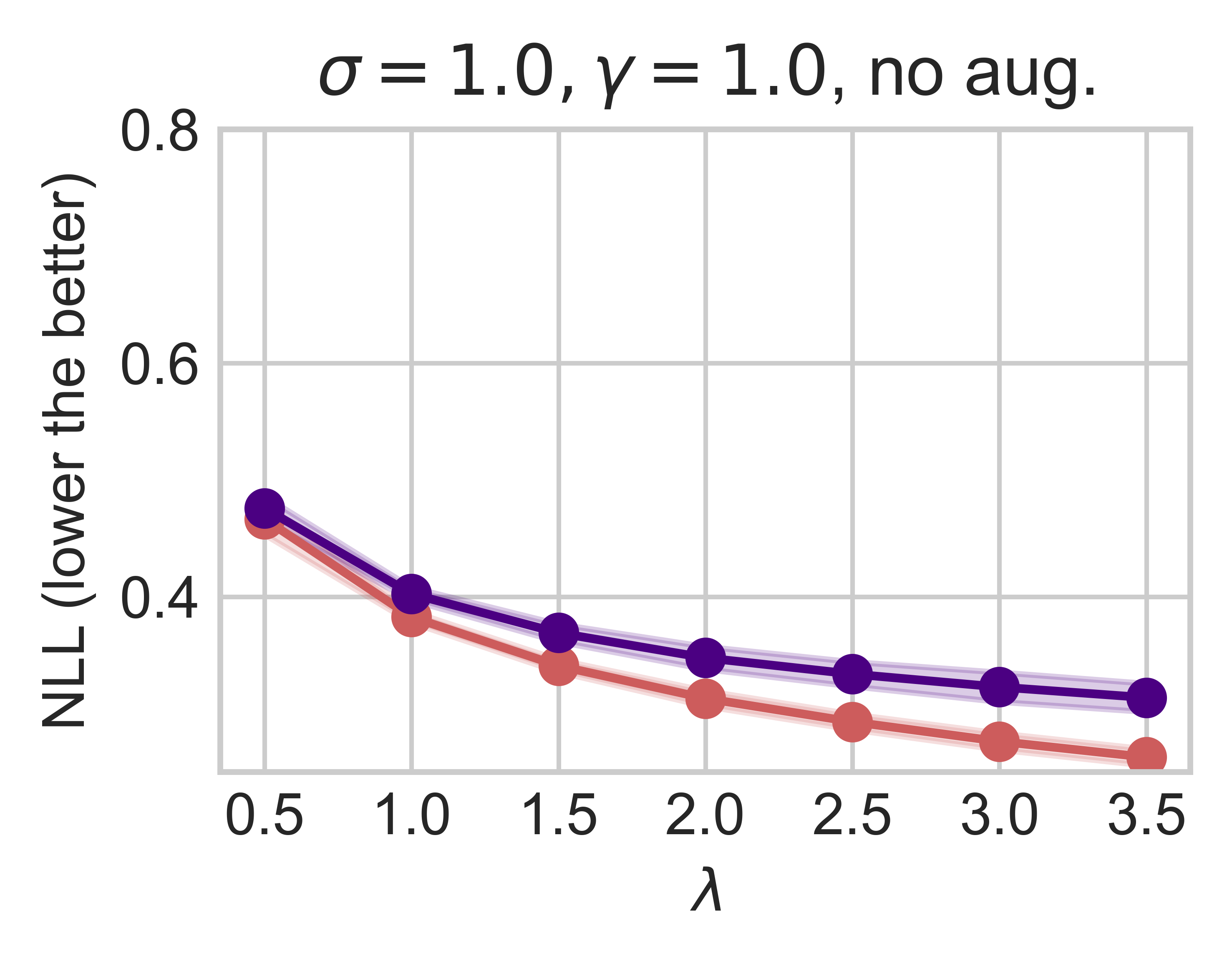}
      \captionsetup{format=hang, justification=centering}
      \caption{Standard prior and standard softmax}
      \label{fig:CPE:StandardPriorlarge_fmnist}
  \end{subfigure}
  \caption{Extended results of Figure \ref{fig:CPE:large} using large CNN via SGLD on Fashion-MNIST.}
\end{figure}

\newpage

\centerline
\centerline
\centerline

\begin{figure}[H]
  \begin{subfigure}{.32\linewidth}
      \includegraphics[width=\linewidth]{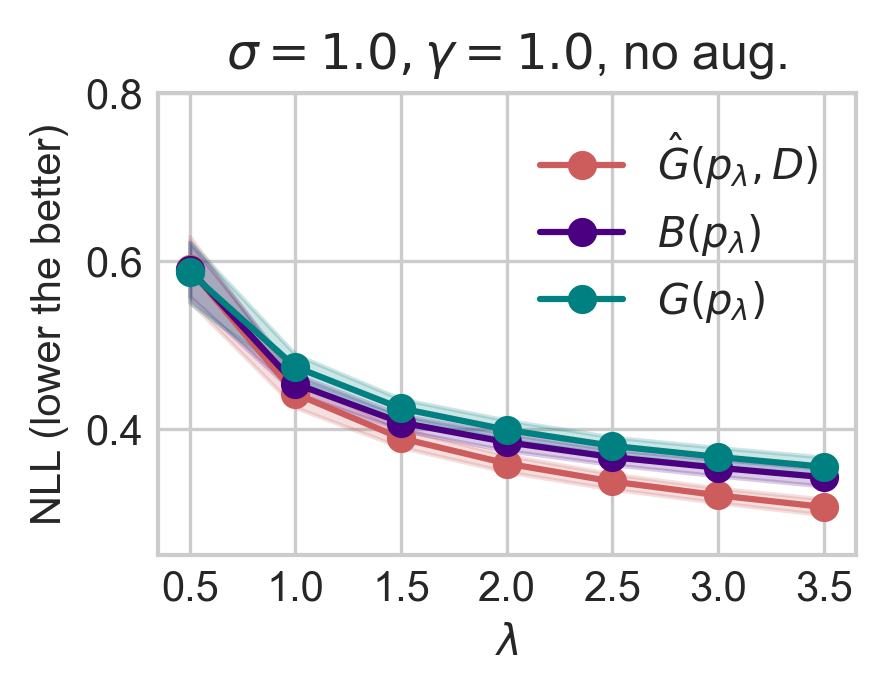}
      \captionsetup{format=hang, justification=centering}
      \caption{Standard prior and standard softmax}
      \label{fig:CPE:fmnist_50}
  \end{subfigure}
  \label{}
  \hfill
  \begin{subfigure}{.32\linewidth}
      \includegraphics[width=\linewidth]{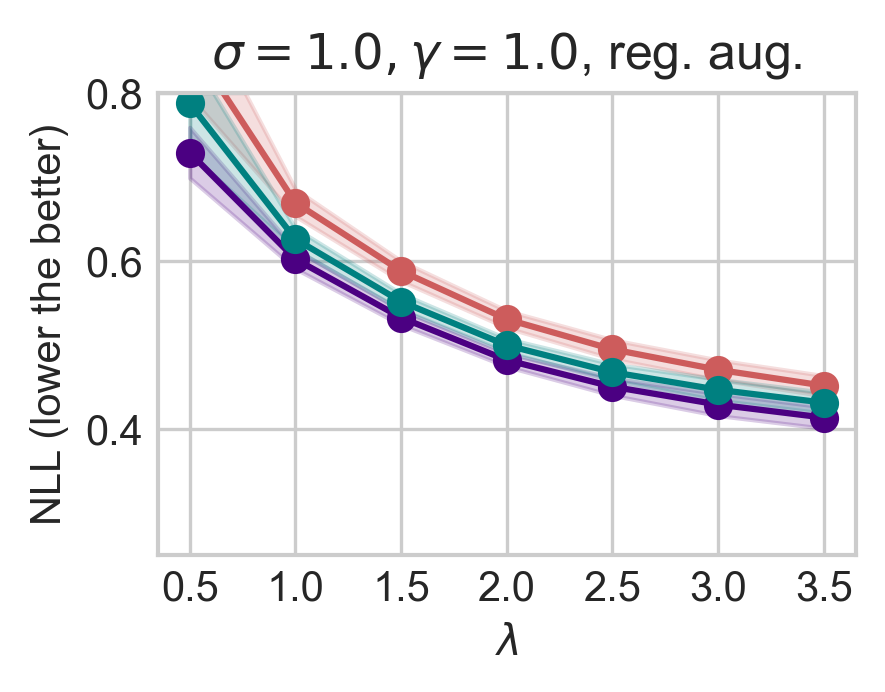}
      \captionsetup{format=hang, justification=centering}
      \caption{Random crop and horizontal flip}
      \label{fig:CPE:fmnist_50_aug}
  \end{subfigure}
  \hfill
  \begin{subfigure}{.32\linewidth}
      \includegraphics[width=\linewidth]{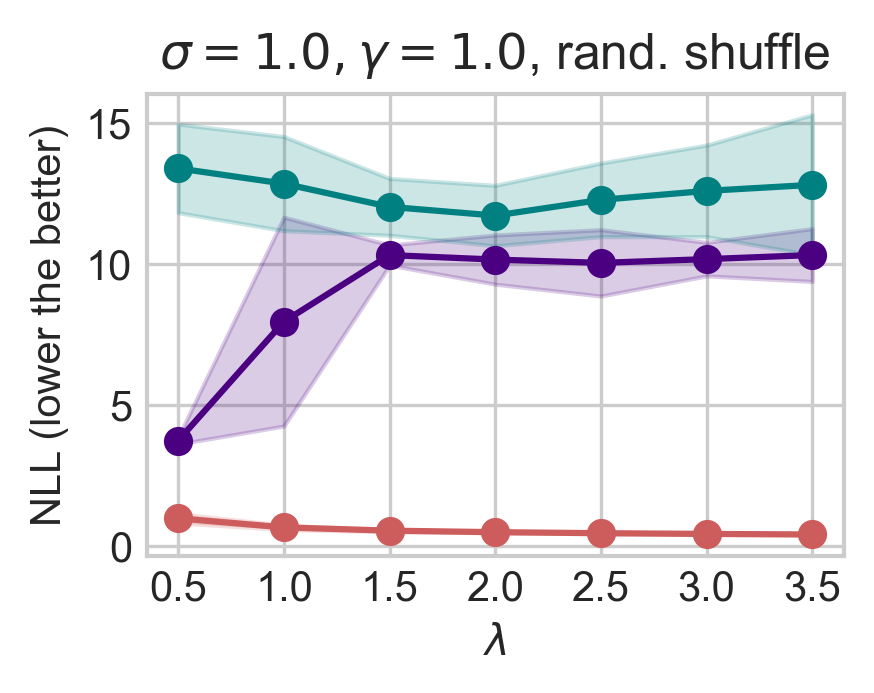}
      \captionsetup{format=hang, justification=centering}
      \caption{Image pixels randomly shuffled}
      \label{fig:CPE:fmnist_50_perm}
  \end{subfigure}
  \caption{Extended results of Figure \ref{fig:CPE:smallAug} using small CNN via SGLD on Fashion-MNIST.}
  \label{fig:CPE:small_fmnist_aug}
\end{figure}

\centerline
\centerline
\centerline

\begin{figure}[H]
  \begin{subfigure}{.32\linewidth}
      \includegraphics[width=\linewidth]{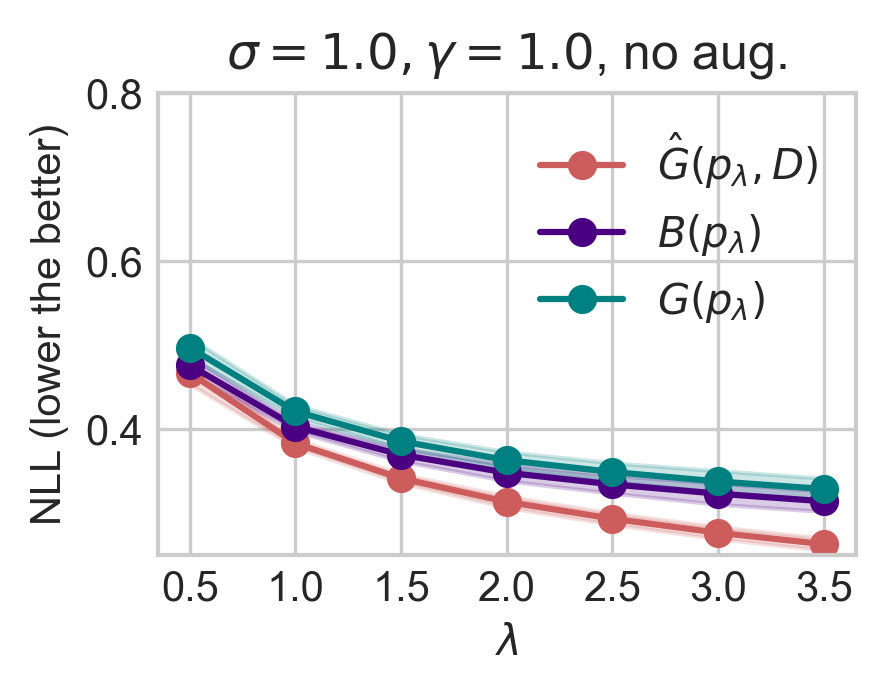}
      \captionsetup{format=hang, justification=centering}
      \caption{Standard prior and standard softmax}
  \end{subfigure}
  \begin{subfigure}{.32\linewidth}
      \includegraphics[width=\linewidth]{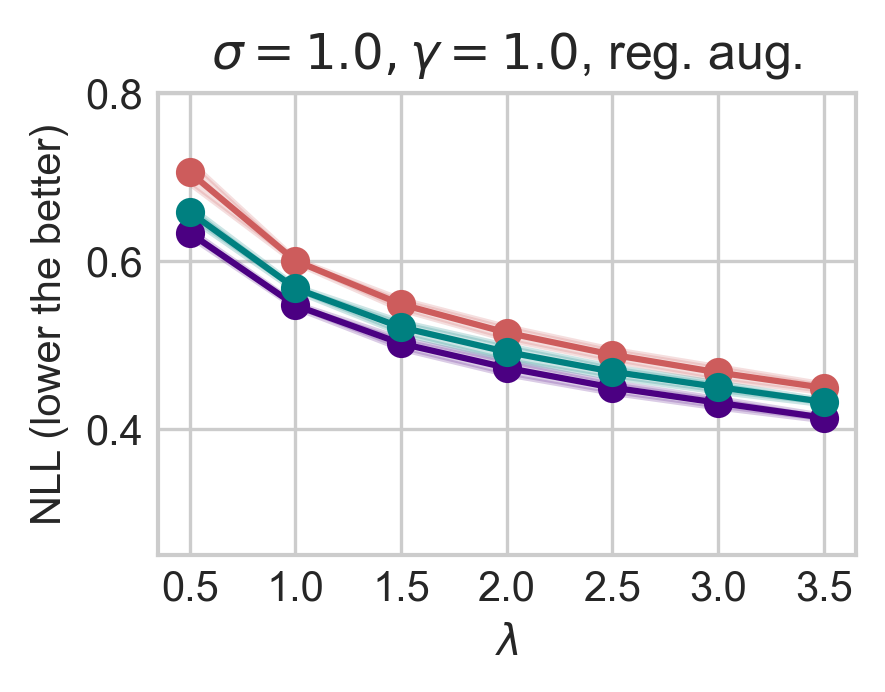}
      \captionsetup{format=hang, justification=centering}
      \caption{Random crop and horizontal flip}
  \end{subfigure}
  \begin{subfigure}{.32\linewidth}
      \includegraphics[width=\linewidth]{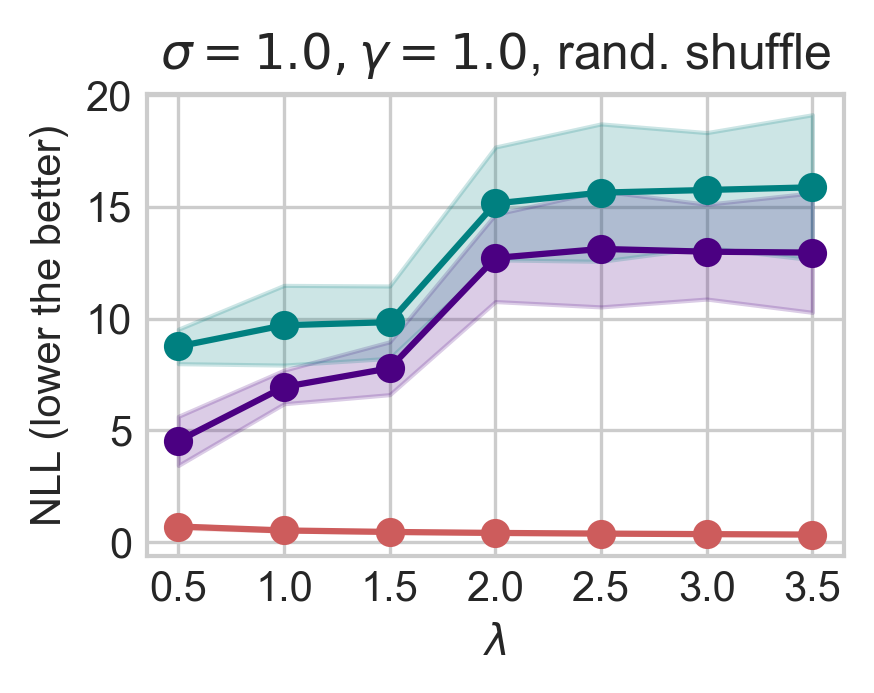}
      \captionsetup{format=hang, justification=centering}
      \caption{Image pixels randomly shuffled}
  \end{subfigure}  
  \caption{Extended results of Figure \ref{fig:CPE:largeAug} using large CNN via SGLD on Fashion-MNIST.}
\end{figure}

\newpage

\paragraph{ResNet-18 and ResNet-50 via SGLD on CIFAR-10}\label{app:sec:sgld-cifar10}

\centerline
\centerline
\centerline

\begin{figure}[H]
  \begin{subfigure}{.32\linewidth}
      \includegraphics[width=\linewidth]{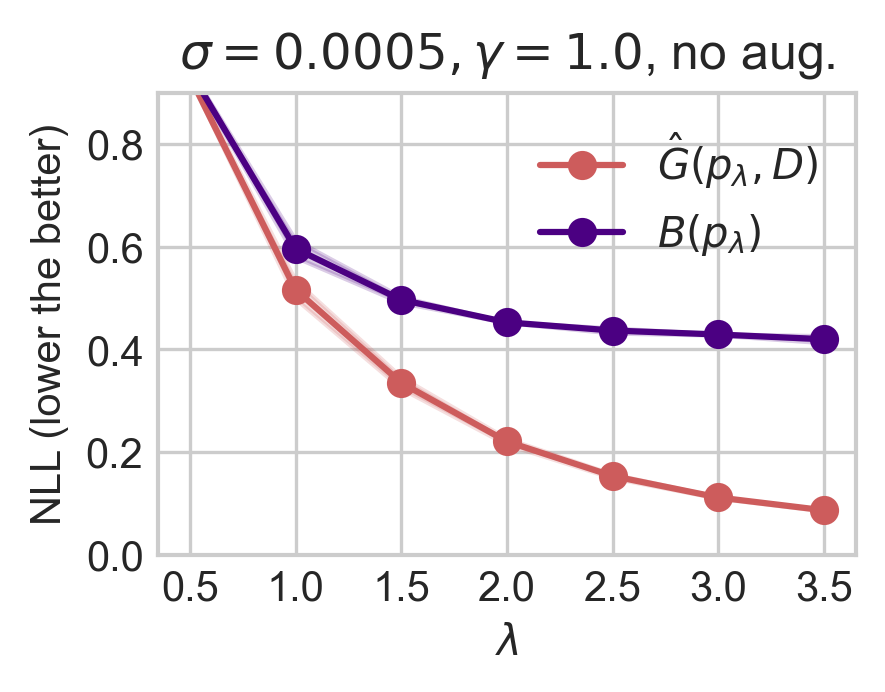}
      \captionsetup{format=hang, justification=centering}
      \caption{Narrow prior and standard softmax}
      \label{fig:CPE:NarrowPriorsmall_cifar10}
  \end{subfigure}
  \hfill
  \begin{subfigure}{.32\linewidth}
      \includegraphics[width=\linewidth]{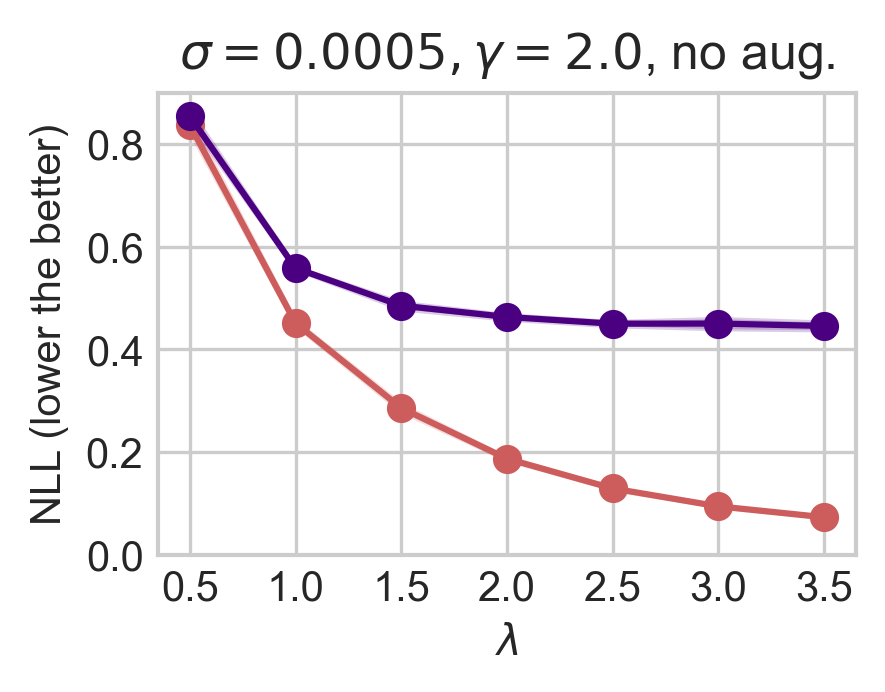}
      \captionsetup{format=hang, justification=centering}
      \caption{Narrow prior and tempered softmax}
      \label{fig:CPE:Softmaxsmall_cifar10}
  \end{subfigure}
  \hfill
  \begin{subfigure}{.32\linewidth}
      \includegraphics[width=\linewidth]{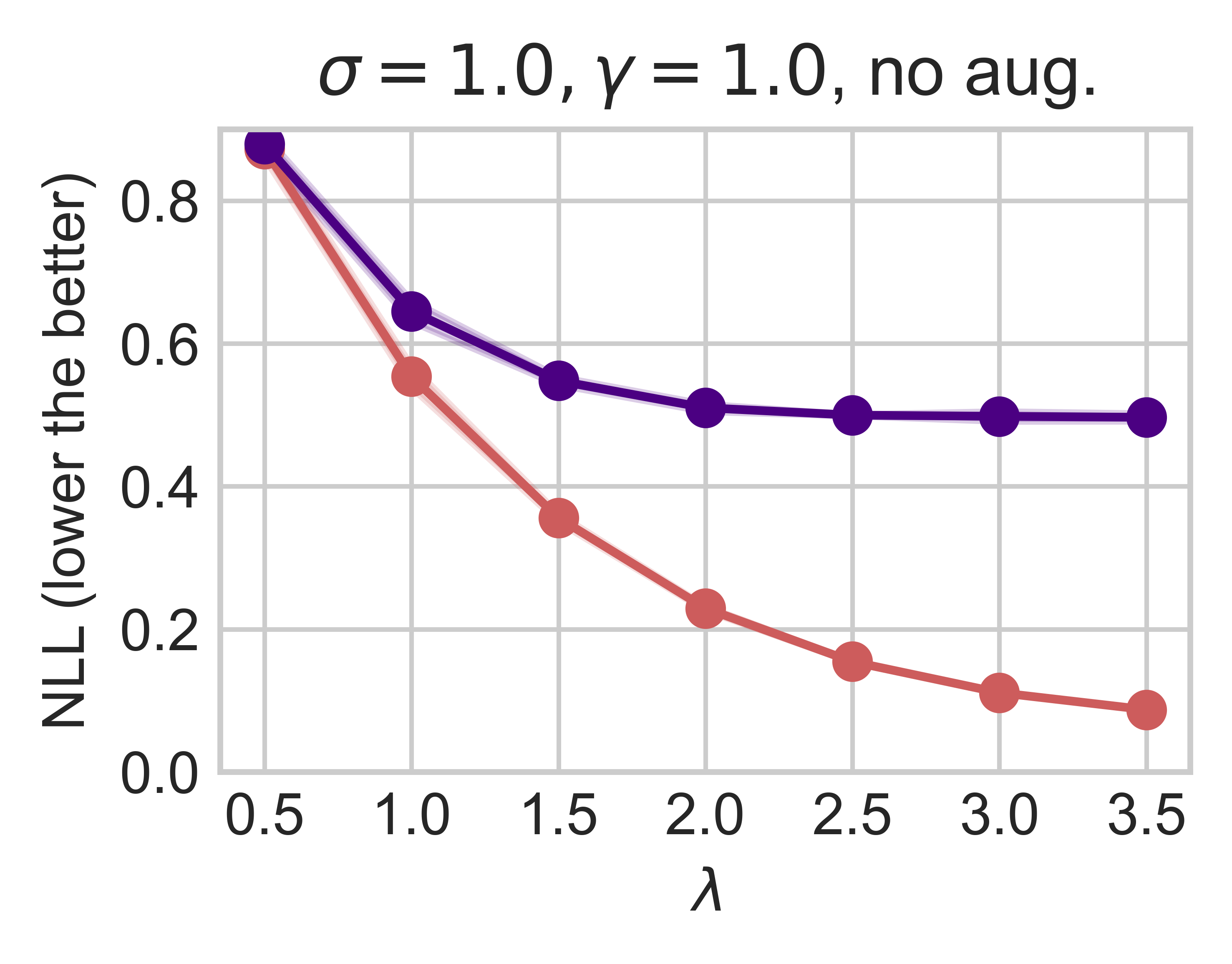}
      \captionsetup{format=hang, justification=centering}
      \caption{Standard prior and standard softmax}
      \label{fig:CPE:StandardPriorsmall_cifar10}
  \end{subfigure}
  \caption{Extended results of Figure \ref{fig:CPE:small} using ResNet-18 via SGLD on CIFAR-10.}
  \label{fig:CPE:small_cifar10}
\end{figure}

\centerline
\centerline
\centerline

\begin{figure}[H]

  \begin{subfigure}{.32\linewidth}
      \includegraphics[width=\linewidth]{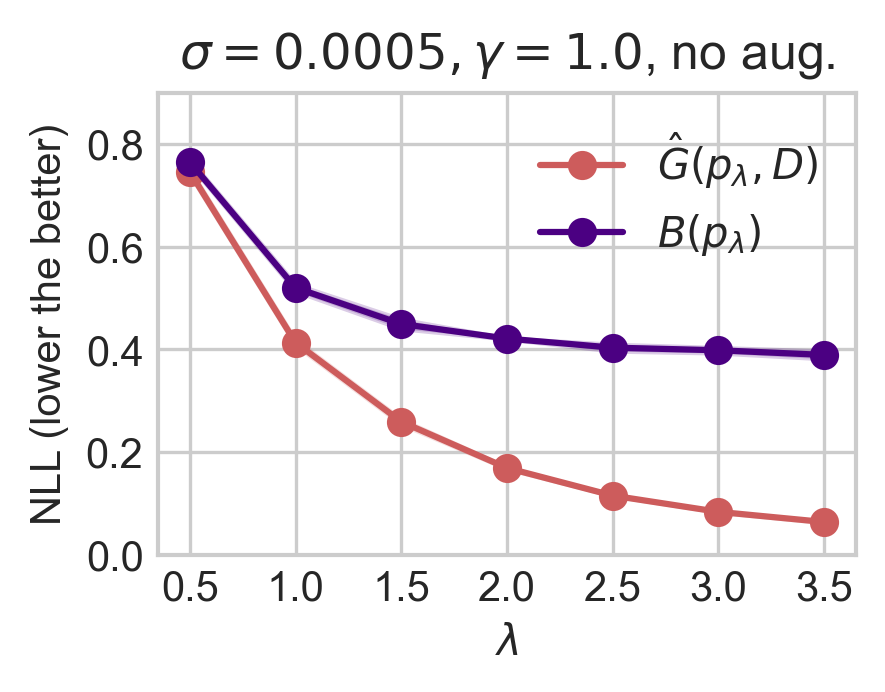}
      \captionsetup{format=hang, justification=centering}
      \caption{Narrow prior and standard softmax}
      \label{fig:CPE:NarrowPriorlarge_cifar10}
  \end{subfigure}
  \begin{subfigure}{.32\linewidth}
      \includegraphics[width=\linewidth]{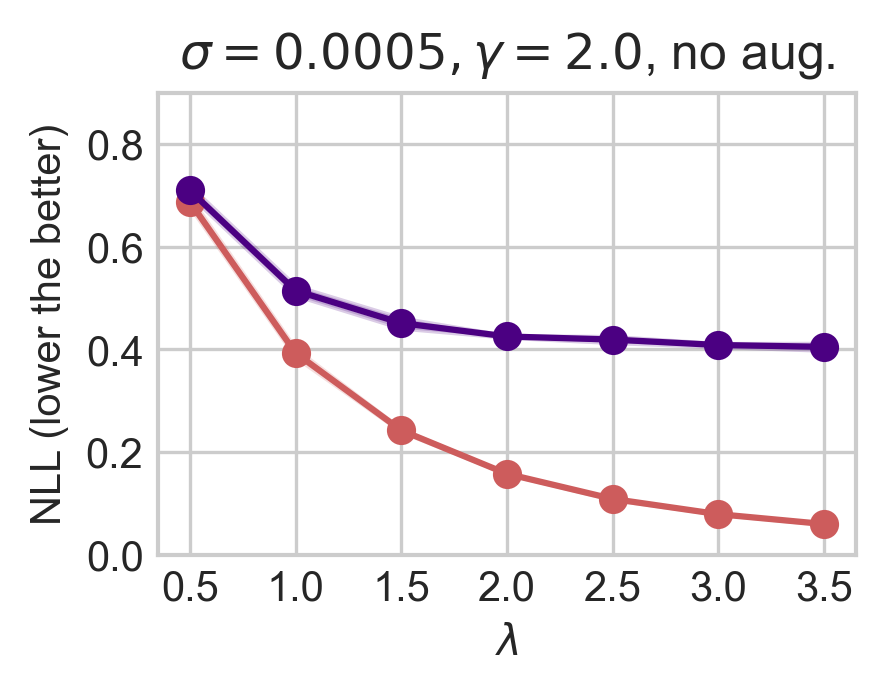}
      \captionsetup{format=hang, justification=centering}
      \caption{Narrow prior and tempered softmax}
      \label{fig:CPE:Softmaxlarge_cifar10}
  \end{subfigure}
  \begin{subfigure}{.32\linewidth}
      \includegraphics[width=\linewidth]{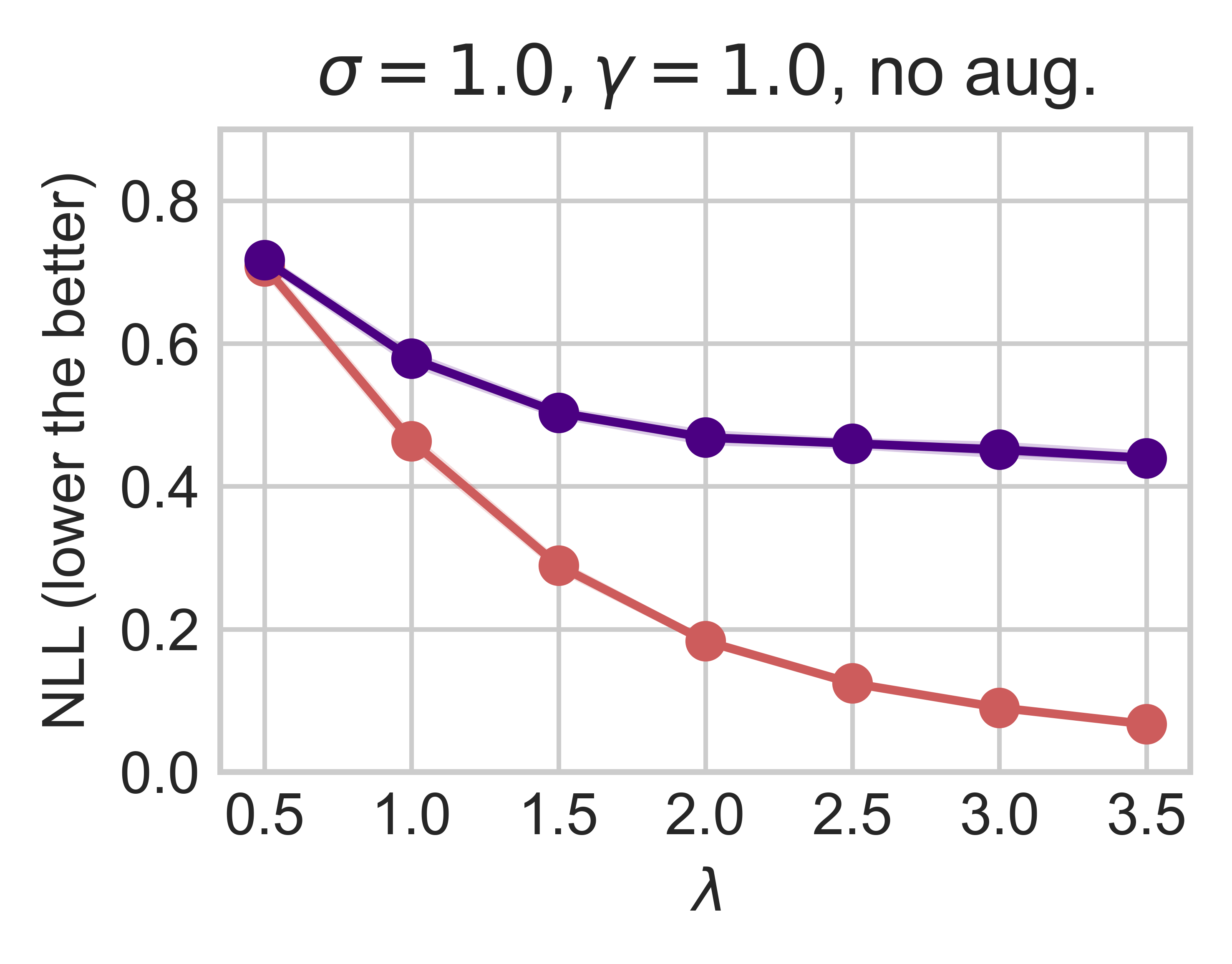}
      \captionsetup{format=hang, justification=centering}
      \caption{Standard prior and standard softmax}
      \label{fig:CPE:StandardPriorlarge_cifar10}
  \end{subfigure}
  \caption{Extended results of Figure \ref{fig:CPE:large} using ResNet-50 via SGLD on CIFAR-10.}
\end{figure}

\newpage

\centerline
\centerline
\centerline

\begin{figure}[H]
  \begin{subfigure}{.32\linewidth}
      \includegraphics[width=\linewidth]{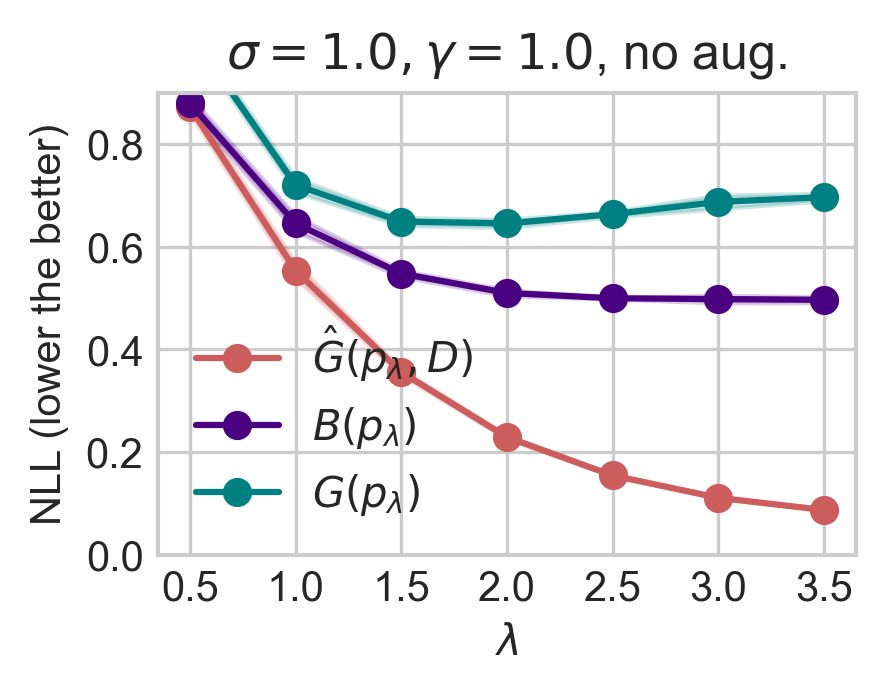}
      \captionsetup{format=hang, justification=centering}
      \caption{Standard prior and standard softmax}
      \label{fig:CPE:cifar10_50}
  \end{subfigure}
  \label{}
  \hfill
  \begin{subfigure}{.32\linewidth}
      \includegraphics[width=\linewidth]{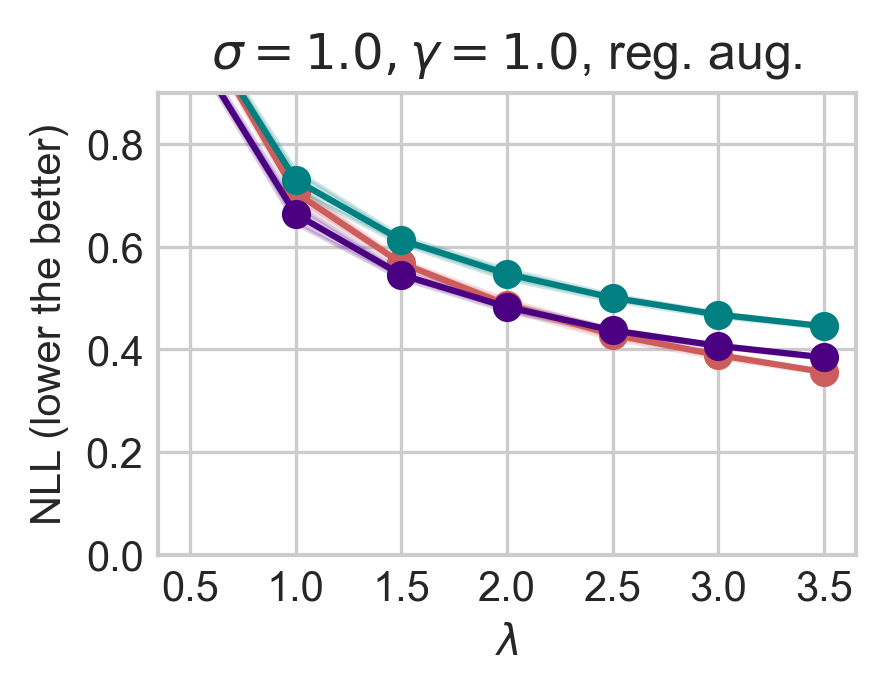}
      \captionsetup{format=hang, justification=centering}
      \caption{Random crop and horizontal flip}
      \label{fig:CPE:cifar10_50_aug}
  \end{subfigure}
  \hfill
  \begin{subfigure}{.32\linewidth}
      \includegraphics[width=\linewidth]{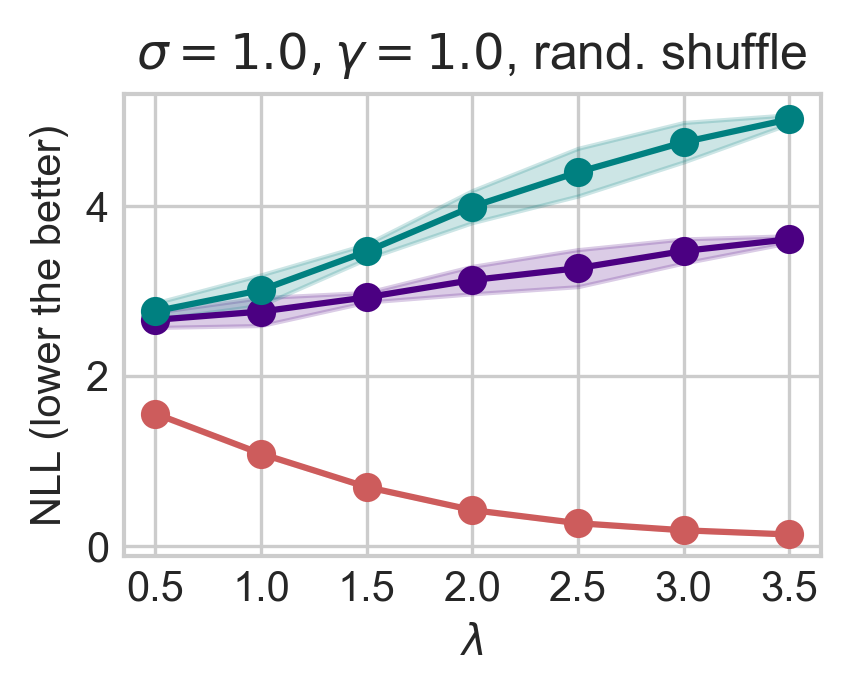}
      \captionsetup{format=hang, justification=centering}
      \caption{Image pixels randomly shuffled}
      \label{fig:CPE:cifar10_50_perm}
  \end{subfigure}
  \caption{Extended results of Figure \ref{fig:CPE:smallAug} using ResNet-18 via SGLD on CIFAR-10.}
  \label{fig:CPE:small_cifar10_aug}
\end{figure}

\centerline
\centerline
\centerline

\begin{figure}[H]
  \begin{subfigure}{.32\linewidth}
      \includegraphics[width=\linewidth]{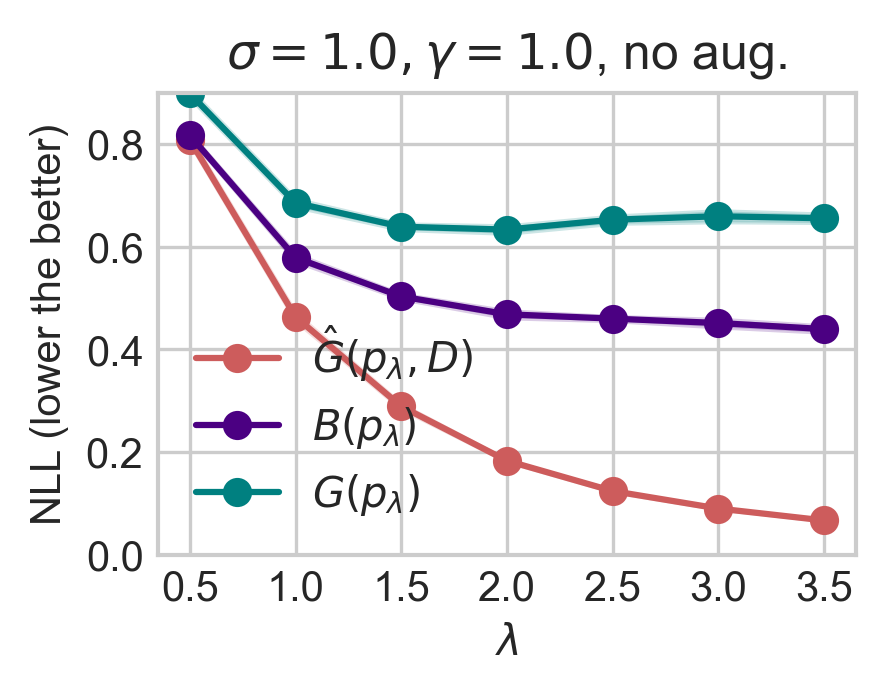}
      \captionsetup{format=hang, justification=centering}
      \caption{Standard prior and standard softmax}
  \end{subfigure}
  \begin{subfigure}{.32\linewidth}
      \includegraphics[width=\linewidth]{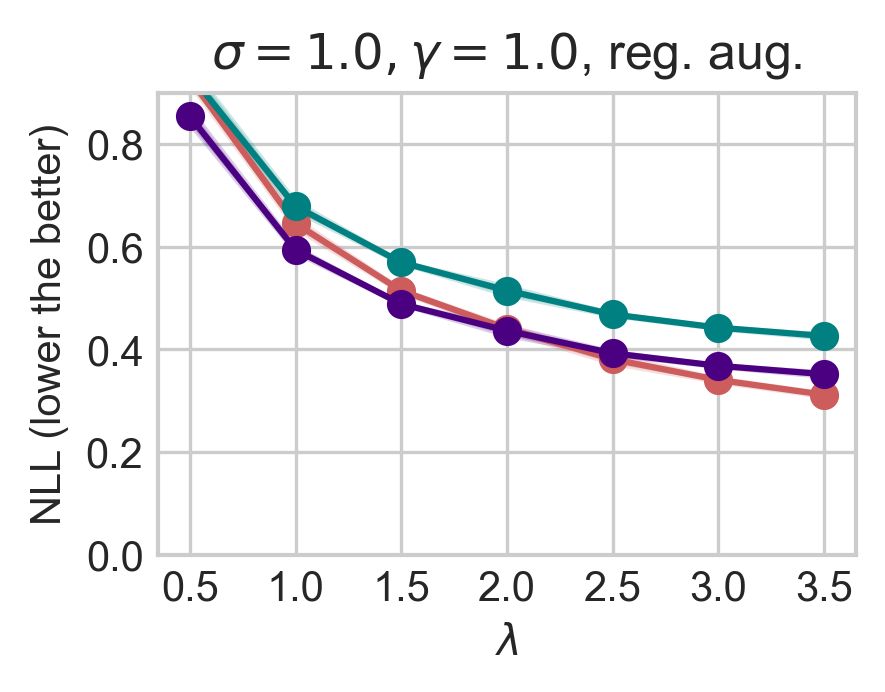}
      \captionsetup{format=hang, justification=centering}
      \caption{Random crop and horizontal flip}
  \end{subfigure}
  \begin{subfigure}{.32\linewidth}
      \includegraphics[width=\linewidth]{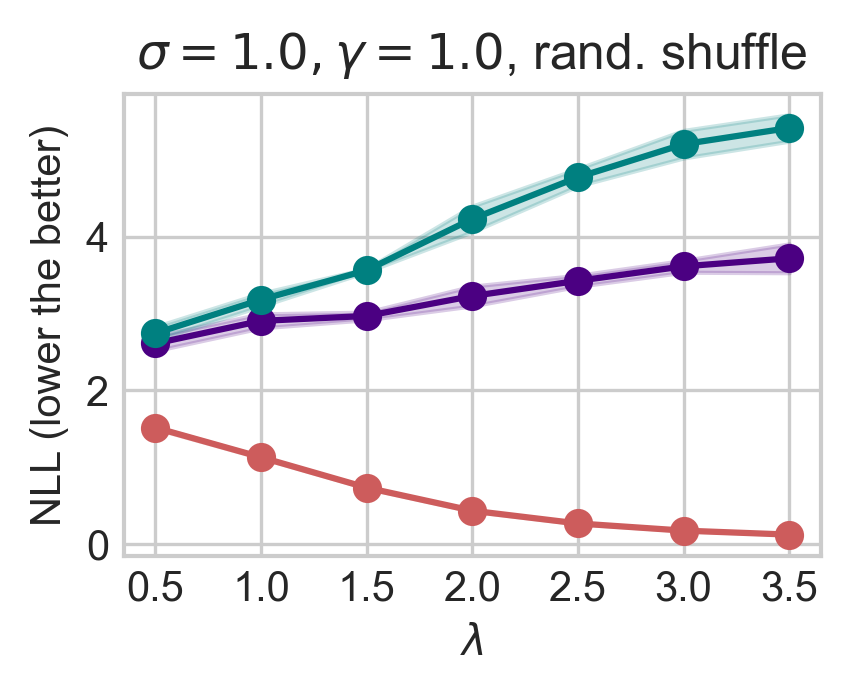}
      \captionsetup{format=hang, justification=centering}
      \caption{Image pixels randomly shuffled}
  \end{subfigure}  
  \caption{Extended results of Figure \ref{fig:CPE:largeAug} using ResNet-50 via SGLD on CIFAR-10.}
\end{figure}

\newpage

\paragraph{ResNet-18 and ResNet-50 via SGLD on CIFAR-100}\label{app:sec:sgld-cifar100}

\centerline
\centerline
\centerline

\begin{figure}[H]
  \begin{subfigure}{.32\linewidth}
      \includegraphics[width=\linewidth]{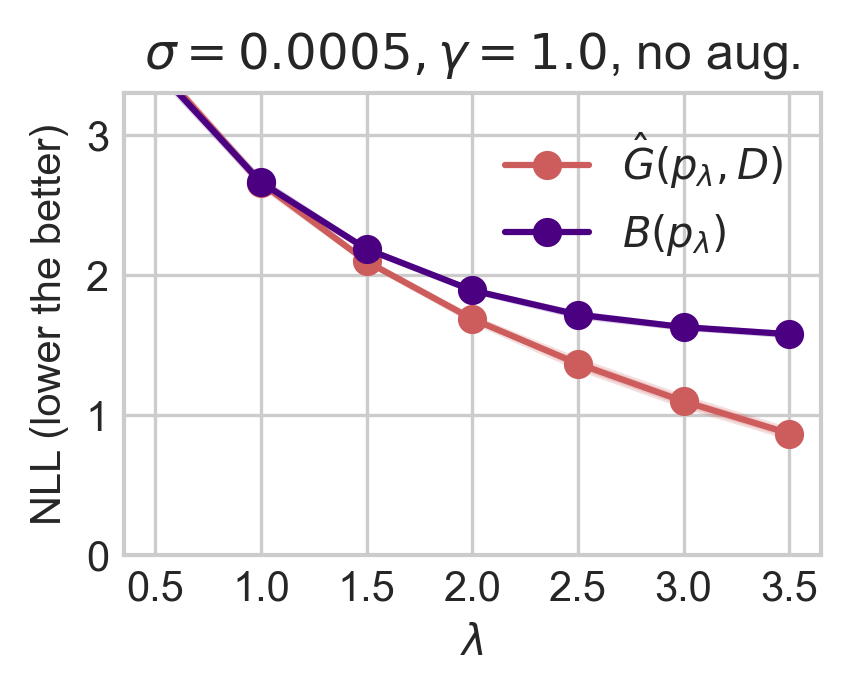}
      \captionsetup{format=hang, justification=centering}
      \caption{Narrow prior and standard softmax}
      \label{fig:CPE:NarrowPriorsmall_cifar100}
  \end{subfigure}
  \hfill
  \begin{subfigure}{.32\linewidth}
      \includegraphics[width=\linewidth]{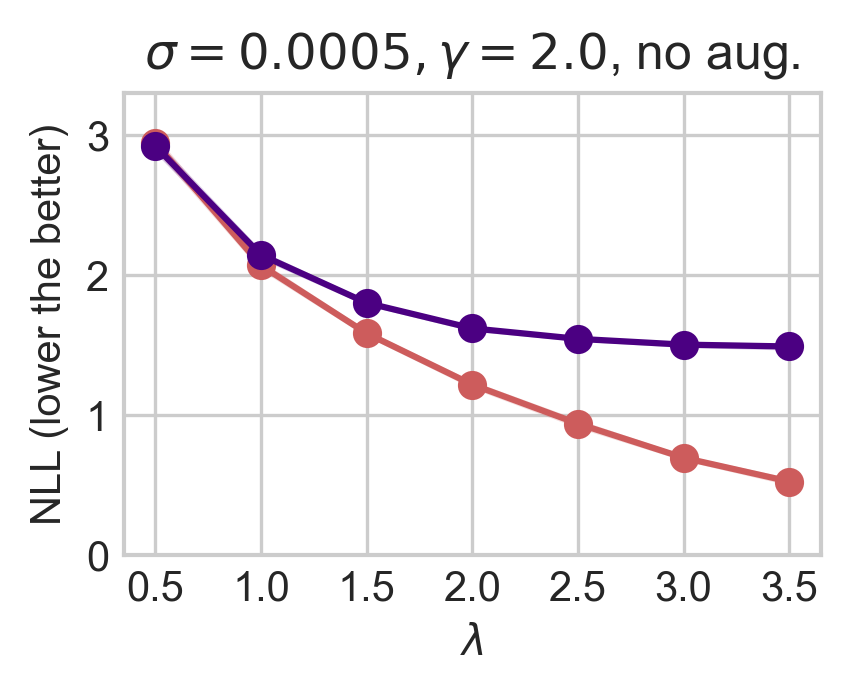}
      \captionsetup{format=hang, justification=centering}
      \caption{Narrow prior and tempered softmax}
      \label{fig:CPE:Softmaxsmall_cifar100}
  \end{subfigure}
  \hfill
  \begin{subfigure}{.32\linewidth}
      \includegraphics[width=\linewidth]{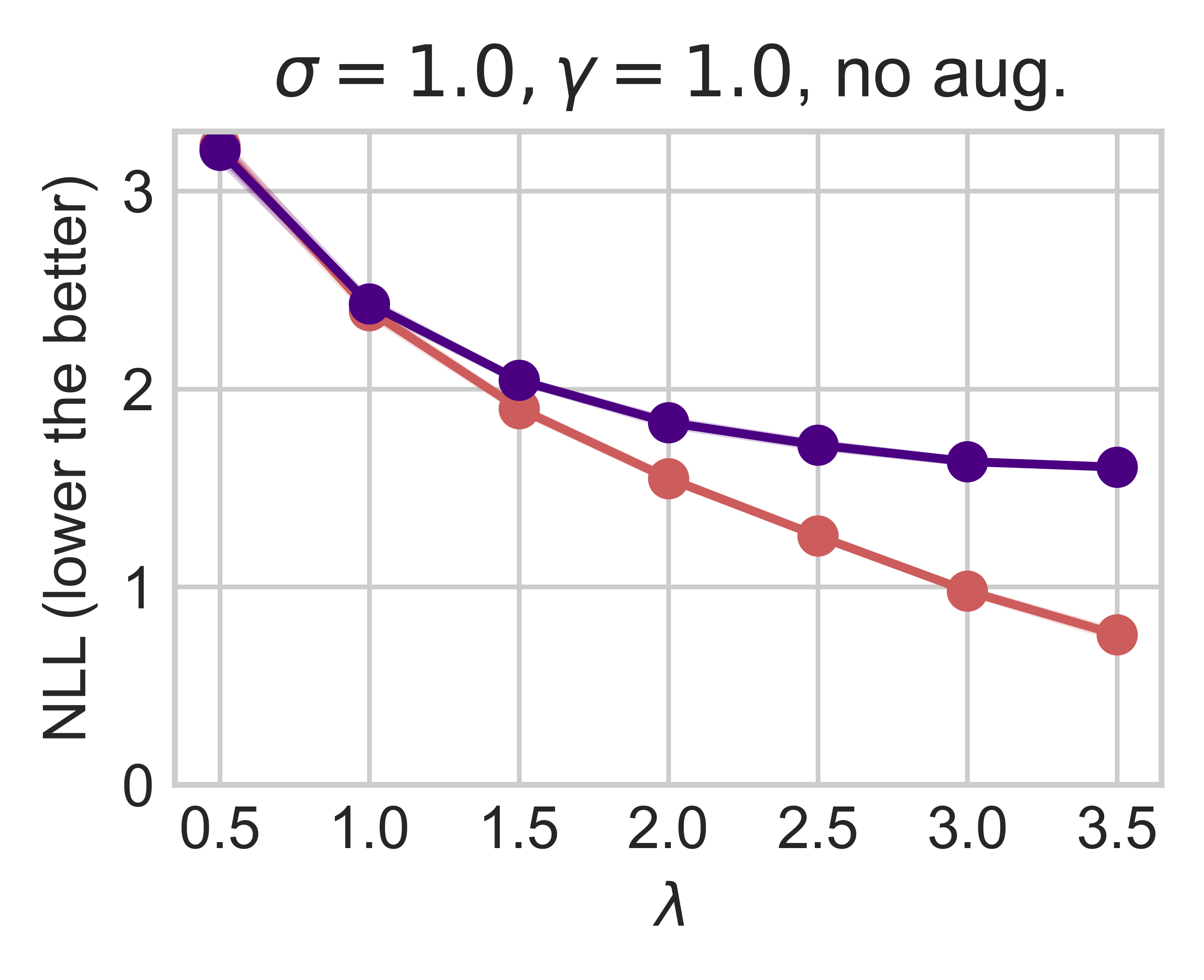}
      \captionsetup{format=hang, justification=centering}
      \caption{Standard prior and standard softmax}
      \label{fig:CPE:StandardPriorsmall_cifar100}
  \end{subfigure}
  \caption{Extended results of Figure \ref{fig:CPE:small} using ResNet-18 via SGLD on CIFAR-100.}
  \label{fig:CPE:small_cifar100}
\end{figure}

\centerline
\centerline
\centerline

\begin{figure}[H]

  \begin{subfigure}{.32\linewidth}
      \includegraphics[width=\linewidth]{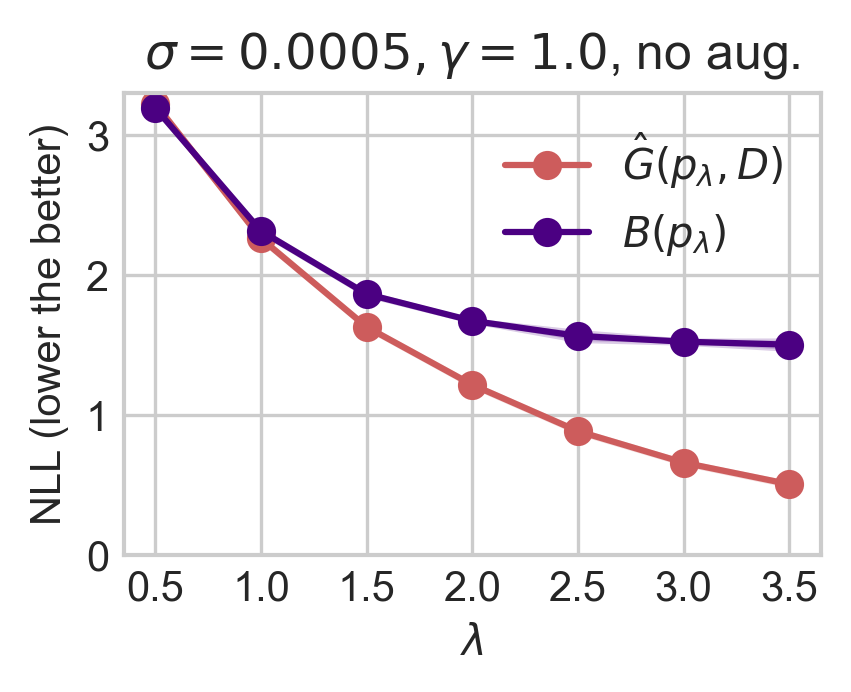}
      \captionsetup{format=hang, justification=centering}
      \caption{Narrow prior and standard softmax}
      \label{fig:CPE:NarrowPriorlarge_cifar100}
  \end{subfigure}
  \begin{subfigure}{.32\linewidth}
      \includegraphics[width=\linewidth]{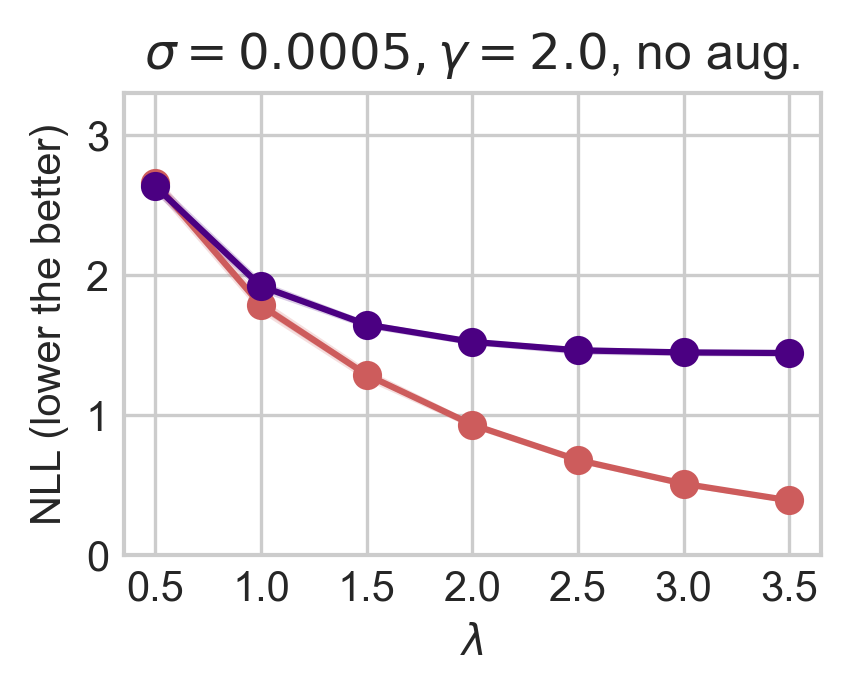}
      \captionsetup{format=hang, justification=centering}
      \caption{Narrow prior and tempered softmax}
      \label{fig:CPE:Softmaxlarge_cifar100}
  \end{subfigure}
  \begin{subfigure}{.32\linewidth}
      \includegraphics[width=\linewidth]{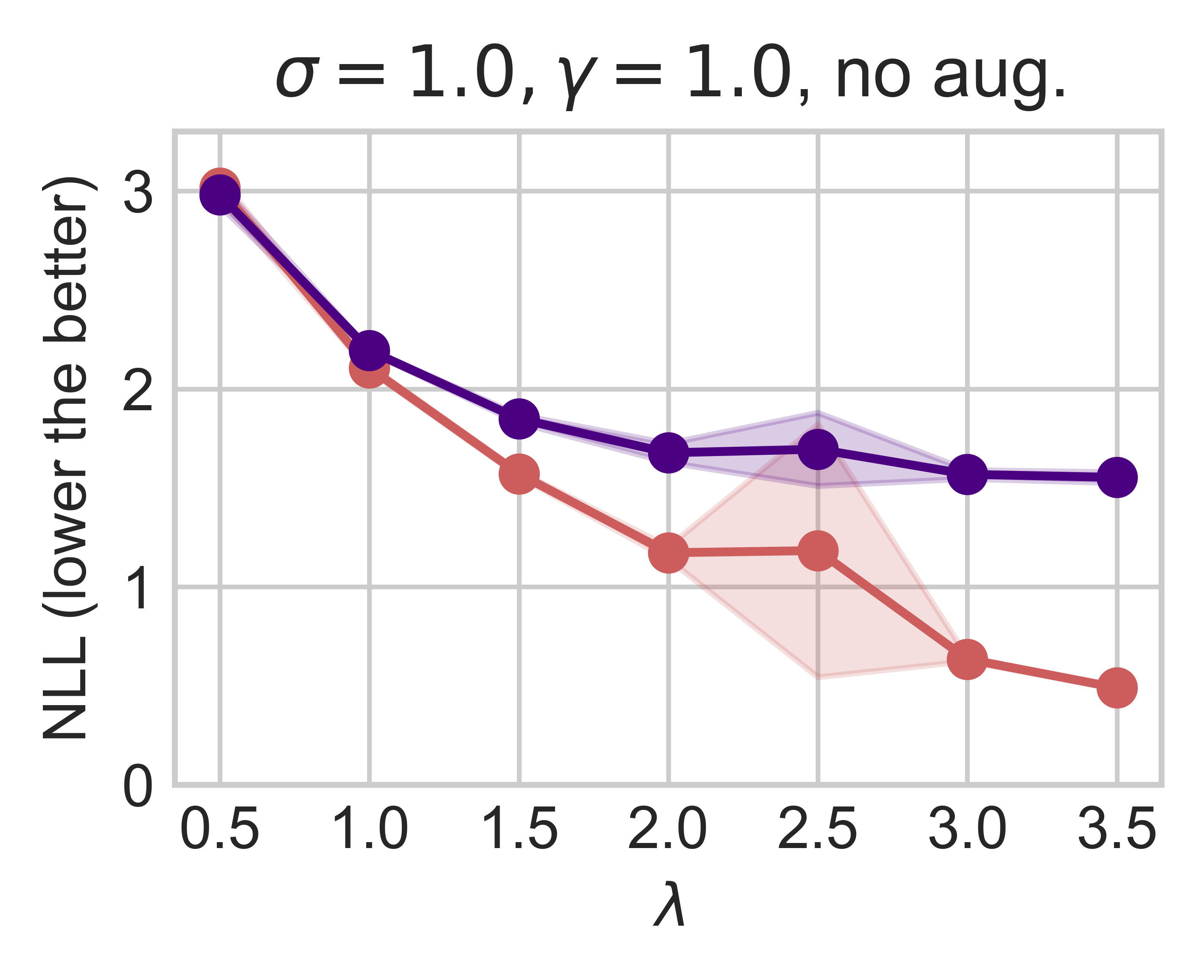}
      \captionsetup{format=hang, justification=centering}
      \caption{Standard prior and standard softmax}
      \label{fig:CPE:StandardPriorlarge_cifar100}
  \end{subfigure}
  \caption{Extended results of Figure \ref{fig:CPE:large} using ResNet-50 via SGLD on CIFAR-100.}
\end{figure}

\newpage

\centerline
\centerline
\centerline

\begin{figure}[H]
  \begin{subfigure}{.32\linewidth}
      \includegraphics[width=\linewidth]{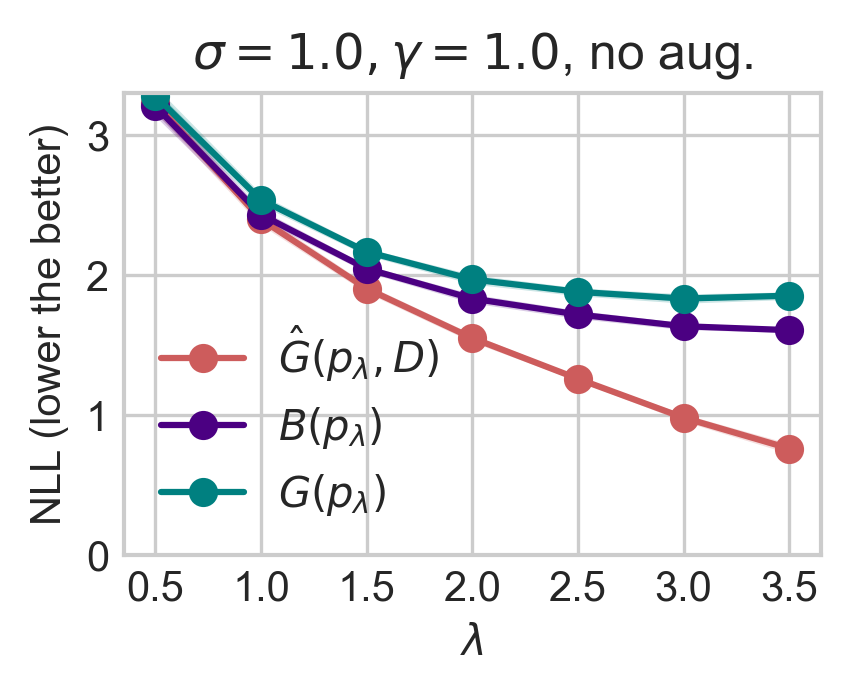}
      \captionsetup{format=hang, justification=centering}
      \caption{Standard prior and standard softmax}
      \label{fig:CPE:cifar100_50}
  \end{subfigure}
  \label{}
  \hfill
  \begin{subfigure}{.32\linewidth}
      \includegraphics[width=\linewidth]{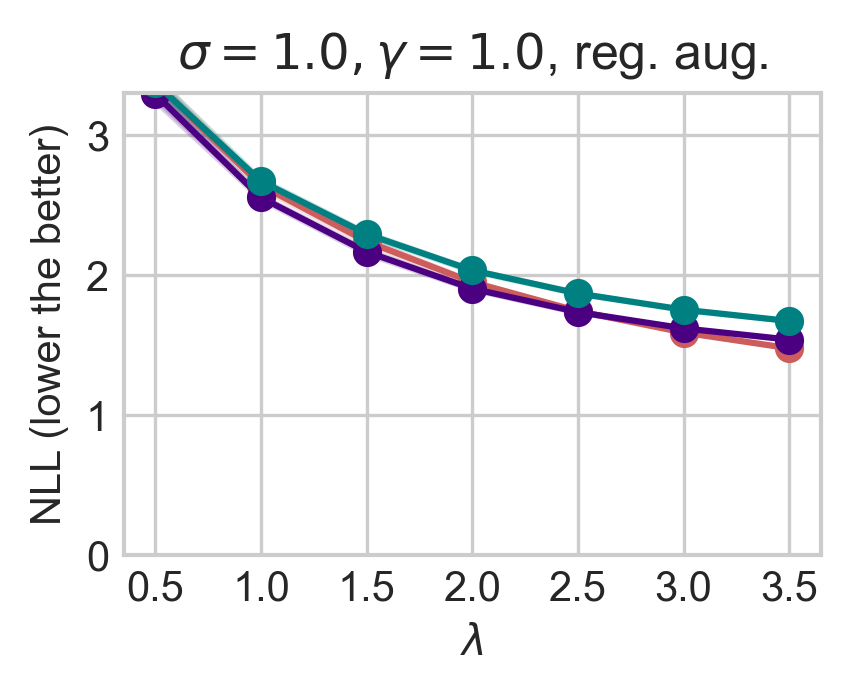}
      \captionsetup{format=hang, justification=centering}
      \caption{Random crop and horizontal flip}
      \label{fig:CPE:cifar100_50_aug}
  \end{subfigure}
  \hfill
  \begin{subfigure}{.32\linewidth}
      \includegraphics[width=\linewidth]{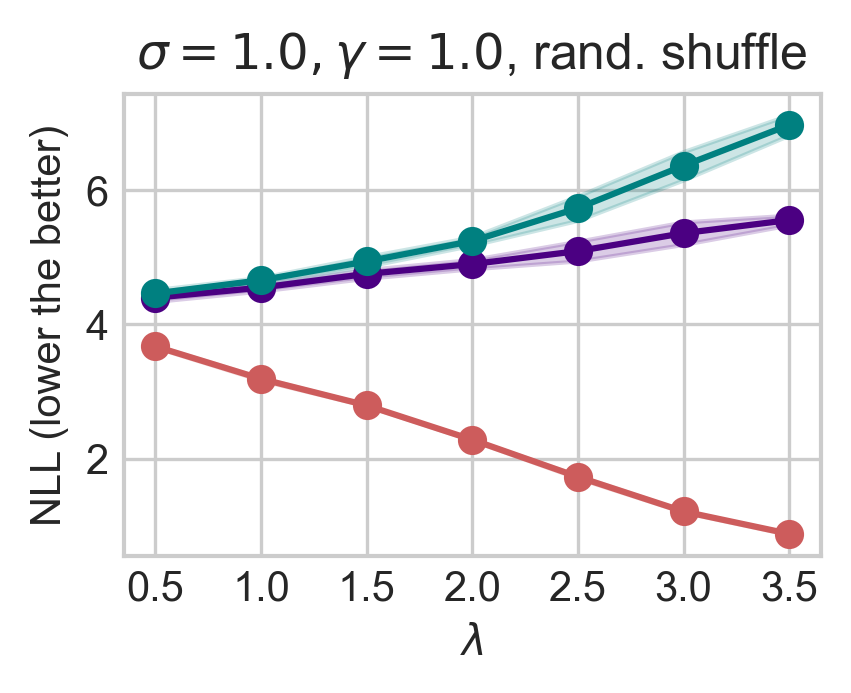}
      \captionsetup{format=hang, justification=centering}
      \caption{Image pixels randomly shuffled}
      \label{fig:CPE:cifar100_50_perm}
  \end{subfigure}
  \caption{Extended results of Figure \ref{fig:CPE:smallAug} using ResNet-18 via SGLD on CIFAR-100.}
  \label{fig:CPE:small_cifar100_aug}
\end{figure}

\centerline
\centerline
\centerline

\begin{figure}[H]
  \begin{subfigure}{.32\linewidth}
      \includegraphics[width=\linewidth]{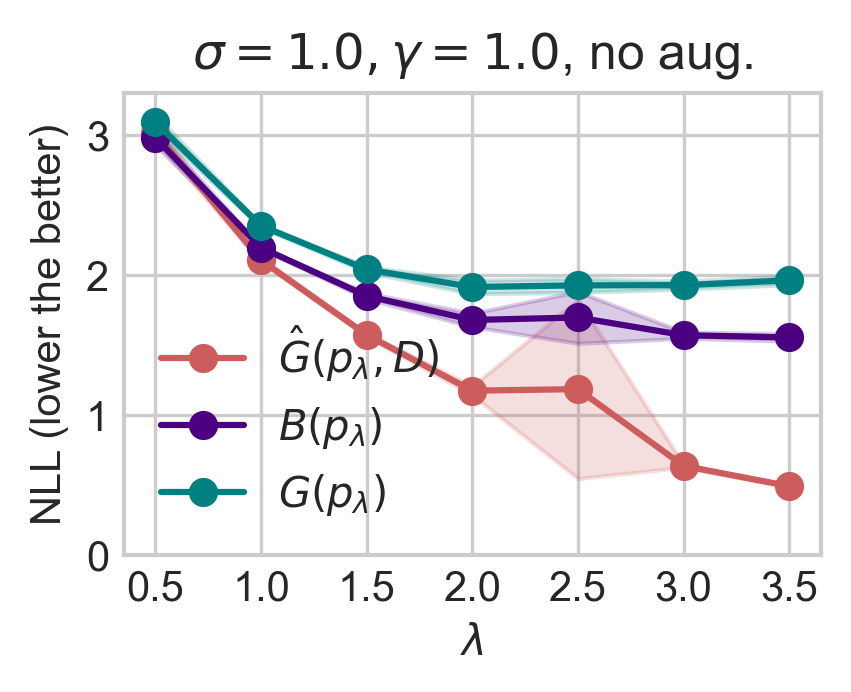}
      \captionsetup{format=hang, justification=centering}
      \caption{Standard prior and standard softmax}
  \end{subfigure}
  \begin{subfigure}{.32\linewidth}
      \includegraphics[width=\linewidth]{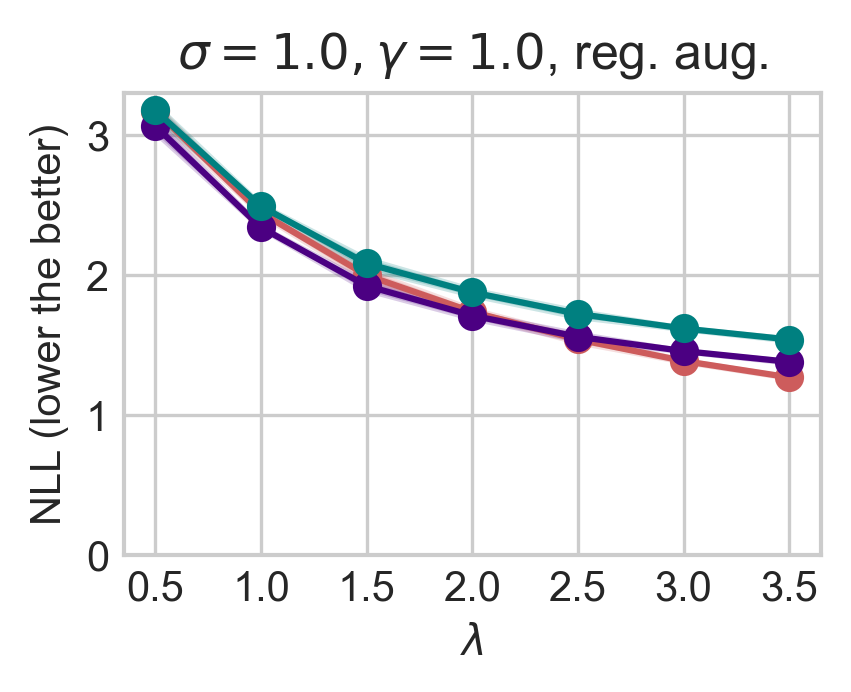}
      \captionsetup{format=hang, justification=centering}
      \caption{Random crop and horizontal flip}
  \end{subfigure}
  \begin{subfigure}{.32\linewidth}
      \includegraphics[width=\linewidth]{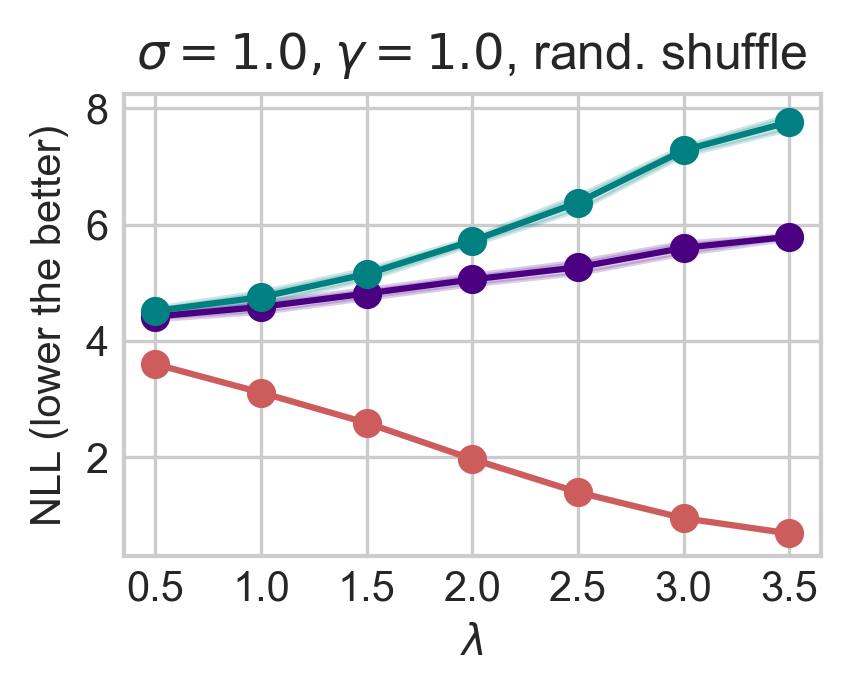}
      \captionsetup{format=hang, justification=centering}
      \caption{Image pixels randomly shuffled}
  \end{subfigure}  
  \caption{Extended results of Figure \ref{fig:CPE:largeAug} using ResNet-50 via SGLD on CIFAR-100.}
\end{figure}

\newpage
\chapter{Recursive PAC-Bayes: A Frequentist Approach to Sequential Prior Updates}
\label{chap:chap3}
The work presented in this chapter is based on a paper that is currently under review in \textit{NeurIPS} 2024: 
\\
\\
\bibentry{wu2024recursivepacbayesfrequentistapproach}. 

\newpage

\section*{Abstract}
    PAC-Bayesian analysis is a frequentist framework for incorporating prior knowledge into learning. It was inspired by Bayesian learning, which allows sequential data processing and naturally turns posteriors from one processing step into priors for the next. However, despite two and a half decades of research, the ability to update priors sequentially without losing confidence information along the way remained elusive for PAC-Bayes. While PAC-Bayes allows construction of data-informed priors, the final confidence intervals depend only on the number of points that were not used for the construction of the prior, whereas confidence information in the prior, which is related to the number of points used to construct the prior, is lost. This limits the possibility and benefit of sequential prior updates, because the final bounds depend only on the size of the final batch.

    We present a novel and, in retrospect, surprisingly simple and powerful PAC-Bayesian procedure that allows  sequential prior updates with no information loss. The procedure is based on a novel decomposition of the expected loss of randomized classifiers. The decomposition rewrites the loss of the posterior as an excess loss relative to a downscaled loss of the prior plus the downscaled loss of the prior, which is bounded recursively.
    As a side result, we also present a generalization of the split-kl and PAC-Bayes-split-kl inequalities to discrete random variables, which we use for bounding the excess losses, and which can be of independent interest. In empirical evaluation the new procedure significantly outperforms state-of-the-art.

\section{Introduction}
PAC-Bayesian analysis was born from an attempt to derive frequentist generalization guarantees for Bayesian-style prediction rules \citep{STW97,McA98}. The motivation was to provide a way to incorporate prior knowledge into the frequentist analysis of generalization. PAC-Bayesian bounds provide high-probability generalization guarantees for randomized classifiers. A randomized classifier is defined by a distribution $\rho$ on a set of prediction rules $\HH$, which is used to sample a prediction rule each time a prediction is to be made. Bayesian posterior is an example of a randomized classifier, whereas PAC-Bayesian bounds hold generally for all randomized classifiers. Prior knowledge is encoded through a prior distribution $\pi$ on $\HH$, and the complexity of a posterior distribution $\rho$ is measured by the Kullback-Leibler (KL) divergence from the prior, $\KL(\rho\|\pi)$. PAC-Bayesian generalization guarantees are optimized by posterior distributions $\rho$ that optimize a trade-off between empirical data fit and divergence from the prior in the KL sense.

Selection of a ``good'' prior plays an important role in the PAC-Bayesian bounds. If one manages to foresee which prediction rules are likely to produce low prediction error and allocate a higher prior mass for them, then the bounds are tighter, because the posterior only needs to make a small deviation from the prior. But if the prior mass on well-performing prediction rules is small, the bounds are loose. A major technique to design good priors is to use part of the data to estimate a good prior and the rest of the data to compute a PAC-Bayes bound. It is known as data-dependent or data-informed priors \citep{APS07}. However, all existing approaches to data-informed priors have three major disadvantages. The first is that the bounds are computed on ``the rest of the data'' that were not used in construction of the prior. Thus, the sample size in the bounds is only a fraction of the total sample size. Therefore, empirically data-informed priors are not always helpful. In many cases starting with an uninformed prior and using all the data to compute the posterior and the bound turns to be superior to sacrificing part of the data for prior construction \citep{APS07,MGG20}. The second disadvantage is that all the confidence information about the prior is lost in the process. In particular, a prior trained on a few data points is treated in the same way as a prior trained on a lot of data. And a third related disadvantage is that sequential data processing is detrimental, because the bounds only depend on the size of the last chunk and all the confidence information from processing earlier chunks is lost in the process.

Our main contribution is a new (and simple) way of decomposing the loss of a randomized classifier defined by the posterior. We write it as an excess loss relative to a downscaled loss of the randomized classifier defined by the prior plus the downscaled loss of the randomized classifier defined by the prior. The excess loss can be bounded using PAC-Bayes-Empirical-Bernstein-style inequalities \citep{TS13,MGG20,WMLIS21,WS22}, whereas the loss of the randomized classifier defined by the prior can be bounded recursively. The recursive bound can both use the data used for construction of the prior and ``the rest of the data'', and thereby preserves confidence information on the prior. Our contribution stands out relative to all prior work on PAC-Bayes, and in fact all prior work on frequentist generalization bounds, because it makes sequential data processing and sequential prior updates meaningful and beneficial. 

We note that while several recent papers experimented with sequential posterior updates by using martingale-style analysis, in all these works the prior remained fixed and only the posterior was changing \citep{CWR23,BG23,RTS24}. Another line of work used tools from online learning to derive PAC-Bayesian bounds \citep{JJKO23}, and in this context \citet{HG23} have used sequential prior updates, but their bounds hold for a uniform aggregation of sequentially constructed posteriors, which is different from standard posteriors studied in our work. The confidence bounds in their work come primarily from aggregation rather than confidence in individual posteriors in the sequence. Our work is the first one allowing sequential prior updates without loss of confidence information.

An additional side contribution of independent interest is a generalization of the split-kl and PAC-Bayes-split-kl inequalities of \citet{WS22} from ternary to general discrete random variables. It is based on a novel representation of discrete random variables as a superposition of Bernoulli random variables.

The paper is organized in the following way. In Section~\ref{sec:idea} we briefly survey the evolution of data-informed priors in PAC-Bayes and present our main idea behind Recursive PAC-Bayes; in Section~\ref{sec:split-kl} we present our generalization of the split-kl and PAC-Bayes-split-kl inequalities, which are later used to bound the excess losses; in Section~\ref{sec:main} we present the Recursive PAC-Bayes bound; in Section~\ref{sec:experiments} we present an empirical evaluation; and in Section~\ref{sec:discussion} we conclude with a discussion.

\section{The evolution of data-informed priors and the idea of recursive PAC-Bayes}\label{sec:idea}

In this section we briefly survey the evolution of data-informed priors, and then present our construction of Recursive PAC-Bayes. We consider the standard binary classification setting, with $\XX$ being a sample space, $\YY$ a label space, $\HH$ a set of prediction rules $h:\XX\to\YY$, and $\ell(h(X),Y) = \1[h(X)\neq Y]$ the zero-one loss function, where $\1[\cdot]$ denotes the indicator function. We let $\DD$ denote a distribution on $\XX\times\YY$ and $S = \lrc{(X_1,Y_1),\dots,(X_n,Y_n)}$ an i.i.d.\ sample from $\DD$. Let $L(h) = \E_{(X,Y)\sim \DD}[\ell(h(X),Y)]$ be the expected and $\hat L(h,S) = \frac{1}{n}\sum_{i=1}^n\ell(h(X_i),Y_i)$ the empirical loss. 

Let $\rho$ be a distribution on $\HH$. A \emph{randomized classifier} associated with $\rho$ samples a prediction rule $h$ according to $\rho$ for each sample $X\in\XX$, and applies it to make a prediction $h(X)$. The expected loss of such randomized classifier, which we call $\rho$, is $\E_{h\sim\rho}[L(h)]$ and the empirical loss is $\E_{h\sim\rho}[\hat L(h,S)]$. For brevity we use $\E_\rho[\cdot]$ to denote $\E_{h\sim\rho}[\cdot]$. 

We use $\KL(\rho\|\pi)$ to denote the Kullback-Leibler divergence between two probability distributions, $\rho$ and $\pi$ \citep{CT06}. For $p,q \in [0,1]$ we further use $\kl(p\|q) = \KL((1-p,p)\|(1-q,q))$ to denote the Kullback-Leibler divergence between two Bernoulli distributions with biases $p$ and $q$.

The goal of PAC-Bayes is to bound $\E_{\rho}[L(h)]$. Below we present how the bounds on $\E_{\rho}[L(h)]$ have evolved. In Figure~\ref{fig:evolution} we also provide a graphical illustration of the evolution.

\subsection{Uninformed priors} 
Early work on PAC-Bayes used \emph{uninformed priors} \citep{McA98}. An uniformed prior $\pi$ is a distribution on $\HH$ that is independent of the data $S$. A classical, and still one of the tightest bounds, is the following. 
\begin{theorem}[PAC-Bayes-$\kl$ Inequality, \citealp{See02}, \citealp{Mau04}] \label{thm:pbkl}
For any probability distribution $\pi$ on $\HH$ that is independent of $S$ and any $\delta \in (0,1)$:
\[
\P[\exists \rho\in{\mathcal{P}}: \kl\lr{\E_\rho[\hat L(h,S)]\middle\|\E_\rho\lrs{L(h)}} \geq \frac{\KL(\rho\|\pi) + \ln(2 \sqrt n/\delta)}{n}]\leq \delta,
\]
where $\mathcal{P}$ is the set of all probability distributions on $\HH$, including those dependent on $S$.
\end{theorem}
A posterior $\rho$ that minimizes $\E_\rho[L(h)]$ has to balance between allocating higher mass to prediction rules $h$ with small $\hat L(h,S)$ and staying close to $\pi$ in the $\KL(\rho\|\pi)$ sense. Since $\pi$ has to be independent of $S$, typical uninformed priors aim ``to leave maximal options open'' for $\rho$ by staying close to uniform.

\subsection{Data-informed priors} \citet{APS07} proposed to split the data $S$ into two disjoint sets, $S = S_1\cup S_2$, and use $S_1$ to construct a \emph{data-informed prior} $\pi$ and compute a bound on $\E_\rho[L(h)]$ using $\pi$ and $S_2$. Since in this approach $\pi$ is independent of $S_2$, Theorem~\ref{thm:pbkl} can be applied. The advantage is that $\pi$ can use $S_1$ to give higher mass to promising classifiers, thus relaxing the regularization pressure $\KL(\rho\|\pi)$ and making it easier for $\rho$ to allocate even higher mass to well-performing classifiers (those with small $\hat L(h,S_2)$). The disadvantage is that the sample size in the bound (the $n$ in the denominator) decreases from the size of $S$ to the size of $S_2$. Indeed, \citeauthor{APS07} observed that the sacrifice of $S_1$ for prior construction does not always pay off.

\subsection{Data-informed priors + excess loss} \citet{MGG20} observed that if we have already sacrificed $S_1$ for the construction of $\pi$, we could also use it to construct a reference prediction rule $h^*$, typically an Empirical Risk Minimizer (ERM) on $S_1$. They then employed the decomposition
\[
    \E_\rho[L(h)] = \E_\rho[L(h) - L(h^*)] + L(h^*)
\]
and used $S_2$ to give a PAC-Bayesian bound on $\E_\rho[L(h) - L(h^*)]$ and a single-hypothesis bound on $L(h^*)$. The quantity $\E_\rho[L(h) - L(h^*)]$ is known as \emph{excess loss}. The advantage of this approach is that when $L(h^*)$ is a good approximation of $\E_\rho[L(h)]$, the excess loss has lower variance than the plain loss $\E_\rho[L(h)]$ and, therefore, is more efficient to bound, whereas the single-hypothesis bound on $L(h^*)$ does not involve the $\KL(\rho\|\pi)$ term. Therefore, it is generally beneficial to use excess losses in combination with data-informed priors. However, as with the previous approach, sacrificing $S_1$ to learn $\pi$ and $h^*$ means that the denominator in the bounds ($n$ in Theorem~\ref{thm:pbkl}) reduces to the size of $S_2$, and it does not always pay off. (We note that the excess loss is not binary and not in the $[0,1]$ interval, and in order to exploit small variance it is actually necessary to apply a PAC-Bayes-Empirical-Bernstein-style inequality \citep{TS13,MGG20,WMLIS21} or the PAC-Bayes-split-kl inequality \citep{WS22} rather than Theorem~\ref{thm:pbkl}, but the point about reduced sample size still applies.)

\subsection{Recursive PAC-Bayes (new)} 
We introduce the following decomposition of the loss
\begin{equation}
\label{eq:RPB-base}
\E_\rho[L(h)] = \E_\rho[L(h) - \gamma \E_\pi[L(h')]] + \gamma \E_\pi[L(h')].
\end{equation}
As before, we decompose $S$ into two disjoint sets $S= S_1\cup S_2$. We make the following major observations:
\begin{itemize}
    \item The quantity $\E_\pi[L(h')]$ on the right is ``of the same kind'' as $\E_\rho[L(h)]$ on the left.
    \item We can take an uninformed prior $\pi_0$ and apply Theorem~\ref{thm:pbkl} (or any other suitable PAC-Bayes bound) to bound $\E_\pi[L(h')]$. (The $\KL$ term in the bound on $\E_\pi[L(h')]$ will be $\KL(\pi\|\pi_0)$.)
    \item We can restrict $\pi$ to depend only on $S_1$, but still use all the data $S$ in calculation of the PAC-Bayes bound on $\E_\pi[L(h')]$, because $\pi$ is a posterior relative to $\pi_0$, and a posterior is allowed to depend on all the data, and in particular on any subset of the data. Therefore, the empirical loss $\E_\pi[\hat L(h',S)]$ can be computed on all the data $S$, and the denominator of the bound in Theorem~\ref{thm:pbkl} can be the size of $S$, and not the size of $S_2$. This is what we call \emph{preservation of confidence information on $\pi$}, because all the data $S$ are used to construct a confidence bound on $\E_\pi[L(h')]$, and not just $S_2$. This is in contrast to the bound on $L(h^*)$ in the approach of \citet{MGG20}, which only allows to use $S_2$ for bounding $L(h^*)$. Note that while we use all the data $S$ in calculation of the bound, we only use $S_1$ and $\E_\pi[\hat L(h',S_1)]$ in the construction of $\pi$. Nevertheless, we can still use the knowledge that we will have $n$ samples when we reach the estimation phase, i.e., when constructing $\pi$ we can leave the denominator of the bound at $n$, allowing more aggressive deviation from $\pi_0$.
    \item If we restrict $\pi$ to depend only on $S_1$, then it is a valid prior for estimation of any posterior quantity $\E_\rho[\cdot]$ based on $S_2$. Thus, if we also restrict $\gamma$ to depend only on $S_1$, we can use any PAC-Bayes-Empirical-Bernstein-style inequality or the PAC-Bayes-split-kl inequality to estimate the excess loss $\E_\rho[L(h) - \gamma \E_\pi[L(h')]]$ based on $S_2$, i.e., based on $\E_\rho[\hat L(h,S_2) - \gamma \E_\pi[\hat L(h',S_2)]]$. If $\gamma \E_\pi[L(h')]$ is a good approximation of $\E_\rho[L(h)]$ and $\E_\rho[L(h)]$ is not close to zero, then the excess loss $\E_\rho[L(h) - \gamma \E_\pi[L(h')]]$ is more efficient to bound than the plain loss $\E_\rho[L(h)]$.
    \item In general, since $\E_\rho[L(h)]$ is expected to improve on $\E_\pi[L(h')]$, it is natural to set $\gamma < 1$. However, $\gamma$ is not allowed to depend on $S_2$, because otherwise $\hat L(h,S_2) - \gamma \E_\pi[\hat L(h',S_2)]$ becomes a biased estimate of $L(h) - \gamma \E_\pi[L(h')]$. We discuss the choice of $\gamma$ in more detail when we present the bound and the experiments.
    \item \citet{BG23} have proposed a sequential martingale-style evaluation of a martingale version of $\E_\rho[L(h)-L(h^*)]$ and $L(h^*)$ in the approach of \citeauthor{MGG20}, but it has not been shown to yield significant improvements yet. The same ``martingalization'' can be directly applied to our decomposition, but to keep things simple we stay with the basic decomposition.
    \item Finally, we note that we can split $S_1$ further and apply Equation~\ref{eq:RPB-base} recursively to bound $\E_\pi[L(h')]$.
\end{itemize}
To set notation for recursive decomposition, we use $\pi_0,\pi_1,\dots,\pi_T$ to denote a sequence of distributions on $\HH$, where $\pi_0$ is an uninformed prior and $\pi_T=\rho$ is the final posterior. We use $\gamma_2,\dots,\gamma_T$ to denote a sequence of coefficients. For $t\geq 2$ we then have the recursive decomposition
\begin{equation}
\label{eq:RPB}
\E_{\pi_t}[L(h)] = \E_{\pi_t}[L(h) - \gamma_t\E_{\pi_{t-1}}[L(h)]] + \gamma_t\E_{\pi_{t-1}}[L(h)].
\end{equation}
To construct $\pi_1,\dots,\pi_T$ we split the data $S$ into $T$ non-overlapping subsets, $S= S_1\cup\cdots\cup S_T$. We restrict $\pi_t$ to depend on $\Utrain_t = \bigcup_{s=1}^t S_s$ only, and we use $\Uval_t = \bigcup_{s=t}^T S_s$ to estimate (recursively) $\E_{\pi_t}[L(h)]$ (see Figure~\ref{fig:evolution} in Appendix~\ref{app:illustrations}). Note that $S_t$ is used both for construction of $\pi_t$ and for estimation of $\E_{\pi_t}[L(h)]$ (it is both in $\Utrain_t$ and $\Uval_t$), resulting in efficient use of the data. It is possible to use any standard PAC-Bayes bound, e.g.,  Theorem~\ref{thm:pbkl}, to bound $\E_{\pi_1}[L(h)]$, and any PAC-Bayes-Empirical-Bernstein-style bound or the PAC-Bayes-split-kl bound to bound the excess losses $\E_{\pi_t}[L(h) - \gamma_t\E_{\pi_{t-1}}[L(h)]]$. The excess losses take more than three values, so in the next section we present a generalization of the PAC-Bayes-split-kl inequality to general discrete random variables, which may be of independent interest. The Recursive PAC-Bayes bound is presented in Section~\ref{sec:main}.

\section{Split-kl and PAC-Bayes-split-kl inequalities for discrete random variables}
\label{sec:split-kl}

The $\kl$ inequality is one of the tightest concentration of measure inequalities for binary random variables. Letting $\kl^{-1,+}(\hat{p},\varepsilon):=\max\lrc{p: p\in[0,1] \text{ and }\kl(\hat{p}\|p)\leq \varepsilon}$ denote the upper inverse of $\kl$ and $\kl^{-1,-}(\hat{p},\varepsilon):=\min\lrc{p: p\in[0,1] \text{ and }\kl(\hat{p}\|p)\leq \varepsilon}$ the lower inverse, it states the following.
\begin{theorem}[$\kl$ Inequality~\citep{Lan05,FBBT21,FBB22}]\label{thm:kl}
Let $Z_1,\cdots,Z_n$ be independent random variables bounded in the $[0,1]$ interval and with $\E[Z_i] = p$ for all $i$. Let $\hat{p}=\frac{1}{n}\sum_{i=1}^n Z_i$ be the empirical mean. Then, for any $\delta\in(0,1)$:
\[
\P[p\geq \kl^{-1,+}\lr{\hat p, \frac{1}{n}\ln\frac{1}{\delta}}
] \leq \delta \qquad ; \qquad
\P[p \leq \kl^{-1,-}\lr{\hat p, \frac{1}{n}\ln\frac{1}{\delta}}]\leq \delta.
\]
\end{theorem}

While the $\kl$ inequality is tight for binary random variables, it is loose for random variables taking more than two values due to its inability to exploit small variance. To address this shortcoming \citet{WS22} have presented the split-kl and PAC-Bayes-split-kl inequalities for ternary random variables. Ternary random variables naturally appear in a variety of applications, including analysis of excess losses, certain ways of analysing majority votes, and in learning with abstention. The bound of \citeauthor{WS22} is based on decomposition of a ternary random variable into a pair of binary random variables and application of the $\kl$ inequality to each of them. Their decomposition yields a tight bound in the binary and ternary case, but loose otherwise. The same decomposition was used by \citet{BG23} to derive a slight variation of the inequality, with the same limitations. We present a novel decomposition of discrete random variables into a superposition of binary random variables. Unlike the decomposition of \citeauthor{WS22}, which only applies in the ternary case, our decomposition applies to general discrete random variables. By combining it with $\kl$ bounds for the binary elements we obtain a tight bound. The decomposition is presented formally below and illustrated graphically in Figure~\ref{fig:decomposition} in Appendix~\ref{app:illustrations}.

\subsection{Split-kl inequality}

Let $Z \in \lrc{b_0,\dots,b_K}$ be a $(K+1)$-valued random variable with $b_0 < b_1 < \cdots < b_K$. For $j\in\lrc{1,\dots,K}$ define $Z_{|j} = \1[Z\geq b_j]$ and $\alpha_j = b_j - b_{j-1}$. Then $Z = b_0+\sum_{j=1}^K \alpha_j Z_{|j}$. For a sequence $Z_1,\dots,Z_n$ of $(K+1)$-valued random variables with the same support, let $Z_{i|j} = \1[Z_i\geq b_j]$ denote the elements of binary decomposition of $Z_i$. 

\begin{theorem}[Split-$\kl$ inequality for discrete random variables]\label{thm:Split_kl}
Let $Z_1,\dots,Z_n$ be i.i.d.\ random variables taking values in $\lrc{b_0,\dots,b_K}$ with $\E[Z_i] = p$ for all $i$. Let $\hat p_{|j} = \frac{1}{n}\sum_{i=1}^n Z_{i|j}$. Then for any $\delta\in(0,1)$:
\[
\P[p\geq b_0+ \sum_{j=1}^K \alpha_j \kl^{-1,+}\lr{\hat p_{|j},\frac{1}{n}\ln\frac{K}{\delta}}]\leq \delta.
\]
\end{theorem}

\begin{proof}
    Let $p_{|j} = \E[\hat p_{|j}]$, then $p = b_0+\sum_{j=1}^K \alpha_j p_{|j}$ and
    \[\P[p\geq b_0+ \sum_{j=1}^K \alpha_j \kl^{-1,+}\lr{\hat p_{|j},\frac{1}{n}\ln\frac{K}{\delta}}] \leq \P[\exists j: p_{|j}\geq \kl^{-1,+}\lr{\hat p_{|j},\frac{1}{n}\ln\frac{K}{\delta}}] \leq \delta,
    \]where the first inequality is by the decomposition of $p$ and the second inequality is by the union bound and Theorem~\ref{thm:kl}.
\end{proof}

\subsection{PAC-Bayes-Split-kl inequality}

Let $f:\HH\times\ZZ\to\lrc{b_0,\dots,b_K}$ be a $(K+1)$-valued loss function. (To connect it to the earlier examples, in the binary prediction case we would have $\ZZ=\XX\times\YY$ with elements $Z = (X,Y)$ and $f(h,Z) = \ell(h(X),Y)$, but we will need a more general space $\ZZ$ later.) For $j\in\lrc{1,\dots,K}$ let $f_j(\cdot,\cdot) = \1[f(\cdot,\cdot) \geq b_j]$. Let $\DD_Z$ be an unknown distribution on $\ZZ$. For $h\in\HH$ let $F(h) = \E_{\DD_Z}[f(h,Z)]$ and $F_j(h)=\E_{\DD_Z}[f_j(h,Z)]$. Let $S = \lrc{Z_1,\dots,Z_n}$ be an i.i.d.\ sample according to $\DD_Z$ and $\hat F_j(h,S)=\frac{1}{n}\sum_{i=1}^n f_j(h, Z_i)$. 

\begin{theorem}[PAC-Bayes-Split-kl Inequality]\label{thm:pac-bayes-split-kl-inequality}
For any distribution $\pi$ on $\HH$ that is independent of $S$ and any $\delta\in(0,1)$:
\[
    \P[\exists\rho\in\mathcal{P}:\E_\rho[F(h)] \geq b_0 + \sum_{j=1}^K \alpha_j \kl^{-1,+}\lr{\E_\rho[\hat F_j(h,S)],\frac{\KL(\rho\|\pi)+\ln\frac{2K\sqrt{n}}{\delta}}{n}}]\leq \delta,
\]
where $\mathcal{P}$ is the set of all possible probability distributions on $\HH$ that can depend on $S$.
\end{theorem}

\begin{proof}
    We have $f(\cdot,\cdot) = b_0 + \sum_{j=1}^K\alpha_j f_j(\cdot,\cdot)$ and $F(h) = b_0 + \sum_{j=1}^K \alpha_j F_j(h)$. Therefore,
\begin{multline*}
    \P[\exists\rho\in\mathcal{P}:\E_\rho[F(h)] \geq b_0 + \sum_{j=1}^K \alpha_j \kl^{-1,+}\lr{\E_\rho[\hat F_j(h,S)],\frac{\KL(\rho\|\pi)+\ln\frac{2K\sqrt{n}}{\delta}}{n}}]
    \\\leq\P[\exists\rho\in\mathcal{P}\text{ and }\exists j:\E_\rho[F_j(h)] \geq \kl^{-1,+}\lr{\E_\rho[\hat F_j(h,S)],\frac{\KL(\rho\|\pi)+\ln\frac{2K\sqrt{n}}{\delta}}{n}}]
    \leq \delta,
\end{multline*}
where the first inequality is by the decomposition of $F$ and the second inequality is by the union bound and application of Theorem~\ref{thm:pbkl} to $F_j$ (note that $f_j$ is a zero-one loss function).
\end{proof}

\section{Recursive PAC-Bayes bound}
\label{sec:main}

Now we derive a Recursive PAC-Bayes bound based on the loss decomposition in Equation~\ref{eq:RPB}. In order to bound $\E_{\pi_t}[L(h) - \gamma_t \E_{\pi_{t-1}}[L(h')]]$ we need empirical estimates of $L(h) - \gamma_t \E_{\pi_{t-1}}[L(h')]$. We denote $F_{\gamma_t}(h,\pi_{t-1}) = L(h) - \gamma_t \E_{\pi_{t-1}}[L(h')] = \E_{\DD\times\pi_{t-1}}[\ell(h(X),Y) - \gamma_t\ell(h'(X),Y)]$, where $\DD\times\pi_{t-1}$ is a product distribution on $\XX\times\YY\times\HH$ and $h'\in\HH$ is sampled according to $\pi_{t-1}$. We define $f_{\gamma_t}(h,(X,Y,h')) = \ell(h(X),Y) - \gamma_t \ell(h'(X),Y) \in \lrc{-\gamma_t, 0, 1-\gamma_t, 1}$, then $F_{\gamma_t}(h,\pi_{t-1}) = \E[f_{\gamma_t}(h,(X,Y,h'))]$. We let $\lrc{b_{t|0},b_{t|1},b_{t|2},b_{t|3}} = \lrc{-\gamma_t, 0, 1-\gamma_t, 1}$ and $f_{\gamma_t|j}(h,(X,Y,h')) = \1[f_{\gamma_t}(h,(X,Y,h'))\geq b_{t|j}]$. We let $F_{\gamma_t|j}(h,\pi_{t-1}) = \E[f_{\gamma_t|j}(h,(X,Y,h'))]$, then $F_{\gamma_t}(h,\pi_{t-1}) = -\gamma_t + \sum_{j=1}^3 (b_{t|j}-b_{t|j-1})F_{\gamma_t|j}(h,\pi_{t-1})$.

For each sample $(X_i,Y_i)$ we sample a prediction rule $h_i$ to serve for $h'$. We let $\hat \pi_{t-1} = \lrc{h_{t-1,1},h_{t-1,2},\dots}$ to be the resulting sequence of prediction rules sampled independently according to $\pi_{t-1}$. We only use the first elements of the sequence in accordance with the size of the sample used in the corresponding estimation. We define $\hat F_{\gamma_t|j}(h,\Uval_t, \hat \pi_{t-1}) = \frac{1}{\nval_t} \sum_{(X,Y)\in\Uval_t, h'\in\hat\pi_{t-1}} f_{\gamma_t|j}(h,(X,Y,h'))$, where $\nval_t = |\Uval_t|$ and we take the first $\nval_t$ elements from $\hat \pi_{t-1}$, one for each sample. Note that $\E[\hat F_{\gamma_t|j}(h,S,\hat \pi_{t-1})] = F_{\gamma_t|j}(h,\pi_{t-1})$. Now we are ready to state the bound.
\begin{theorem}[Recursive PAC-Bayes Bound]
\label{thm:RPB}
    Let $S = S_1\cup\dots\cup S_T$ be an i.i.d.\ sample split in an arbitrary way into $T$ non-overlapping subsamples, and let $\Utrain_t = \bigcup_{s=1}^t S_s$ and $\Uval_t = \bigcup_{s=t}^T S_s$. Let $\nval_t = |\Uval_t|$. Let $\pi_0^*,\pi_1^*,\dots,\pi_T^*$ be a sequence of distributions on $\HH$, where $\pi_t^*$ is allowed to depend on $\Utrain_t$, but not the rest of the data. Let $\gamma_2,\dots,\gamma_T$ be a sequence of coefficients, where $\gamma_t$ is allowed to depend on $\Utrain_{t-1}$, but not the rest of the data. For $t\in\lrc{1,\dots,T}$ let $\PP_t$ be a set of distributions on $\HH$, which are allowed to depend on $\Utrain_t$. Then for any $\delta\in(0,1)$:
    \[
    \P[\exists t\in\lrc{1,\dots,T}\text{ and } \pi_t \in \PP_t: \E_{\pi_t}[L(h)] \geq \B_t(\pi_t)] \leq \delta,
    \]
    where $\B_t(\pi_t)$ is a PAC-Bayes bound on $\E_{\pi_t}[L(h)]$ defined recursively as follows. For $t=1$
    \[
    \B_1(\pi_1) = \kl^{-1,+}\lr{\E_{\pi_1}[\hat L(h,S)],\frac{\KL(\pi_1\|\pi_0^*)+\ln\frac{2T\sqrt{n}}{\delta}}{n}}.
    \]
    For $t\geq 2$ we let $\Ex_t(\pi_t,\gamma_t)$ denote a PAC-Bayes bound on $\E_{\pi_t}[L(h) - \gamma_t \E_{\pi_{t-1}^*}[L(h')]]$ given by
   \begin{align*}
    &\Ex_t(\pi_t,\gamma_t) = -\gamma_t +\\  
    &\sum_{j=1}^3 (b_{t|j} - b_{t|j-1})
    \kl^{-1,+}
    \lr{\E_{\pi_t}\lrs{\hat F_{\gamma_t|j}(h,\Uval_t, \hat \pi_{t-1}^*)}, \frac{\KL(\pi_t\|\pi_{t-1}^*)+\ln\frac{6T\sqrt{\nval_t}}{\delta}}{\nval_t}}
    \end{align*}
    and then
    \begin{equation}
    \label{eq:RPB-rec}
    \B_t(\pi_t) = \Ex_t(\pi_t) + \gamma_t \B_{t-1}(\pi_{t-1}^*).
    \end{equation}
\end{theorem}

\begin{proof}
    By Theorem~\ref{thm:pbkl} we have $\P[\exists \pi_1 \in \PP_1: \E_{\pi_1}[L(h)] \geq B_1(\pi_1)] \leq \frac{\delta}{T}$. Further, by Theorem~\ref{thm:pac-bayes-split-kl-inequality} for $t\in\lrc{2,\dots,T}$ we have \\
    $\P[\exists \pi_t\in\PP_t: \E_{\pi_t}[L(h) - \gamma_t \E_{\pi_{t-1}^*}[L(h')]] \geq \Ex_t(\pi_t,\gamma_t)] \leq \frac{\delta}{T}$. The theorem follows by a union bound and the recursive decomposition of the loss (Equation~\ref{eq:RPB}).
\end{proof}

\paragraph*{Discussion}
\begin{itemize}
    \item Note that $\pi_1^*,\dots,\pi_T^*$ can be constructed sequentially, but $\pi_t^*$ can only be constructed based on the data in $\Utrain_t$, meaning that in the construction of $\pi_t^*$ we can only rely on $\E_{\pi_t}\lrs{\hat F_{\gamma_t|j}(h,S_t, \hat \pi_{t-1})}$, but not on $\E_{\pi_t}\lrs{\hat F_{\gamma_t|j}(h,\Uval_t, \hat \pi_{t-1})}$. Also note that $S_t$ is part of both $\Utrain_t$ and $\Uval_t$ (see Figure~\ref{fig:evolution} in Appendix~\ref{app:illustrations} for a graphical illustration). In other words, when we evaluate the bounds we can use additional data. And even though the additional data can only be used in the evaluation stage, we can still use the knowledge that we will get more data for evaluation when we construct $\pi_t^*$. For example, we can take
    \begin{equation}
    \label{eq:opt-pi-t}
    \pi_t^* = \arg\min_\pi \sum_{j=1}^3 (b_{t|j} - b_{t|j-1})\kl^{-1,+}\lr{\E_{\pi}\lrs{\hat F_{\gamma_t|j}(h,S_t, \hat \pi_{t-1}^*)}, \frac{\KL(\pi\|\pi_{t-1}^*)+\ln\frac{6T\sqrt{\nval_t}}{\delta}}{\nval_t}}.
    \end{equation}
    The empirical losses above are calculated on $S_t$ corresponding to $\pi_t^*$, but the sample sizes $\nval_t$ correspond to the size of the validation set $\Uval_t$ rather than the size of $S_t$. This allows to be more aggressive in deviating with $\pi_t^*$ from $\pi_{t-1}^*$ by sustaining larger $\KL(\pi_t^*\|\pi_{t-1}^*)$ terms.
    \item Similarly, $\gamma_2,\dots,\gamma_T$ can also be constructed sequentially, as long as $\gamma_t$ only depends on $\Utrain_{t-1}$ (otherwise $\hat F_{\gamma_t|j}(h,S_t,\hat \pi_{t-1}^*)$ becomes a biased estimate of $F_{\gamma_t|j}(h,\pi_{t-1}^*)$).
    \item We naturally want to have improvement over recursion rounds, meaning $\B_t(\pi_t^*) < \B_{t-1}(\pi_{t-1}^*)$. Plugging this into Equation~\ref{eq:RPB-rec}, we obtain $\Ex(\pi_t^*,\gamma_t) + \gamma_t B_{t-1}(\pi_{t-1}^*) < B_{t-1}(\pi_{t-1}^*)$, which implies that we want $\gamma_t$ to be sufficiently small to satisfy $\gamma_t < 1 - \frac{\Ex_t(\pi_t^*,\gamma_t)}{B_{t-1}(\pi_{t-1}^*)}$. At the same time, $\gamma_t$ should be non-negative. Therefore, improvement over recursion steps can only be maintained as long as $\Ex_t(\pi_t^*,\gamma_t) < B_{t-1}(\pi_{t-1}^*)$. We also note that if $\E_{\pi_t^*}[L(h)] > 0$, then overly small $\gamma_t$ can increase the excess loss. Thus, unless the loss is close to zero, $\gamma_t$ should not be too small either.
\end{itemize}

\section{Experiments}
\label{sec:experiments}

In this section, we provide an empirical comparison of our Recursive PAC-Bayes (RPB) procedure to the following prior work: i) Uninformed priors (Uninformed), \citep{DR17}; ii) Data-informed priors (Informed) \citep{APS07,PRSS21}; iii) Data-informed prior + excess loss (Informed + Excess) \citep{MGG20,WS22}. All the experiments were run on a laptop.

We start with describing the details of the optimization procedure, and then present the results.

\subsection{Details of the optimization and evaluation procedure}
\label{sec:optimization}
We constructed $\pi_1^*,\dots,\pi_T^*$ 
sequentially 
using the optimization objective \ref{eq:opt-pi-t}, and computed the bound using the recursive procedure in Theorem~\ref{thm:RPB}.
There are a few technical details concerning convexity of the optimization procedure and infinite size of the set of prediction rules $\HH$ that we address next.

\subsubsection{Convexification of the loss functions}\label{sec:surrogate-loss}

The functions $f_{\gamma_t|j}(h,(X,Y,h'))$ defined in Section~\ref{sec:main} are non-convex and non-differentiable:
$f_{\gamma_t|j}(h,(X,Y,h')) = \1[f_{\gamma_t}(h,(X,Y,h'))\geq b_{t|j}] =$ \\
$ \1[\ell(h(X),Y) - \gamma_t \ell(h'(X),Y) \ge b_{t|j}]$. In order to facilitate optimization, we approximate the external indicator function $\1[z\ge z_0]$ by a sigmoid function $\omega(z;c_1,z_0)=(1+\exp(c_1(z-z_0)))^{-1}$ with a fixed parameter $c_1>0$ specified in Appendix~\ref{app:sec:other-exp-details}.

Furthermore, since the zero-one loss $\ell(h(X),Y)$ is also non-differentiable, we adopt the cross-entropy loss, as in most modern training procedures \citep{PRSS21}. Specifically, for a $k$-class classification problem, let $h:\mathcal{X}\rightarrow \R^k$ represent the function implemented by the neural network, assigning each class a real value. 
Let $u=h(X)$ be the assignment, with $u_i$ being the $i$-th value of the vector.
To convert this real-valued vector into a probability distribution over classes, we apply the softmax function $\sigma:\R^k\rightarrow \Delta^{k-1}$, where $\sigma(u)_i=\exp(c_2 u_i)/\sum_j \exp(c_2 u_j)$ for some $c_2>0$ for each entry. The cross-entropy loss $\ell^{\text{ce}}:\R^k\times [k]\rightarrow \R$ is defined by $\ell^{\text{ce}}(u,Y)=-\log(\sigma(u)_Y)$. However, since this loss is unbounded, whereas the PAC-Bayes-kl bound requires losses within $[0,1]$, we enforce a $[0,1]$-valued cross-entropy loss by first lower-bounding the probability assigned to $Y$ by taking $\tilde{\sigma}(u)_Y=\max(\sigma(u)_Y,p_{\min})$ for $p_{\min}>0$, and then  rescaling it to $[0,1]$ by taking $\tilde{\ell}^{\text{ce}}(u,Y)=-\log(\tilde{\sigma}(u)_Y)/\log(1/p_{\min})$ \citep{DR17}.

We emphasize that in the evaluation of the bound (using Theorem~\ref{thm:RPB}), we directly compute the zero-one loss and the $f_{\gamma_t|j}$ functions without employing the approximations.

\subsubsection{Relaxation of the PAC-Bayes-kl bound}\label{sec:relaxation-of-pac-bayes-kl}
The PAC-Bayes-$\kl$ bound is often criticized for being unfriendly to optimization \citep{RTS24}. Therefore, several relaxations have been proposed, including the PAC-Bayes-classic bound \citep{McA99}, the PAC-Bayes-$\lambda$ bound \citep{TIWS17}, and the PAC-Bayes-quadratic bound \citep{RTS19,PRSS21}, among others. In our optimization we have adopted the bound of \citet{McA99} instead of the kl-based bounds in Equation~\ref{eq:opt-pi-t}.

We again emphasize that in the evaluation of the bound we used the kl-based bounds in Theorem~\ref{thm:RPB}.

\subsubsection{Estimation of $\E_\pi[\cdot]$} \label{sec:estimation-gibbs-loss}

Due to the infinite size of $\HH$ and lack of a closed-form expression for $\E_{\pi_1}[\hat L(h,S)]$ and $\E_{\pi_t}[\hat F_{\gamma_t|j}(h,\Uval_t,\hat \pi_{t-1}^*)]$ appearing in Theorem~\ref{thm:RPB}, we approximate them by sampling \citep{PRSS21}. For optimization, we sample one classifier for each mini-batch during stochastic gradient descent. For evaluation, we sample one classifier for each data in the corresponding evaluation dataset. Due to approximation of the empirical quantities the final bound in Theorem~\ref{thm:RPB} requires an additional concentration bound. (We note that the extra bound is only required for computation of the final bound, but not for optimization of $\hat \pi_t^*$.) Specifically, let $\hat \pi_t^*=\{h_{t,1},h_{t,2},\cdots,h_{t,m}\}$ be $m$ samples drawn i.i.d. from $\pi_t^*$. Then 
for any function $f(h)$ taking values in $[0,1]$ (which is the case for $\hat L(h,S)$ and $\hat F_{\gamma_t|j}(h,\Uval_t,\hat \pi_{t-1}^*)$) and $\delta'\in(0,1)$ we have
\[
    \P[\E_{\pi_t^*}[f(h)] \geq \kl^{-1,+}\lr{\frac{1}{m}\sum_{i=1}^m f(h_{t,i}), \frac{1}{m}\log \frac{1}{\delta'}}]\leq \delta'.
\]
It is worth noting that $\E_{\pi_t^*}[f(h)]$ is evaluated for a fixed $\pi_t^*$, meaning that there is no selection involved, and therefore no $\KL$ term appears in the bound above. We, of course, take a union bound over all the quantities being estimated. 

\subsection{Experimental results}

We evaluated our approach and compared it to prior work using multi-class classification tasks on MNIST \citep{lecun-mnisthandwrittendigit-2010} and Fashion MNIST \citep{xiao2017fashion} datasets, both with 60000 training data. The experimental setup was based on the work of \citet{DR17} and \citet{PRSS21}.  Similar to them we used Gaussian distributions for all the priors and posteriors, modeled by probabilistic neural networks. Technical details are provided in Appendix~\ref{sec:app:experimental details}.

The empirical evaluations on MNIST and Fashion MNIST are presented in Table \ref{tab:mnist_results} and Table \ref{tab:fmnist_results} respectively. For the Uninformed approach, we trained and evaluated the bound using the entire training dataset directly. For the other two baseline methods, Informed and Informed + Excess Loss, we used half of the training data to train the informed prior and an ERM $h^*$ for the excess loss, and the other half to learn the posterior. For our Recursive PAC-Bayes (RPB), we chose $\gamma_t=1/2$ for all $t$, and conducted experiments with $T=2$, $T=4$, $T=6$, and $T=8$ to study the impact of recursion depth. (Each value of $T$ corresponded to a separate run of the algorithm and a separate evaluation of the bound, i.e., they should not be seen as successive refinements.) We applied geometric split of the data, where at each recursion step the data were split in two equal halves. Specifically, for $T=2$ the split was (30000,30000) points, for $T=4$ it was (7500, 7500, 15000, 30000) points, for $T=6$ it was (1875, 1875, 3750, 7500, 15000, 30000) points, and for $T=8$ it was (468, 469, 938, 1875, 3750, 7500, 15000, 30000) points. The motivation was to let early recursion steps with few data points efficiently learn the prior, while keeping enough data for fine-tuning in the later steps. Note that with this approach the value of $\nval_t = |\Uval_t| = \sum_{s=t}^T |S_s|$, which is in the denominator of the bounds in Theorem \ref{thm:RPB}, is at least $\frac{n}{2}$. 

\begin{table}[H]
\caption{Comparison of the classification loss of the final posterior $\rho$ on the entire training data, $\E_\rho[\hat L(h,S)]$ (Train 0-1), and on the testing data, $\E_\rho[\hat L(h,S^{\text{test}})]$ (Test 0-1), and the corresponding bounds for each method on MNIST. We report the mean and one standard deviation over 5 repetitions. ``Unif.'' abbreviates the Uniform approach, ``Inf.'' the Informed, ``Inf. + Ex.'' the Informed + Excess Loss, and ``RPB'' the Recursive PAC-Bayes.
}
\centering
\begin{tabular}{l|ccc}
\multirow{2}{*}{} & \multicolumn{3}{c}{MNIST}\\
& Train 0-1 & Test 0-1 & Bound \\
\hline
Uninf.       & .343 (2e-3) & .335 (3e-3) & .457 (2e-3) \\
Inf.         & .377 (8e-4) & .371 (6e-3) & .408 (9e-4) \\
Inf. + Ex.   & .157 (2e-3) & .151 (3e-3) & .192 (2e-3) \\
RPB $T=2$    & .143 (2e-3) & .139 (3e-3) & .321 (3e-3) \\
RPB $T=4$    & .112 (1e-3) & .109 (1e-3) & .203 (8e-4) \\
RPB $T=6$    & .103 (1e-3) & .101 (1e-3) & .166 (1e-3) \\
RPB $T=8$    & \textbf{.101 (1e-3)} & \textbf{.097 (2e-3)} & \textbf{.158 (2e-3)} \\
\end{tabular}
\label{tab:mnist_results}
\end{table}

\begin{table}[H]
\caption{Comparison of the classification loss of the final posterior $\rho$ on the entire training data, $\E_\rho[\hat L(h,S)]$ (Train 0-1), and on the testing data, $\E_\rho[\hat L(h,S^{\text{test}})]$ (Test 0-1), and the corresponding bounds for each method on Fashion MNIST. We report the mean and one standard deviation over 5 repetitions. ``Unif.'' abbreviates the Uniform approach, ``Inf.'' the Informed, ``Inf. + Ex.'' the Informed + Excess Loss, and ``RPB'' the Recursive PAC-Bayes.
}
\centering
\begin{tabular}{l|ccc}
\multirow{2}{*}{} & \multicolumn{3}{c}{Fashion MNIST}\\
& Train 0-1 & Test 0-1 & Bound \\
\hline
Uninf.       & .382 (2e-3) & .384 (2e-3) & .464 (2e-3) \\
Inf.         & .412 (1e-3) & .413 (6e-3) & .440 (1e-3) \\
Inf. + Ex.   & .280 (4e-3) & .285 (5e-3) & .342 (6e-3) \\
RPB $T=2$    & .257 (3e-3) & .266 (5e-3) & .404 (3e-3) \\
RPB $T=4$    & .203 (2e-3) & .213 (3e-3) & .293 (1e-3) \\
RPB $T=6$    & .186 (4e-4) & .198 (1e-3) & .255 (1e-3) \\
RPB $T=8$    & \textbf{.181 (1e-3)} & \textbf{.192 (3e-3)} & \textbf{.242 (1e-3)} \\
\end{tabular}
\label{tab:fmnist_results}
\end{table}


Tables \ref{tab:mnist_results} and \ref{tab:fmnist_results} show that while $T=2$, which corresponds to the data split in the Informed and Informed + Excess Loss approaches, outperforms the Uninformed and Informed approaches on both datasets, it still underperforms compared to the Informed + Excess Loss approach. However, with deeper recursions, RPB yields dramatic improvements, offering significant advantages over prior work.

Tables \ref{tab:detail-rpb8-mnist} and \ref{tab:detail-rpb8-fmnist} provide a glimpse into the training progress of RPB with $T=8$ by showing the evolution of the key quantities along the recursive process. Similar tables for other values of $T$ are provided in Appendix \ref{app:sec:more-tables}, along with training details for other methods. The tables show an impressive reduction of the $\KL$ term and significant improvement of the bound as the recursion proceeds, demonstrating the effectiveness of the approach.

\begin{table}[H]
\caption{Insight into the training process of the Recursive PAC-Bayes for $T=8$ on MNIST. The table shows the evolution of $\Ex_t(\pi_t^*,\gamma_t)$, $B_t(\pi_t^*)$, and other quantities as the training progresses with $t$. We define $\hat F_{\gamma_t}(h,\Uval_t,\hat \pi_{t-1}) = -\gamma_t + \sum_{j=1}^3 (b_{t|j}-b_{t|j-1})\hat F_{\gamma_t|j}(h,\Uval_t,\hat\pi_{t-1})$.}
\centering
\begin{tabular}{c|ccccc|c}
$t$  & $\nval_t$ & $\E_{\pi_t}[\hat F_{\gamma_t}(h,\Uval_t,\hat \pi_{t-1})]$ & $\frac{\KL(\pi_t^*\|\pi_{t-1}^*)}{\nval_t}$ & $\Ex_t(\pi_t^*,\gamma_t)$ & $B_t(\pi^*_t)$ & Test 0-1 \\ \hline
1 & 60000 &             & .009 (3e-4) &             & .612 (9e-3) & .532 (.011) \\
2 & 59532 & -0.046 (4e-3) & .031 (1e-3) & .114 (2e-3) & .421 (5e-3) & .215 (7e-3)  \\
3 & 59063 & .040 (3e-3) & .013 (9e-4) & .125 (3e-3) & .336 (2e-3) & .146 (3e-3)  \\
4 & 58125 & .049 (1e-3) & .005 (3e-4) & .099 (1e-3) & .267 (7e-4) & .120 (2e-3)  \\
5 & 56250 & .052 (4e-4) & .002 (1e-4) & .083 (1e-3) & .217 (1e-3) & .111 (2e-3)  \\
6 & 52500 & .051 (1e-3) & .001 (4e-5) & .076 (1e-3) & .185 (1e-3) & .104 (2e-3) \\
7 & 45000 & .050 (1e-3) & 8e-4 (6e-5) & .073 (1e-3) & .166 (1e-3) & .099 (1e-3)  \\
8 & 30000 & .050 (1e-3) & 6e-4 (4e-5) & .074 (1e-3) & .158 (2e-3) & .097 (2e-3) 
\end{tabular}
\label{tab:detail-rpb8-mnist}
\end{table}

\begin{table}[H]
\caption{Insight into the training process of the Recursive PAC-Bayes for $T=8$ on Fashion MNIST.}
\centering
\begin{tabular}{c|ccccc|c}
$t$  & $\nval_t$ & $\E_{\pi_t}[\hat F_{\gamma_t}(h,\Uval_t,\hat \pi_{t-1})]$ & $\frac{\KL(\pi_t^*\|\pi_{t-1}^*)}{\nval_t}$ & $\Ex_t(\pi_t^*,\gamma_t)$ & $B_t(\pi^*_t)$ & Test 0-1 \\ \hline
1 & 60000 &             & .003 (7e-5) &             & .733 (7e-3) & .686 (8e-3) \\
2 & 59532 & -0.043 (8e-3) & .023 (6e-4) & .104 (9e-3) & .470 (8e-3) & .309 (7e-3)  \\
3 & 59063 & .083 (4e-3) & .008 (3e-4) & .161 (3e-3) & .396 (4e-3) & .242 (1e-3)  \\
4 & 58125 & .090 (3e-3) & .004 (5e-4) & .142 (4e-3) & .341 (5e-4) & .216 (5e-3)  \\
5 & 56250 & .093 (3e-3) & .001 (2e-4) & .126 (3e-3) & .297 (4e-3) & .204 (4e-3)  \\
6 & 52500 & .090 (1e-3) & 6e-4 (6e-5) & .117 (1e-3) & .265 (1e-3) & .195 (3e-3) \\
7 & 45000 & .090 (1e-3) & 4e-4 (2e-5) & .115 (1e-3) & .248 (1e-3) & .195 (5e-4)  \\
8 & 30000 & .090 (1e-3) & 4e-4 (1e-5) & .117 (1e-3) & .242 (1e-3) & .192 (3e-3) 
\end{tabular}
\label{tab:detail-rpb8-fmnist}
\end{table}

\section{Discussion}
\label{sec:discussion}

We have presented the first PAC-Bayesian bound that supports sequential prior updates and preserves confidence information on the prior. The work closes a long-standing gap between Bayesian and Frequentist learning by making sequential data processing and sequential updates of prior knowledge meaningful and beneficial in the frequentist framework, as it has always been in the Bayesian framework. We have shown that apart from theoretical beauty the approach is beneficial in practice.

\section{Appendix}
\subsection{Illustrations}
\label{app:illustrations}
In this appendix we provide graphical illustrations of the basic concepts presented in the paper.

\begin{figure}[H]
    \centering
    \includegraphics[width=\textwidth]{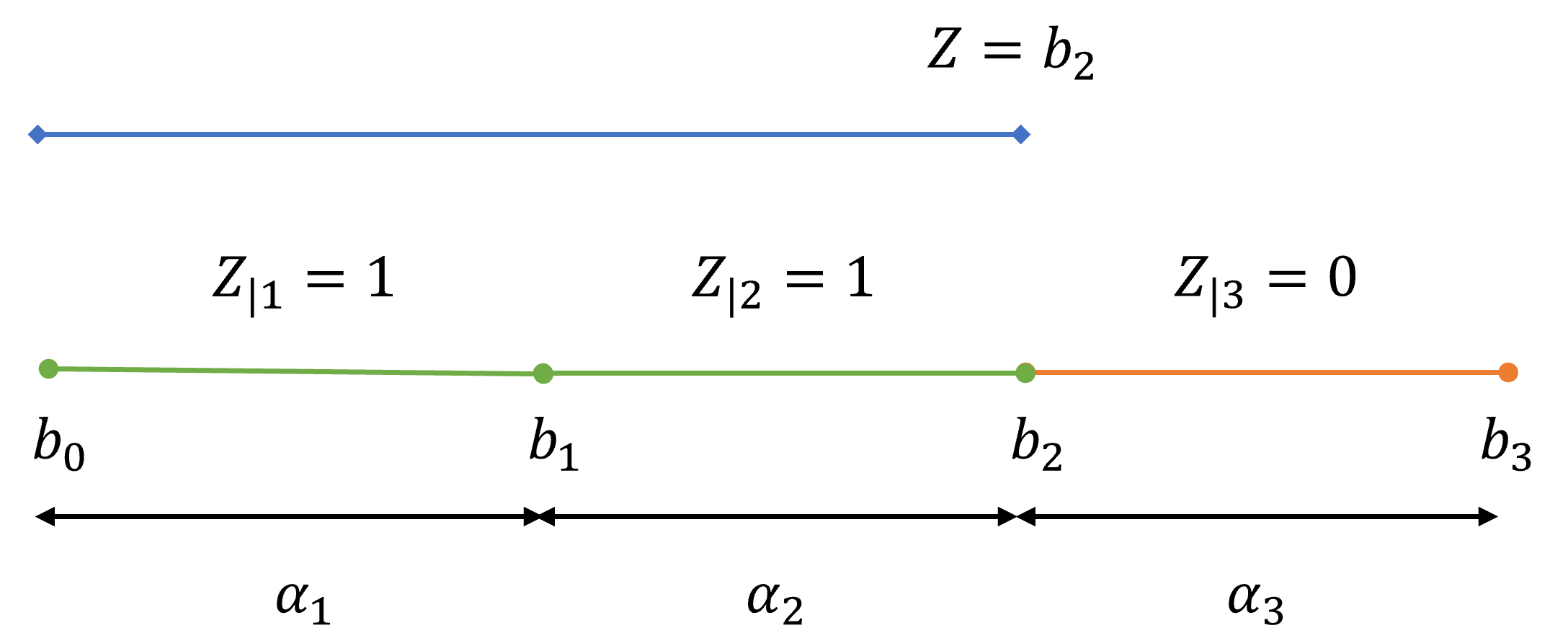}
    \caption{\textbf{Decomposition of a discrete random variable into a superposition of binary random variables.} The figure illustrates a decomposition of a discrete random variable $Z$ with domain of four values $b_0 < b_1 < b_2 < b_3$ into a superposition of three binary random variables, $Z = b_0 + \sum_{j=1}^3 \alpha_j Z_{|j}$. A way to think about the decomposition is to compare it to a progress bar. In the illustration $Z$ takes value $b_2$, and so the random variables $Z_{|1}$ and $Z_{|2}$ corresponding to the first two segments ``light up'' (take value 1), whereas the random variable $Z_{|3}$ corresponding to the last segment remains ``turned off'' (takes value 0). The value of $Z$ equals the sum of the lengths $\alpha_j$ of the ``lighted up'' segments.}
    \label{fig:decomposition}
\end{figure}

\begin{figure}[H]
    \centering
    \includegraphics[width=.95\textwidth]{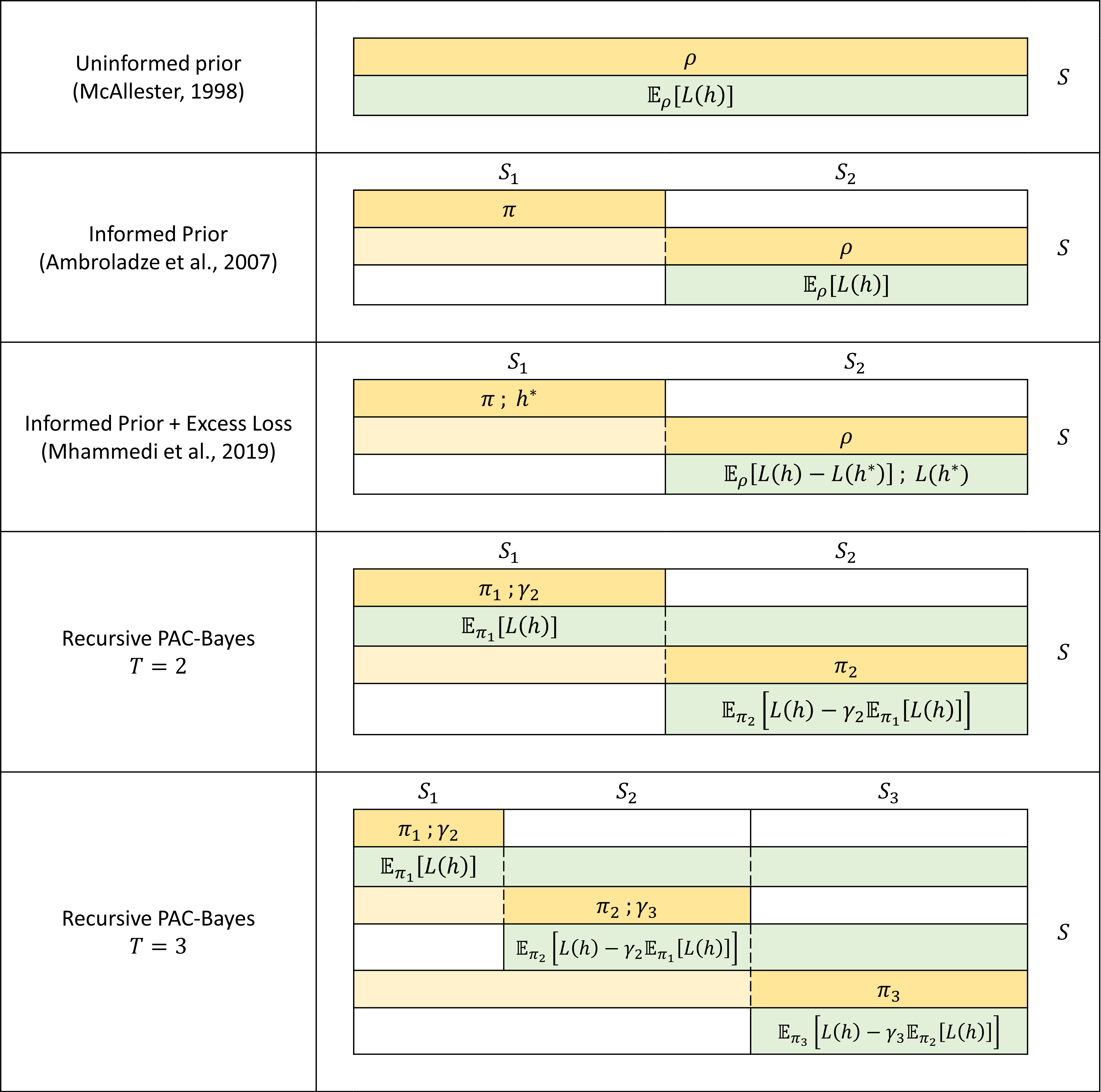}
    \caption{\textbf{Evolution of PAC-Bayes.} The figure shows how data are used by different methods. Dark yellow shows data used directly for optimization of the indicated quantities. Light yellow shows data involved indirectly through dependence on the prior. Light green shows data used for estimation of the indicated quantities. In Recursive PAC-Bayes data are released and used sequentially chunk-by-chunk, as indicated by the dashed lines. For example, in the $T=3$ case $\E_{\pi_1}[L(h)]$ is first evaluated on $S_1$ to construct $\pi_1$ and $\gamma_2$, then in the first recursion step on $S_1\cup S_2$, and in the last step on all $S$.
    }
    \label{fig:evolution}
\end{figure}

\subsection{Experimental details}\label{sec:app:experimental details}

In this section, we provide the details of the datasets in Appendix \ref{app:sec:datasets}, our neural network architectures in Appendix \ref{app:sec:nn}, and other details in Appendix \ref{app:sec:other-exp-details}. We provide further statistics for all the methods on both datasets in Appendix \ref{app:sec:more-tables}.

\subsubsection{Datasets}\label{app:sec:datasets}
We perform our evaluation on two datasets, MNIST \citep{lecun-mnisthandwrittendigit-2010} and Fashion MNIST \citep{xiao2017fashion}. We will introduce these two datasets in the following.

\paragraph{MNIST}
The MNIST (Modified National Institute of Standards and Technology) dataset is one of the most renowned and widely used datasets in the field of machine learning, particularly for training and testing in the domain of image processing and computer vision. It consists of a large collection of handwritten digit images, spanning the numbers 0 through 9. 

The MNIST dataset comprises a total of 70,000 grayscale images of handwritten digits, where the training set has 60,000 images and the test set has 10,000 images. Each image in the dataset is 28x28 pixels, resulting in a total of 784 pixels per image. The images are in grayscale, with pixel values ranging from 0 (black) to 255 (white). Each image is associated with a label from 0 to 9, indicating the digit that the image represents. The images are typically stored in a single flattened array of 784 elements, although they can also be represented in a 28x28 matrix format.

\paragraph{Fashion MNIST}
The Fashion MNIST dataset is a contemporary alternative to the traditional MNIST dataset, created to provide a more challenging benchmark for machine learning algorithms. It consists of images of various clothing items and accessories, offering a more complex and varied dataset for image classification tasks.

The Fashion MNIST dataset contains a total of 70,000 grayscale images, where the training set has 60,000 images and the test set has 10,000 images. Each image in the dataset is 28x28 pixels, resulting in a total of 784 pixels per image. The images are in grayscale, with pixel values ranging from 0 (black) to 255 (white). Each image is associated with one of 10 categories, representing different types of fashion items. The categories are: 1. T-shirt/top 2. Trouser 3. Pullover 4. Dress 5. Coat 6. Sandal 7. Shirt 8. Sneaker 9. Bag 10. Ankle boot. Similar to MNIST, the images are stored in a single flattened array of 784 elements but can also be represented in a 28x28 matrix format.

\subsubsection{Neural network architectures}\label{app:sec:nn}

For all methods, we adopt a family of factorized Gaussian distributions to model both priors and posteriors, characterized by the form $\pi=\mathcal{N}(w,\sigma \mathbf I)$ where $w\in\R^d$ denotes the mean vector, and $\sigma$ represents the scalar variance. We use feedforward neural networks for the MNIST dataset \citep{lecun-mnisthandwrittendigit-2010}, while using convolutional neural networks for the Fashion MNIST dataset \citep{xiao2017fashion}.

Both our feedforward neural network and convolutional neural network are probabilistic, and each layer has a factorized (i.e. mean-field) Gaussian distribution.

Our feedforward neural network has the following architecture:
\begin{enumerate}[left=0.3cm]
    \item Input layer. Input size: $28 \times 28$ (flattened to 784 features).

    \item Probabilistic linear layer 1. Input features: 784, output features: 600, activation: ReLU.

    \item Probabilistic linear layer 2. Input features: 600, output features: 600, activation: ReLU.

    \item Probabilistic linear layer 3. Input features: 600, output features: 600, activation: ReLU.

    \item Probabilistic linear layer 4. Input features: 600, output features: 10, activation: Softmax.
\end{enumerate}

Our convolutional neural network has the following architecture:
\begin{enumerate}
    \item Input layer. Input size: $1 \times 28 \times 28$.

    \item Probabilistic convolutional layer 1. Input channels: 1, output channels: 32, kernel size: 3x3, activation: ReLU.

    \item Probabilistic convolutional layer 2. Input channels: 32, output channels: 64, kernel size: 3x3, activation: ReLU.

    \item Max pooling layer. Pooling size: 2x2.

    \item Flattening layer. Flattens the output from the previous layers into a single vector.

    \item Probabilistic linear layer 1. Input features: 9216, output features: 128, activation: ReLU.

    \item Probabilistic linear layer 2 (output layer). Input features: 128, output features: 10, activation: Softmax.
\end{enumerate}

\subsubsection{Other details in the experiments}\label{app:sec:other-exp-details}

\paragraph{General for all methods}
The methods in comparisons are trained and evaluated using the procedure described in Section \ref{sec:idea} and visually illustrated in Figure \ref{fig:evolution}. We will provide some further details for each method later in the following.
For all methods in comparison, we apply the optimization and evaluation method described in 
Section \ref{sec:optimization}. For the approximation described in Section \ref{sec:surrogate-loss}, we set the parameters $c_1=c_2=5$. The lower bound for the prediction $p_{\min}=1e-5$. The $\delta$ in our bound and all the other methods is selected to be $\delta=0.025$. As mentioned in Section \ref{sec:relaxation-of-pac-bayes-kl}, we use the PAC-Bayes-classic bound by \citeauthor{McA99} in replacement of PAC-Bayes-$\kl$ when doing optimization. Note that for all methods, we also have to estimate the empirical loss of the posterior $\E_\pi[\cdot]$ described in Section \ref{sec:estimation-gibbs-loss}. We also allocate the budget for the union bound for the estimation such that these estimations in the bound are controlled with probability at least $1-\delta'$, where we chose $\delta'=0.01$. Therefore, the ultimate bounds for all methods hold with probability at least $1-\delta-\delta'$. Note that we do not consider such bounds during optimization but only when estimating the bounds.

For all methods, we adopt a family of factorized Gaussian distributions to model both priors and posteriors of all the learnable parameters of the classifiers, characterized by the form $\pi=\mathcal{N}(w,\sigma \mathbf I)$ where $w\in\R^d$ denotes the mean vector, and $\sigma$ represents the scalar variance. For all methods, we initialize an uninformed prior $\pi_0=\mathcal{N}(w_0,\sigma_0 \mathbf{I})$ that is independent of data, where the mean is randomly initialized, and the variance $\sigma_0$ is initialized to 0.03 \citep{PRSS21}. 

In the training process of all methods in our experiments, we set the batch size to 250, the number of training epochs to 200, and use stochastic gradient descent with a learning rate of 0.001 and a momentum of 0.95.

\paragraph{Uninformed priors}
We take $\pi_0$ defined above as the uninformed prior. We then learn the posterior $\rho$ from the prior using the entire training dataset $S$, applying a PAC-Bayes bound. We evaluate the bound using, again, the entire training dataset $S$.

\paragraph{Data-informed priors}
We start with the same $\pi_0$ as the uninformed prior. We train the informed prior $\pi_1$ using $S_1$ with $|S_1|=|S|/2$ by minimizing a PAC-Bayes bound. The posterior $\rho$ is then learned using the informed prior $\pi_1$ and the subset $S_2$ with $|S_2|=|S|/2$, again by minimizing a PAC-Bayes bound. The bound is evaluated using $S_2$.

\paragraph{Data-informed priors + excess loss}
We train the informed prior $\pi_1$ and the reference classifier $h^*$ using $S_1$ that contains half of the training dataset. $\pi_1$ is obtained by minimizing a PAC-Bayes bound with the uninformed prior $\pi_0$, while the reference classifier $h^*$ is obtained by an empirical risk minimizer (ERM). The posterior $\rho$ is obtained by minimizing a PAC-Bayes bound on the excess loss between $\rho$ and $h^*$. The prior used in the bound for both training and evaluation is the data-informed prior  $\pi_1$. Therefore, the data for both training and evaluation of $\rho$ must be the other half of data $S_2$.

\subsubsection{Further results for the experiments}\label{app:sec:more-tables}

In this section, we report some more statistics for all methods.

For all methods, to calculate the classification loss of $\rho$ on the testing data, $\E_\rho[\hat L(h,S^{\text{test}})]$ (Test 0-1), we sample one classifier for each data. The train 0-1 loss for all methods is computed on the entire training dataset $S$, while the test 0-1 loss for all methods is computed on the test dataset $S_{\text{test}}$.

\paragraph{Recursive PAC-Bayes}
We report the additional results of Recursive PAC-Bayes on MNIST with $T=2$ in Table \ref{tab:detail-rpb2-mnist}, $T=4$ in Table \ref{tab:detail-rpb4-mnist}, and $T=6$ in Table~\ref{tab:detail-rpb6-mnist}. We report Recursive PAC-Bayes on Fashion MNIST with $T=2$ in Table \ref{tab:detail-rpb2-fmnist}, $T=4$ in Table \ref{tab:detail-rpb4-fmnist}, and $T=6$ in Table~\ref{tab:detail-rpb6-fmnist}.

\begin{table}[h]
\caption{Insight into the training process of the Recursive PAC-Bayes for $T=2$ on MNIST. 
}
\centering
\begin{tabular}{c|ccccc|c}
$t$  & $\nval_t$ & $\E_{\pi_t}[\hat F_{\gamma_t}(h,\Uval_t,\hat \pi_{t-1})]$ & $\frac{\KL(\pi_t^*\|\pi_{t-1}^*)}{\nval_t}$ & $\Ex_t(\pi_t^*,\gamma_t)$ & $B_t(\pi^*_t)$ & Test 0-1 \\ \hline
1 & 60000 &             & .024 (3e-5) &             & .370 (1e-3) & .254 (2e-3) \\
2 & 30000 & .013 (3e-3) & .024 (1e-4) & .136 (3e-3) & .321 (3e-3) & .139 (3e-3) 
\end{tabular}
\label{tab:detail-rpb2-mnist}
\end{table}

\begin{table}[h]
\caption{Insight into the training process of the Recursive PAC-Bayes for $T=4$ on MNIST. 
}
\centering
\begin{tabular}{c|ccccc|c}
$t$  & $\nval_t$ & $\E_{\pi_t}[\hat F_{\gamma_t}(h,\Uval_t,\hat \pi_{t-1})]$ & $\frac{\KL(\pi_t^*\|\pi_{t-1}^*)}{\nval_t}$ & $\Ex_t(\pi_t^*,\gamma_t)$ & $B_t(\pi^*_t)$ & Test 0-1 \\ \hline
1 & 60000 &             & .023 (8e-5) &             & .374 (1e-3) & .258 (1e-3) \\
2 & 52500 & -4e-4 (1e-3) & .025 (3e-4) & .118 (1e-3) & .305 (1e-3) & .126 (2e-3)  \\
3 & 45000 & .053 (1e-3) & .002 (9e-5) & .087 (2e-3) & .240 (2e-3) & .114 (1e-3)  \\
4 & 30000 & .054 (1e-3) & .001 (2e-5) & .083 (1e-3) & .203 (8e-4) & .109 (1e-3)  
\end{tabular}
\label{tab:detail-rpb4-mnist}
\end{table}

\begin{table}[H]
\caption{Insight into the training process of the Recursive PAC-Bayes for $T=6$ on MNIST.}
\centering
\begin{tabular}{c|ccccc|c}
$t$  & $\nval_t$ & $\E_{\pi_t}[\hat F_{\gamma_t}(h,\Uval_t,\hat \pi_{t-1})]$ & $\frac{\KL(\pi_t^*\|\pi_{t-1}^*)}{\nval_t}$ & $\Ex_t(\pi_t^*,\gamma_t)$ & $B_t(\pi^*_t)$ & Test 0-1 \\ \hline
1 & 60000 &             & .019 (7e-5) &             & .425 (1e-3) & .311 (3e-3) \\
2 & 58125 & -0.013 (1e-3) & .032 (6e-4) & .128 (2e-3) & .341 (3e-3) & .139 (1e-3)  \\
3 & 56250 & .050 (1e-3) & .003 (1e-4) & .093 (6e-4) & .264 (1e-3) & .117 (2e-3)  \\
4 & 52500 & .051 (1e-3) & .001 (6e-5) & .080 (9e-4) & .212 (5e-4) & .108 (2e-3)  \\
5 & 45000 & .051 (1e-3) & 9e-4 (3e-5) & .076 (2e-3) & .182 (1e-3) & .104 (6e-4)  \\
6 & 30000 & .049 (1e-3) & 7e-4 (3e-5) & .074 (1e-3) & .166 (1e-3) & .101 (1e-3) 
\end{tabular}
\label{tab:detail-rpb6-mnist}
\end{table}

\begin{table}[h]
\caption{Insight into the training process of the Recursive PAC-Bayes for $T=2$ on Fashion MNIST. 
}
\centering
\begin{tabular}{c|ccccc|c}
$t$  & $\nval_t$ & $\E_{\pi_t}[\hat F_{\gamma_t}(h,\Uval_t,\hat \pi_{t-1})]$ & $\frac{\KL(\pi_t^*\|\pi_{t-1}^*)}{\nval_t}$ & $\Ex_t(\pi_t^*,\gamma_t)$ & $B_t(\pi^*_t)$ & Test 0-1 \\ \hline
1 & 60000 &             & .011 (3e-5) &             & .466 (1e-3) & .389 (5e-3) \\
2 & 30000 & .064 (3e-3) & .013 (2e-4) & .171 (4e-3) & .404 (3e-3) & .266 (5e-3)
\end{tabular}
\label{tab:detail-rpb2-fmnist}
\end{table}

\begin{table}[h]
\caption{Insight into the training process of the Recursive PAC-Bayes for $T=4$ on Fashion MNIST.
}
\centering
\begin{tabular}{c|ccccc|c}
$t$  & $\nval_t$ & $\E_{\pi_t}[\hat F_{\gamma_t}(h,\Uval_t,\hat \pi_{t-1})]$ & $\frac{\KL(\pi_t^*\|\pi_{t-1}^*)}{\nval_t}$ & $\Ex_t(\pi_t^*,\gamma_t)$ & $B_t(\pi^*_t)$ & Test 0-1 \\ \hline
1 & 60000 &             & .011 (1e-4) &             & .476 (4e-3) & .397 (6e-3) \\
2 & 52500 & .032 (6e-4) & .017 (9e-4) & .147 (2e-3) & .386 (3e-3) & .240 (4e-3)  \\
3 & 45000 & .100 (3e-3) & 3e-3 (1e-4) & .138 (5e-3) & .331 (4e-3) & .222 (5e-3)  \\
4 & 30000 & .095 (2e-3) & 7e-4 (5e-5) & .128 (2e-3) & .293 (1e-3) & .213 (3e-3)  
\end{tabular}
\label{tab:detail-rpb4-fmnist}
\end{table}

\begin{table}[H]
\caption{Insight into the training process of the Recursive PAC-Bayes for $T=6$ on Fashion MNIST.}
\centering
\begin{tabular}{c|ccccc|c}
$t$  & $\nval_t$ & $\E_{\pi_t}[\hat F_{\gamma_t}(h,\Uval_t,\hat \pi_{t-1})]$ & $\frac{\KL(\pi_t^*\|\pi_{t-1}^*)}{\nval_t}$ & $\Ex_t(\pi_t^*,\gamma_t)$ & $B_t(\pi^*_t)$ & Test 0-1 \\ \hline
1 & 60000 &             & 9e-3 (8e-5) &             & .534 (4e-3) & .462 (5e-3)  \\
2 & 58125 & .013 (6e-3) & .023 (1e-3) & .151 (4e-3) & .418 (5e-3) & .254 (6e-3)  \\
3 & 56250 & .091 (3e-3) & 3e-3 (4e-4) & .141 (2e-3) & .350 (3e-3) & .223 (1e-3)  \\
4 & 52500 & .090 (2e-3) & 9e-4 (9e-5) & .121 (1e-3) & .296 (1e-3) & .207 (1e-3)  \\
5 & 45000 & .090 (1e-3) & 6e-4 (3e-5) & .117 (2e-3) & .265 (2e-3) & .199 (2e-3)   \\
6 & 30000 & .093 (1e-3) & 5e-4 (2e-5) & .122 (1e-3) & .255 (1e-3) & .198 (1e-3)  
\end{tabular}
\label{tab:detail-rpb6-fmnist}
\end{table}

\paragraph{Uninformed priors}
We report the additional results of uninformed priors \citep{McA98} on MNIST and Fashion MNIST in Table \ref{tab:detail-uninformed}. As described earlier in Section \ref{sec:idea}, Section \ref{sec:experiments}, and Section \ref{app:sec:other-exp-details}, we evaluate the bound using the entire training set.

\begin{table}[H]
    \caption{Further details to compute the bound for the uninformed prior approach on MNIST and Fashion MNIST.}
    \centering
    \begin{tabular}{l|ccc|c}
          & $\E_\rho[\hat L(h,S)]$ & $\frac{\KL(\rho\|\pi_0)}{n}$ & Bound &  Test 0-1 \\
         \hline
    MNIST  & .343 (2e-3) & .023 (4e-5) & .457 (2e-3) &  .335 (3e-3) \\
    F-MNIST  & .382 (2e-3) & .011 (8e-6) & .464 (2e-3) & .384 (5e-3)
    \end{tabular}
    \label{tab:detail-uninformed}
\end{table}

\paragraph{Data-informed priors}
We report the additional results of data-informed priors \citep{APS07} on MNIST and Fashion MNIST in Table \ref{tab:detail-informed}.  As described earlier in Section \ref{sec:idea}, Section \ref{sec:experiments}, and Section \ref{app:sec:other-exp-details}, we evaluate the bound using $S_2$ that is independent of the data-informed prior $\pi_1$.

\begin{table}[H]
    \caption{Further details to compute the bound for the data-informed prior on MNIST and Fashion MNIST.}
    \centering
    \begin{tabular}{l|ccc|c}
         & $\E_\rho[\hat L(h,S_2)]$ & $\frac{\KL(\rho\|\pi_0)}{|S_2|}$ & Bound & Test 0-1 \\
         \hline
    MNIST  & .376 (8e-4) & 8e-4 (9e-6) & .408 (9e-4) & .371 (6e-3) \\
    F-MNIST  & .412 (1e-3) & 4e-4 (7e-6) & .440 (1e-3) & .413 (6e-3)
    \end{tabular}
    \label{tab:detail-informed}
\end{table}

\paragraph{Data-informed priors + excess loss} 
We report the additional results of data-informed priors + excess loss \citep{MGG20} on MNIST and Fashion MNIST in Table \ref{tab:detail-informedexcess-general} and \ref{tab:detail-informedexcess-bounds}.  As described earlier in Section \ref{sec:idea}, Section \ref{sec:experiments}, and Section \ref{app:sec:other-exp-details}, we evaluate the bound using $S_2$ that is independent of the data-informed prior $\pi_1$ and the reference prediction rule $h^*$. The bound is composed of two parts: a bound on the excess loss of $\rho$ with respect to $h^*$ (Excess bound) and a single hypothesis bound on $h^*$ ($h^*$ bound). We report the two components of the bound in Table \ref{tab:detail-informedexcess-general}. We provide further details to compute these bounds from the losses of their corresponding quantities in Table \ref{tab:detail-informedexcess-bounds}. 

\begin{table}[H]
    \caption{Details to compute the bound for the data-informed prior and excess loss on MNIST and Fashion MNIST. The table shows the bound on the excess loss of $\rho$ with respect to $h^*$ (Excess bound) and a single hypothesis bound on $h^*$ ($h^*$ bound).}
    \centering
    \begin{tabular}{l|ccc|c}
          & Ex. Bound & $h^*$ Bound & Bound & Test 0-1 \\
         \hline
        MNIST & .162 (1e-3) & .029 (4e-4) & .192 (2e-3) & .151 (3e-3) \\
        F-MNIST & .196 (5e-3) & .145 (1e-3) & .342 (6e-3) & .285 (5e-3)
    \end{tabular}
    \label{tab:detail-informedexcess-general}
\end{table}

\begin{table}[H]
    \caption{Further details to compute the bound for the data-informed prior and excess loss on MNIST and Fashion MNIST. The table shows the empirical excess loss $\E_\rho[\hat \Delta(h,h^*,S_2)]$, where we define $\Delta(h,h^*,S_2)=\hat L(h,S_2)-\hat L(h^*,S_2)$, and its bound (Excess Bound). It also shows the empirical loss of the reference prediction rule $\hat L(h^*,S_2)$ and its bound. The computation of such bound does not involve the $\KL$ term.}
    \centering
    \begin{tabular}{l|ccc|cc|c}
         & $\E_\rho[\hat \Delta(h,h^*,S_2)]$ & $\frac{\KL(\rho\|\pi_1)}{|S_2|}$ & Ex. Bound & $\hat L(h^*,S_2)$ & $h^*$ Bound & Bound \\
         \hline
        MNIST & -0.011 (3e-3) & .035 (5e-4) & .162 (1e-3) & .026 (4e-4) & .029 (4e-4) & .192 (2e-3)\\
        F-MNIST & .104 (6e-3) & .018 (5e-4) & .196 (5e-3) & .112 (1e-3) & .145 (1e-3) & .342 (6e-3)
    \end{tabular}
    \label{tab:detail-informedexcess-bounds}
\end{table}

\chapter{Another Perspective on Recursive PAC-Bayes: Connections to Cold Posteriors and KL-Annealing}
\label{chap:chap4}
\newpage

\section{Introduction}
The field of machine learning has experienced tremendous advancements in recent years, with Bayesian inference playing a pivotal role in the development of robust and reliable models. In particular, the PAC-Bayesian framework has been instrumental in providing generalization guarantees for learning algorithms by incorporating prior knowledge into the learning process \citep{STW97,McA98}. Despite its theoretical strengths, traditional PAC-Bayesian approaches face significant challenges when it comes to sequential prior updates, often losing valuable confidence information in the process \citep{McA98, APS07, MGG20}. This limitation has spurred the development of novel methodologies aimed at preserving the information of priors as they are updated with new data \citep{CWR23,BG23,RTS24}.

In this context, Recursive PAC-Bayes (RPB), introduced in Chapter \ref{chap:chap3}, emerges as a powerful technique that addresses the shortcomings of traditional PAC-Bayesian methods. By employing a novel decomposition of expected loss and recursive bounding techniques, Recursive PAC-Bayes enables sequential prior updates without losing confidence information. This innovative approach not only preserves the integrity of prior knowledge but also continually enhances the learning process.

Parallel to these advancements, cold posteriors \citep{WRVS+20} and KL-annealing \citep{DBLP:conf/conll/BowmanVVDJB16, DBLP:journals/jmlr/BinghamCJOPKSSH19, DBLP:journals/pami/ZhangBKM19, DBLP:journals/corr/SonderbyRMSW16}, two commonly used practical techniques, have shown empirical success in the application of Bayesian methods. Cold posteriors \citep{WRVS+20}, discussed thoroughly in Chapter \ref{chap:chap2} and characterized by the use of a temperature scaling factor of less than one ($T < 1$), have demonstrated significant improvements in the performance of Bayesian neural networks. This counterintuitive adjustment challenges traditional Bayesian paradigms, offering new insights into optimizing posterior distributions.

KL-annealing \citep{DBLP:conf/conll/BowmanVVDJB16, DBLP:journals/jmlr/BinghamCJOPKSSH19, DBLP:journals/pami/ZhangBKM19, DBLP:journals/corr/SonderbyRMSW16}, another practical technique, gradually increases the weight of the Kullback-Leibler (KL) divergence term in the loss function during training. This method strikes a crucial balance between fitting the empirical data and maintaining a close distance between the prior and posterior distributions, thus enhancing the stability and performance of variational inference.

In this chapter, we aim to explore the connections between these methodologies and Recursive PAC-Bayes. By delving into the updating rules of Recursive PAC-Bayes, we will uncover how its recursive updates can be viewed through the lens of cold posteriors and KL-annealing. Through this exploration, we seek to provide another perspective on the success of Recursive PAC-Bayes.

\begin{itemize}
    \item In Section \ref{sec:4.1}, we elucidate the connection between cold posteriors and Recursive PAC-Bayes.
    \item In Section \ref{sec:4.2}, we provide a brief introduction to KL-annealing and demonstrate its connection to Recursive PAC-Bayes.
\end{itemize}

\section{Recursive PAC-Bayes and cold posteriors}\label{sec:4.1}
To see the connection between Recursive PAC-Bayes and cold posteriors \citep{WRVS+20} (if needed, please refer to Chapter \ref{chap:chap2} for introduction or background knowledge), we first present the form of cold posteriors (Equation \ref{eq:likelihoodTempering}):
\begin{align*}
    p_\lambda(\bmtheta|D) \propto p(\bmY|\bmX,\bmtheta)^\lambda p(\bmtheta), \quad \lambda > 1 \quad \text{where } \lambda=\dfrac{1}{T}.
\end{align*}
Such cold posteriors can be obtained by optimizing the (generalized) ELBO objective (Equation \ref{eq:likelihoodTempering:ELBO}) \citep{alquier2016properties,HMPB+17}. To make it clearer, we rewrite Equation \ref{eq:likelihoodTempering:ELBO} in the form of average loss:
\begin{align*}
    p_\lambda(\bmtheta|D) &= \argmin_\rho \E_\rho[-\ln p(D|\bmtheta)] + \frac{1}{\lambda}\operatorname{KL}(\rho(\bmtheta|D),p(\bmtheta)) \\
    &= \argmin_\rho \E_\rho[-\ln \prod_{(\bmy,\bmx) \in D} p(\bmy|\bmx,\bmtheta)] + \frac{1}{\lambda}\operatorname{KL}(\rho(\bmtheta|D),p(\bmtheta)) \\
    &= \argmin_\rho \E_\rho[-\sum_{(\bmy,\bmx) \in D} \ln p(\bmy|\bmx,\bmtheta)] + \frac{1}{\lambda}\operatorname{KL}(\rho(\bmtheta|D),p(\bmtheta)) \\
    &= \argmin_\rho \E_\rho[-\frac{1}{n}\sum_{(\bmy,\bmx) \in D} \ln p(\bmy|\bmx,\bmtheta)] + \underbrace{\frac{1}{\lambda}}_{\text{temperature $T$}} \underbrace{\frac{\operatorname{KL}(\rho(\bmtheta|D),p(\bmtheta))}{n}}_{\text{average KL}},
\end{align*}
where $n$ is the size of the dataset $D$ used for constructing the posterior. It is now clear that obtaining cold posteriors ($T < 1$, or $\lambda > 1$ equivalently) corresponds to searching a posterior that optimizes a (generalized) ELBO objective where the coefficient of the average KL term is lesser than $1$.

On the other hand, in Recursive PAC-Bayes, we obtain the sequential posteriors using Equation \ref{eq:opt-pi-t}:
\begin{align*}
    \pi_t^* &= \arg\min_\pi \sum_{j=1}^3 (b_{t|j} - b_{t|j-1})\\
    &\cdot \kl^{-1,+}\left(\E_{\pi}\left[\hat F_{\gamma_t|j}(h,S_t, \hat \pi_{t-1}^*)\right], \frac{\KL(\pi\|\pi_{t-1}^*)+\ln\frac{6T\sqrt{\nval_t}}{\delta}}{\nval_t}\right) \\
    &= \arg\min_\pi \sum_{j=1}^3 (b_{t|j} - b_{t|j-1})\\
    &\cdot \kl^{-1,+}\left(\E_{\pi}\left[\hat F_{\gamma_t|j}(h,S_t, \hat \pi_{t-1}^*)\right], \underbrace{\frac{|S_t|}{\nval_t}}_{\text{temperature $T$}} \underbrace{\frac{\KL(\pi\|\pi_{t-1}^*)}{|S_t|}}_{\text{average KL}}+\dfrac{\ln\frac{6T\sqrt{\nval_t}}{\delta}}{\nval_t}\right).
\end{align*}
As shown in Figure \ref{fig:evolution}, for every step $t$ until the final step $T$, the denominator of the coefficient of the average KL term, $\nval_t = |\Uval_t| = \sum_{s=t}^T |S_s|$, is always larger than the size of the current dataset $|S_t|$ used for the construction of the posteriors, thus resulting in a cold temperature $T=\frac{|S_t|}{\nval_t}<1$.

In this sense, Recursive PAC-Bayes can be seen as implicitly utilizing cold posteriors for all sequential steps $t$ except for the last step $T$. As Bayesian neural networks ($\lambda=1$) are shown to have an underfitting phenomenon with curated image datasets (please refer to Chapter \ref{chap:chap2}), such a procedure allows Recursive PAC-Bayes to make more aggressive updates for fitting the data instead of staying close to the uninformed priors, especially in the early stages. Thus, it mitigates the underfitting problem and achieves better performance.

However, while this connection helps explain the success of Recursive PAC-Bayes, it is important to note the differences. Recursive PAC-Bayes is only ``partially cold'' compared to cold posteriors, and in each update step on the corresponding chunk of data $S_t$, the implicit temperature is different, increasing gradually from the lowest temperature ($T \ll 1$, or $\lambda \gg 1$ equivalently) to normal ($T =1$, or $\lambda=1$). More importantly, the implicit temperature in Recursive PAC-Bayes, $\frac{|S_t|}{\nval_t}$, is determined by the data splitting strategy rather than as a pre-defined hyperparameter in cold posteriors. Interestingly, obtained from rigorous theoretical bounds, such a method might provide insights for choosing appropriate temperatures for Bayesian predictive models.

\section{Recursive PAC-Bayes and KL-annealing}\label{sec:4.2}

\begin{figure}[H]
    \centering
    \includegraphics[width=.95\textwidth]{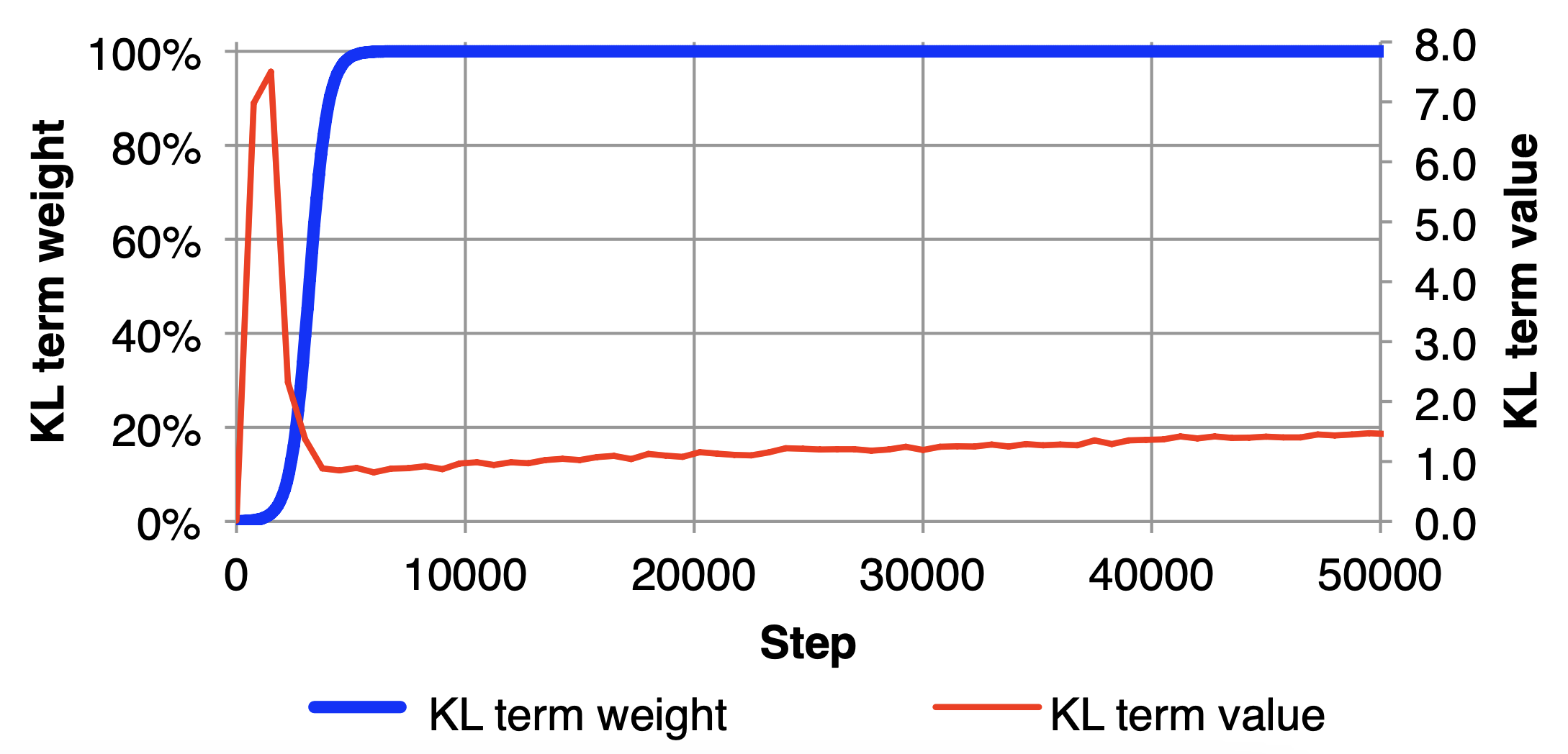}
    \caption{\textbf{Figure 2 in \cite{DBLP:conf/conll/BowmanVVDJB16}: a graphical illustration of KL-Annealing.} 
    } \label{fig:anneal}
\end{figure}

In the realm of Bayesian inference and variational methods, KL-annealing \citep{DBLP:conf/conll/BowmanVVDJB16, DBLP:journals/jmlr/BinghamCJOPKSSH19, DBLP:journals/pami/ZhangBKM19, DBLP:journals/corr/SonderbyRMSW16} has emerged as a powerful technique to improve the training of probabilistic models, particularly variational autoencoders (VAEs) and Bayesian neural networks. KL-annealing addresses the challenge of balancing the trade-off between data fidelity and regularization during the optimization process, a crucial aspect in ensuring that the model learns a meaningful representation of the data while avoiding overfitting.

KL-annealing addresses this challenge by gradually increasing the weight of the KL divergence term in the ELBO during training, as demonstrated graphically in Figure \ref{fig:anneal}. This approach allows the model to initially focus on fitting the data, gradually introducing the regularization effect of the KL term to ensure the optimization of the posterior distribution remains well-behaved. The modified ELBO with annealed KL divergence is given by:
\begin{align*}
    \text{ELBO}_\beta = \mathbb{E}_{q(\bmtheta)} \left[ \log p(D \mid \bmtheta) \right] - \beta(t) \cdot \text{KL}(q(\bmtheta), p(\bmtheta)),
\end{align*}
where $\beta(t)$ is an annealing parameter that starts from a small value (often close to zero) and is gradually increased to one during the training process.

KL-annealing offers several benefits that enhance the training and performance of probabilistic models. By gradually increasing the weight of the KL divergence term in the ELBO during training, it allows the model to initially focus on fitting the data accurately, leading to improved convergence rates and more precise representations. This technique helps prevent mode collapse, ensuring that the model captures diverse features of the data. Additionally, KL-annealing stabilizes the training dynamics by avoiding abrupt changes in the optimization landscape, resulting in a balanced trade-off between data fidelity and regularization. This balance improves the model's ability to generalize well to new, unseen data, making it a valuable approach in variational inference and Bayesian neural networks.

As shown in Section \ref{sec:4.1}, the connection between Recursive PAC-Bayes and KL-annealing becomes evident: both methods gradually increase the KL coefficient, which converges to one at the final stages of optimization. Although Recursive PAC-Bayes may benefit from KL-annealing, there is a significant difference. Recursive PAC-Bayes employs different KL coefficients (i.e., temperatures) for each chunk of data, with optimization performed until convergence. In contrast, KL-annealing trains on all data at once, using smaller-than-one KL coefficients (cold temperatures) for more efficient optimization in the early stages, but ultimately converges with a KL coefficient equal to one.

\section{Conclusion}
In this chapter, we explored the Recursive PAC-Bayes (RPB) framework and its connections to cold posteriors and KL-annealing. Recursive PAC-Bayes addresses traditional PAC-Bayesian challenges by enabling sequential prior updates without losing confidence information, thereby enhancing the learning process with each iteration.

The concept of cold posteriors, which involves a temperature scaling factor less than one, aligns with Recursive PAC-Bayes's approach of using ``cold'' coefficients in the KL term during optimization (except for the final update). This adjustment allows Recursive PAC-Bayes to make more aggressive updates early in the learning process, helping to mitigate underfitting and improve performance.

Similarly, KL-annealing gradually increases the weight of the KL divergence term during training, balancing data fidelity and regularization dynamically for better optimization. Recursive PAC-Bayes's recursive updates inherently incorporate this principle, leading to stable training dynamics and better generalization.

By considering the perspectives of cold posteriors and KL-annealing, we showed that Recursive PAC-Bayes implicitly leverages their strengths. This connection provides a deeper understanding of the success of Recursive PAC-Bayes and highlights its potential for advancing machine learning models.

\chapter{(Generalized) ELBO Revisited: Another Decomposition for Mean-Field Variational Global Latent Variable Models}
\label{chap:chap5}
\newpage

\section{Introduction}
Variational inference \citep{blei2017variational, DBLP:journals/ml/JordanGJS99, DBLP:journals/ftml/WainwrightJ08} has become a cornerstone technique for approximate Bayesian inference in complex models. The Evidence Lower Bound (ELBO) is a crucial objective in this framework, facilitating the optimization of variational distributions to approximate true posteriors. While \cite{hoffman2016elbo} have provided insightful decompositions of the ELBO for local latent variable models like Variational Autoencoders (VAEs) \citep{DBLP:journals/corr/KingmaW13}, there remains a need for a deeper understanding of ELBO dynamics in global latent variable models.

In this chapter, we introduce a novel decomposition for variational global latent variable models, where the latent variables (usually the model or network weights) are shared across all data points, contrasting with the local latent variables (usually latent representation of each data point) in models like VAEs. More specifically, our decomposition holds for any mean-field variational distribution, which is the common design choice in practice. 

Our proposed decomposition aims to shed light on several aspects of the training and evaluation of variational global latent variable models:
\begin{itemize}
    \item Understanding training dynamics: by dissecting the ELBO into interpretable components, we can gain insights into how different parts of the objective function influence the training process, potentially leading to more effective optimization strategies.
    \item Evaluating the variational posteriors: the decomposition provides a finer-grained analysis of the variational posterior quality, enabling practitioners to diagnose and address issues in the approximation process more effectively.
    \item Finer control of the optimization process: with a clearer understanding of the contributions of different terms in the ELBO, we can develop more nuanced training controls, e.g., inserting separate coefficients for each term, potentially improving convergence and performance.
\end{itemize}

\section{Rewriting the (generalized) ELBO for mean-field variational global latent variable models}
In (generalized) variational inference, we consider approximating the (tempered) Bayesian posterior $p(\bmtheta|D)^{\lambda}$ with a variational distribution $q(\bmtheta)$. Such variational posteriors are typically obtained by optimizing the (generalized) ELBO objective (essentially Equation \ref{eq:likelihoodTempering:ELBO}) within the selected variational family:
\begin{align*}
    q(\bmtheta) &= \argmin_q \E_q[-\ln p(D|\bmtheta)] + \frac{1}{\lambda}\operatorname{KL}(q(\bmtheta),p(\bmtheta)),
\end{align*} 
where $p(\boldsymbol{\theta})$ is the prior and typically chosen to be the standard normal distribution for tractability. In practice, people often consider a factorized (i.e., mean-field) variational distribution, $q(\bmtheta)=\prod_{d=1}^{D}q_d(\theta_d)$, where $\bmtheta \in \bmTheta$, the parameters of the predictive model, is a $D$-dimension vector and $d$ indexes the dimension. 

Inspired by \cite{hoffman2016elbo}, we present a similar but novel decomposition of the KL term in ELBO. However, while \cite{hoffman2016elbo}'s decomposition is for variational local latent variables (e.g. VAEs), our decomposition applies to variational global latent variable models (e.g. Bayesian neural networks). To begin our decomposition, we first define the following:
\begin{align*}
    &q(d,\theta)=q(d)q(\theta|d), \quad q(\theta|d)=q_d(\theta_d), \quad q(d)=\dfrac{1}{D}, \quad q^{\text{avg}}(\theta)=\dfrac{1}{D}\sum_{d=1}^{D}q_d(\theta_d)\\
    &p(d,\theta)=p(d)p(\theta|d), \quad p(\theta|d)=p(\theta), \quad p(d)=\dfrac{1}{D},
\end{align*}
where the index $d$ is treated as a random variable and $q^{\text{avg}}(\theta)=\sum_{d=1}^{D}q(d,\theta)=\dfrac{1}{D}\sum_{d=1}^{D}q_d(\theta_d)$ is the average (or ``mixture'') variational posterior over the dimensions. Then, the (generalized) ELBO can be rewritten as:
\begin{align} \label{eq:5.1}
    \text{ELBO}_{\lambda}=\E_q[-\ln p(D|\bmtheta)] + \frac{1}{\lambda} \left( D *\left( \operatorname{KL}\left(q^{\text{avg}}\left(\theta\right), p(\theta) \right) + \mathbb{I}_{q(d,\theta)}[d,\theta]\right) \right),
\end{align}
where $\mathbb{I}_{q(d,\theta)}[d,\theta]$ is the mutual information between the index $d$ and the parameter $\theta$ of the $d$-th dimension. 

Next, we demonstrate the decomposition of the KL term in Equation \ref{eq:5.1}:
\begin{align*}
    &  \text{KL}\left(q\left(\boldsymbol{\theta}\right), p(\boldsymbol{\theta})\right)\\
    &= \sum_{d=1}^{D}\text{KL}(q_d(\theta_d),p(\theta)) \\
    &= \sum_{d=1}^{D} \int q_d\left(\theta_d\right) \log \dfrac{q_d\left(\theta_d\right)}{p\left(\theta\right)} d\theta \text{\quad definition of KL} \\ 
    &= \sum_{d=1}^{D} \int q\left(\theta|d\right) \log \dfrac{q\left(\theta|d\right)}{p\left(\theta|d\right)} d\theta \text{\quad substitute the definitions accordingly} \\
    &= D \sum_{d=1}^{D} \int \dfrac{1}{D} q\left(\theta|d\right) \log \dfrac{q\left(\theta|d\right)\dfrac{1}{D}}{p\left(\theta|d\right)\dfrac{1}{D}} d\theta \text{\quad add $D$ and $\dfrac{1}{D}$} \\
    &= D \sum_{d=1}^{D} \int q(d) q\left(\theta|d\right) \log \dfrac{q\left(\theta|d\right)q(d)}{p\left(\theta|d\right)p(d)} d\theta \text{\quad substitute the definitions accordingly}\\
    &= D \sum_{d=1}^{D} \int q\left(d,\theta\right) \log \dfrac{q\left(d,\theta\right)}{p\left(d,\theta\right)} d\theta \text{\quad product rule} \\ 
    &= D \sum_{d=1}^{D} \int q\left(d,\theta\right) \log \dfrac{q^{\text{avg}}\left(\theta\right)q\left(d|\theta\right)}{p\left(\theta\right)p\left(d|\theta\right)} d\theta \text{\quad product rule}\\
    &= D \sum_{d=1}^{D} \int q\left(d,\theta\right) \left(\log \dfrac{q^{\text{avg}}\left(\theta\right)}{p\left(\theta\right)}+\log \dfrac{q\left(d|\theta\right)}{p\left(d\right)}\right) d\theta \quad \text{since } p(d|\theta)=p(d) \\
    &= D \left(\sum_{d=1}^{D} \int q\left(d,\theta\right) \log \dfrac{q^{\text{avg}}\left(\theta\right)}{p\left(\theta\right)}+ q\left(d,\theta\right)\log \dfrac{q\left(d|\theta\right)}{p\left(d\right)} d\theta \right) \text{\quad combine the log}\\
    &= D \left(\int \sum_{d=1}^{D} q\left(d,\theta\right) \log \dfrac{q^{\text{avg}}\left(\theta\right)}{p\left(\theta\right)}d\theta + \int\sum_{d=1}^{D} q\left(d,\theta\right)\log \dfrac{q\left(d|\theta\right)}{p\left(d\right)} d\theta \right) \text{\quad Fubini's theorem}\\
    &= D \left( \int q^{\text{avg}}\left(\theta\right) \log \dfrac{q^{\text{avg}}\left(\theta\right)}{p\left(\theta\right)}d\theta + \int\sum_{d=1}^{D} q\left(d,\theta\right)\log \dfrac{q\left(d,\theta\right)}{q\left(d\right)q^{\text{avg}}\left(\theta\right)} d\theta \right) \text{\quad $d$ is integrated out}\\
    &= D \left(\text{KL}\left(q^{\text{avg}}\left(\theta\right), p(\theta) \right) + \mathbb{I}_{q(d,\theta)}[d,\theta] \right). \text{\quad substitute the definitions accordingly}
\end{align*}

Intuitively, the mutual information term $\mathbb{I}_{q(d,\theta)}[d,\theta]$ could be able to describe the learning process. In the beginning, $\mathbb{I}_{q(d,\theta)}[d,\theta]$ should be $0$ as we start from a factorized prior, where the dimension $d$ is independent of the parameter $\theta$. As learning progresses, $\mathbb{I}_{q(d,\theta)}[d,\theta]$ should increase gradually and stay at a stable level. On the other hand, the average variational posterior, $q^{\text{avg}}(\theta)=\dfrac{1}{D}\sum_{d=1}^{D}q_d(\theta_d)$, could be arbitrarily close to the prior $p(\theta)$, enabling the term $\operatorname{KL}\left(q^{\text{avg}}\left(\theta\right), p(\theta) \right)$ to be arbitrarily small. Think of each dimension $q_d(\theta_d)$, even though being different and diverse in itself, could still constitute a mixture centered at $0$.

Furthermore, one might consider adding extra parameters to control the optimization objective as needed \citep{DBLP:conf/aistats/Esmaeili0JBSPBD19}, ending in the following form for (generalized) ELBO:
\begin{align} \label{eq:5.2}
    \text{ELBO}_{\lambda}=\E_q[-\ln p(D|\bmtheta)] + D *\left( \frac{1}{\lambda_1}\operatorname{KL}\left(q^{\text{avg}}\left(\theta\right), p(\theta) \right) + \frac{1}{\lambda_2}\mathbb{I}_{q(d,\theta)}[d,\theta]\right).
\end{align} 

\section{Conclusion}
In this chapter, we have introduced a novel decomposition for the Evidence Lower Bound (ELBO) tailored to mean-field variational global latent variable models. This decomposition could provide a deeper understanding of the training dynamics and optimization challenges inherent in these models, which involve latent variables shared across all data points. By dissecting the ELBO into more interpretable components, this decomposition facilitates finer control over the optimization process and allows for a more detailed evaluation of the quality of variational posteriors. These insights might lead to the development of more effective training strategies, ultimately enhancing the performance and reliability of variational global latent variable models.

\chapter{Summary and Discussion}
\label{chap:chap6}

In Chapter \ref{chap:chap2}, we provided a detailed analysis of the cold posterior effect (CPE) in Bayesian deep learning. Theoretically, the Bayesian posterior is considered optimal when the model is perfectly specified. However, empirical evidence shows that posteriors with a temperature \( T < 1 \) can outperform the standard Bayesian posterior, a phenomenon known as CPE. We investigated the underlying causes of CPE, establishing that it occurs primarily due to underfitting resulting from model misspecification. Moreover, we demonstrated that these tempered posteriors are valid Bayesian posteriors corresponding to different combinations of likelihoods and priors, thus justifying temperature adjustment as an effective strategy for addressing underfitting. This finding is significant as it not only deepens our theoretical understanding of CPE but also provides a practical method for improving Bayesian model performance through temperature fine-tuning.

In Chapter \ref{chap:chap3}, we presented the development of Recursive PAC-Bayes (RPB), a novel framework for sequentially updating posteriors in a manner that preserves confidence information. Traditional PAC-Bayesian analysis faces challenges in maintaining confidence when posteriors are updated sequentially, limiting its utility in real-world applications. The introduction of RPB addresses this issue by decomposing the expected loss of randomized classifiers into an excess loss relative to a downscaled prior loss. This approach enables recursive bounding, allowing for meaningful and beneficial sequential data processing without losing crucial confidence information. We further enhanced the practical applicability of PAC-Bayesian analysis by generalizing key inequalities and developing a robust optimization procedure based on RPB. The empirical results indicated that RPB significantly outperforms existing methods, making it a valuable tool for improving the generalization capabilities of neural networks.


In Chapter \ref{chap:chap4}, we explored the Recursive PAC-Bayes (RPB) framework and its relationship with cold posteriors and KL-annealing, two empirically effective techniques in Bayesian modeling. We demonstrated that the intermediate posteriors generated by RPB's updating rules are inherently ``cold'', aligning with the principles underlying cold posteriors. Furthermore, RPB's recursive updates is similar to the process of KL-annealing, with the potential of bringing in more stable training dynamics. These connections not only shed light on the success of RPB but also suggest intriguing avenues for future research.

In Chapter \ref{chap:chap5}, we introduced a novel decomposition of the evidence lower bound (ELBO) for variational global latent variable models. The proposed decomposition breaks down the ELBO into more interpretable components, with the potential of offering a finer-grained understanding of the training dynamics and posterior quality in these models. This decomposition might improve optimization strategies by allowing for more nuanced control of the temperature parameter \( T \), leading to better convergence and model performance.


\newpage
\phantomsection
\pagestyle{plain}
\chapter*{List of Publications}
\addcontentsline{toc}{section}{List of Publications}
The work presented in this thesis has led to this publication \citep{zhang2024the} and this preprint \citep{wu2024recursivepacbayesfrequentistapproach}. Additionally, the work during my Ph.D. has resulted in another publication \citep{DBLP:journals/entropy/LyuCZT23}. However, for the coherence of the contributions, \cite{DBLP:journals/entropy/LyuCZT23} is not included in the thesis.

\begin{enumerate}
    \item \bibentry{zhang2024the}. 
    \item \bibentry{wu2024recursivepacbayesfrequentistapproach}.
    \item \bibentry{DBLP:journals/entropy/LyuCZT23}.
\end{enumerate}


\newpage

\addcontentsline{toc}{section}{Bibliography}
\bibliography{references}

\end{document}